\PassOptionsToPackage{table}{xcolor}

\documentclass[twoside,11pt]{article}
\usepackage{multirow}
\usepackage{blindtext}

\usepackage[preprint]{jmlr2e}
\usepackage[utf8]{inputenc} 
\usepackage[T1]{fontenc}    
\usepackage{hyperref}       
\usepackage{url}            
\usepackage{booktabs}       
\usepackage{amsfonts}       
\usepackage{nicefrac}       
\usepackage{microtype}      
\usepackage{amsmath}
\usepackage{fancyhdr}       
\usepackage{graphicx}       
\usepackage{setspace}
\usepackage{tikz}
\usepackage{multirow}
\usepackage{array}
\usepackage{siunitx}
\usepackage[percent]{overpic}  
\usepackage{subcaption}
\usepackage{bm}
\usepackage{graphicx}
\usepackage{tabularx}
\usepackage{tcolorbox}
\usepackage{bbm}
\usepackage{makecell} 
\usepackage{verbatim}
\usetikzlibrary{arrows.meta, positioning, shapes.multipart, fit, backgrounds}
\usetikzlibrary{calc}
\usepackage{booktabs,siunitx,tabularx,threeparttable,makecell}
\usepackage{booktabs, multirow, siunitx}

\usepackage{algorithmic}
\usepackage{algorithm}
\usepackage{adjustbox}

\usepackage{pgfplots}
\usepackage{pgfplotstable}
\usepackage{wrapfig}
\usepackage{graphicx}
\usepackage{chngcntr}

\pgfplotsset{compat=1.17}

\graphicspath{{img/}}     

\newcommand{\norm}[1]{\left\| #1 \right\|}
\newcommand{\abs}[1]{\left | #1 \right |}

\newcommand{\normdist}[2]{\mathcal{N}(#1, #2)}
\newcommand{\unifdist}[2]{\mathcal{U}(#1, #2)}

\newcommand{\di}[1]{\mathop{d#1}}

\usepackage{lastpage}
\jmlrheading{}{}{1-\pageref{LastPage}}{1/21}{9/22}{21-0000}{e Frutos, Vázquez, Olmos and Míguez}

\ShortHeadings{Robust Implicit Generative Models via ISL}{de Frutos, Vázquez, Olmos and Míguez}
\firstpageno{1}

\begin{document}

\title{Robust training of implicit generative models for multivariate and heavy-tailed distributions with an invariant statistical loss}


\author{%
  \name José Manuel de Frutos\textsuperscript{a,*} \email jofrutos@ing.uc3m.es \\ 
  \name Manuel A. Vázquez\textsuperscript{a}      \email mavazque@ing.uc3m.es   \\ 
  \name Pablo M. Olmos\textsuperscript{a}         \email pamartin@ing.uc3m.es  \\ 
  \name Joaquín Míguez\textsuperscript{a}         \email jmiguez@ing.uc3m.es   \\
  \addr \textsuperscript{a}Department of Signal Theory and Communications, Carlos III University of Madrid, Leganés 28911, Madrid, Spain
}
\def\thefootnote{\fnsymbol{footnote}}%
\footnotetext[1]{Corresponding author.}%
\def\thefootnote{\arabic{footnote}}%
\editor{My editor}

\maketitle

\begin{abstract}



Traditional implicit generative models rely on adversarial discriminators, which can lead to unstable optimization and mode collapse. In contrast, the \emph{Invariant Statistical Loss} (ISL) framework \citep{de2024training} defines a divergence between real and generated distributions by comparing empirical ranks. We provide a formal characterization of ISL as a proper divergence on continuous distributions, and prove that it is both continuous and differentiable, making it suitable for gradient‐based training without adversarial instability.


To broaden ISL’s applicability, we address two key limitations. Many real-world distributions exhibit heavy tails that standard Gaussian-based generators fail to capture. We introduce Pareto-ISL, replacing the latent Gaussian noise with a generalized Pareto distribution, enabling the generator to better represent both central and extreme events.
We also extend ISL to multivariate data through ISL-slicing, a scalable approach that projects data onto random one-dimensional subspaces, computes rank-based losses per slice, and averages them. This preserves joint structure while remaining computationally efficient.

Together, these contributions make ISL a practical and robust objective across a range of settings. In experiments, Pareto-ISL improves tail fidelity, and ISL-slicing scales effectively to high-dimensional data while also serving as a strong pretraining strategy for GANs.

\end{abstract}


\section{Introduction}
\subsection{Motivation}
Generative modeling lies at the heart of many machine‐learning applications—from density estimation \citep{kingma2013auto, ho2020denoising, papamakarios2021normalizing} and data augmentation \citep{shorten2019survey} to unsupervised representation learning \citep{radford2015unsupervised, chen2016infogan}. Broadly speaking, these models fall into two categories (see \cite{mohamed2016learning}): \emph{prescribed models}, which posit and optimize an explicit density (e.g., variational autoencoders \citep{kingma2013auto} or diffusion-based models \citep{ho2020denoising}), and \emph{implicit models}, which learn to generate samples by transforming simple latent noise through a neural network, without ever evaluating a tractable likelihood, for example, generative adversarial networks (GANs).

Since the seminal introduction of GANs by \citet{goodfellow2014generative}, implicit generative modeling has become a cornerstone of modern unsupervised learning. GANs define a two-player minimax game between a generator, which transforms latent noise into samples, and a discriminator, which learns to distinguish real from generated data. While they are capable of producing sharp and high-fidelity outputs, early GANs were plagued by training instability and \emph{mode collapse}, where the generator fails to capture the full diversity of the data distribution (see \citet{arora2018principled}).

To address these issues, a wide range of architectural and theoretical strategies have been developed. The Wasserstein GAN (WGAN) \citep{arjovsky2017wasserstein} replaces the Jensen–Shannon divergence with the Earth Mover’s distance, resulting in more stable gradients and training dynamics. This requires enforcing Lipschitz continuity on the discriminator—initially via weight clipping and later improved with gradient penalties \citep{gulrajani2017improved}. Spectral normalization \citep{miyato2018spectral}, unrolled optimization \citep{metz2016unrolled}, and Jacobian regularization \citep{mescheder2018training} can be used to improve convergence and reduce oscillatory behavior. Multi-discriminator setups \citep{durugkar2016generative, choi2022mcl} and kernel-based alternatives like MMD-GAN \citep{li2017mmd} and Cramér GAN \citep{bellemare2017cramer} have also been explored to address mode collapse. At the architectural level, models such as BigGAN \citep{brock2018large} and StyleGAN \citep{karras2019style, karras2020analyzing, karras2021alias} introduce progressive growing, adaptive normalization, and class conditioning to improve sample quality and control.

While these advances improve training stability and sample realism, they primarily focus on distributions with bounded or light-tailed support. A few recent efforts have begun to tackle the problem of modeling heavy-tailed distributions. For example, ExGAN \citep{bhatia2021exgan} introduces a tail-adaptive training objective to emphasize rare, high-magnitude samples. Pareto GAN \citep{huster2021pareto} modifies the latent space by injecting non-Gaussian, Pareto-distributed noise to better capture extreme-value behavior. However, these models continue to rely on adversarial training, which becomes even more fragile in heavy-tailed settings and requires careful tuning and regularization.

Despite these advances, GAN training remains highly sensitive to hyperparameters and often requires heuristic stabilization techniques. Moreover, implicit models frequently struggle to capture extreme events and heavy-tailed behaviors—phenomena that are especially important in domains such as finance, climate modeling, and anomaly detection \citep{lu2023cm, gong2024testing, seo2024stabilized}.

To address these issues, we build on the \emph{Invariant Statistical Loss} (ISL) framework of \citet{de2024training}, which defines a statistical discrepancy based on the invariance properties of rank statistics. Unlike adversarial methods that rely on a learned critic to compare distributions, ISL replaces the discriminator entirely with a principled, distribution-free criterion.

At the core of ISL lies the following observation: suppose we draw \(K\) i.i.d. samples \(\tilde{y}_1, \dots, \tilde{y}_K \sim \tilde{p}\) from a generative model and take a real data point \(y \sim p\). If we count how many of the samples $\tilde{y}_i$ are less than $y$—a quantity called the rank of $y$—then this rank is uniformly distributed on $\{0, \dots, K\}$ if and only if the model and data distributions coincide, i.e., $p = \tilde{p}$ (see Figure \ref{fig:isl-intuition-cartoon}). This means that the rank statistic acts as an implicit test of distributional equality, without ever needing to evaluate or estimate the true data density \(p\).

Based on this observation, we can repeat the procedure with multiple data points and define a loss function by measuring how far the empirical distribution of ranks deviates from uniformity. This yields a tractable and robust discrepancy that can be estimated efficiently using only samples and optimized via a differentiable surrogate that mimics the rank histogram. The result is a fully sample-based divergence that is invariant with respect to the true form of \(p\), does not involve min-max games, and avoids the instability commonly found in adversarial training. Moreover, as we demonstrate in this work, ISL can be extended and adapted to tackle key limitations of current implicit methods—namely, their poor tail fidelity and tendency to collapse on modes.

While the original ISL formulation in \citet{de2024training} offers a compelling alternative to adversarial training, it is essentially constrained to univariate data and assumed light-tailed target distributions. Extending ISL to multivariate by aligning marginal distributions is ineffective, as it ignores dependencies and becomes computationally intractable in high dimensional distributions. Furthermore, using simple latent priors like Gaussians restricts the model capacity to capture extreme-value behavior. Moreover, the theoretical formulation and properties of ISL have not been rigorously established.
\vspace{0.5cm}
\begin{figure}[ht]
  \centering
  \begin{tikzpicture}[
    every node/.style={font=\sffamily\scriptsize},
    dot/.style={circle, fill=blue!40, inner sep=1pt},
    red/.style={circle, fill=red!60, inner sep=1.5pt},
    histbar/.style={fill=blue!30, draw=black, line width=0.15mm},
    arrow/.style={-{Latex[length=1.2mm]}, thick, opacity=0.4},
    scale=0.95
  ]
    \begin{scope}[shift={(0,0)}] 
      \draw[gray!50, thick] (0,10) -- (6,10);
      \foreach \x in {0.5,1.0,...,5.5} {
        \node[dot] at (\x,10) {};
      }
      \foreach \x in {0.8,1.3,2.2,2.8,3.5,4.0,5.0} {
        \node[red] at (\x,10) {};
      }
      \node at (3,9.7) {\scriptsize Real samples (red) ranked among model samples (blue)};

      \foreach \x in {0,...,5} {
        \draw[histbar] (\x,5.2) rectangle +(0.9,1.5);
      }
      \node at (3,4.8) {\scriptsize Rank Histogram};
      \node at (3,4.2) {\textbf{(a) Ideal case: \(p = \tilde{p}\)}};

      \draw[arrow] (3,9.2) -- (3,7.4);
    \end{scope} \label{fig:isl-intuition-cartoon a}

    \begin{scope}[shift={(8.5,0)}]  \label{fig:isl-intuition-cartoon b}
      \draw[gray!50, thick] (0,10) -- (6,10);
      \foreach \x in {0.3,0.7,1.1,1.5,1.9,2.3,2.7,3.1, 5.2} {
        \node[dot] at (\x,10) {};
      }
      \foreach \x in {1.1,4.2,4.5,4.8,5.1,3.4} {
        \node[red] at (\x,10) {};
      }
      \node at (3.0,9.7) {\scriptsize Real samples (red) fall in tails of model samples (blue)};

      \foreach \x/\h in {0/1.5, 1/2.0, 2/1.9, 3/0.6, 4/0.3, 5/0.1} {
        \draw[histbar] (\x,5.2) rectangle +(0.9,\h);
      }
      \node at (3,4.8) {\scriptsize Rank Histogram};
      \node at (3,4.2) {\textbf{(b) Model mismatch: \(p \neq \tilde{p}\)}};

      \draw[arrow] (3,9.2) -- (3,7.4);
    \end{scope}
  \end{tikzpicture}

  \caption{\small High-level intuition behind ISL. When the distribution \(\tilde{p}\) (distribution associated to the generator neural network) matches the data distribution \(p\), the ranks of real samples \(y\) among generated samples \(\tilde{y}_1, \dots, \tilde{y}_K\) are uniformly distributed (Figure (a)). If the distributions differ, the rank histogram becomes unbalanced (Figure (b)), revealing a difference that ISL uses to detect how well the model fits the data. Red points correspond to real data samples from \(p\), while blue points represent samples from \(\tilde{p}\) (model samples).}
  \label{fig:isl-intuition-cartoon}
\end{figure}

\subsection{Contributions}

In this work, we address each of the limitations outlined above. We first formally define the ISL as a discrepancy between two probability distributions, and then prove that it constitutes a valid divergence by establishing its continuity and (almost everywhere) differentiability with respect to the generator parameters. To overcome the tail modeling issue, we introduce \emph{Pareto-ISL}, which replaces the latent noise with a Generalized Pareto distribution and leverages unbounded generators to capture the behavior of heavy-tailed distributions. For high-dimensional data, we propose \emph{ISL-slicing}, a scalable extension that yields a multivariate ISL by averaging over random one-dimensional projections that preserve statistical dependencies while avoiding marginal-factorization or adversarial discriminators. Collectively, we believe that these advances make ISL a viable and competitive framework for modern generative modeling across univariate, multivariate, and heavy-tailed settings. A more detailed description of the contributions of the work is provided below.
\begin{enumerate}
    \item Divergence guarantees. We establish that the ISL, denoted \(d_K(p, \tilde{p})\), defines a valid statistical divergence. Moreover, when the generator model \(g_\theta(z)\), with $\theta$ a vector of parameters and $z$ a random input, satisfies mild regularity conditions—such as continuity and a uniform Lipschitz property—\(d_K(p, \tilde{p}_\theta)\) becomes a Lipschitz-continuous function of \(\theta\), and hence differentiable almost everywhere. These properties justify the use of \(d_K\) as a theoretically sound and optimizable objective for training implicit generative models.
    
    \item Pareto‐ISL for heavy tailed data distributions. We extend ISL to model heavy-tailed distributions by replacing the standard Gaussian input with generalized Pareto noise, following principles from extreme value theory \citep{coles2001introduction}. 
    This adaptation enables the generator to more accurately capture extreme‐value behavior in the tails while maintaining fidelity in the central region of the distribution. Empirical results show that Pareto-ISL models distributions exhibiting heavy-tailed behavior more faithfully, improving on both Gaussian-latent ISL and adversarial schemes.

    \item Scalable multivariate training with ISL-slicing. We introduce \emph{ISL-slicing}, a scalable divergence for high-dimensional data that computes one-dimensional ISL ranks along \(m\) random projections on the $(d-1)$-dimensional hypersphere (where $d$ is the dimension of the data space) and averages the results. ISL-slicing avoids the need for marginal factorization or adversarial training, maintains sensitivity to joint structure, and runs in \(\mathcal{O}(mNK)\) time (where $m$ is the number of random projections, $N$ is the batch size, and $K$ is the number of fictitious samples), making it efficient for high-dimensional settings.
    
    \item Empirical validation and GAN pretraining. We demonstrate the effectiveness of ISL methods across a range of benchmarks, showing performance on par with state-of-the-art implicit models. Furthermore, we show that ISL-slicing also serves as an effective pretraining objective for GANs, significantly reducing mode collapse and enhancing sample diversity. In our experiments, ISL pretraining enables simple GANs to outperform more sophisticated (and computationally costly) adversarial models.
\end{enumerate}

\subsection{Organization of the paper}

In Section \ref{preliminaries} we recall basic material on rank statistics from \cite{de2024training}, introduce the ISL divergence and its smooth surrogate, and show that it defines a valid statistical divergence that is differentiable in the generator parameters. Section \ref{section 4} presents Pareto‐ISL for heavy‐tailed modeling. In Section \ref{section Random Projections} we define ISL‐slicing—averaging 1D ISL over random projections—to extend the proposed divergence to multivariate settings and demonstrate its effectiveness across various benchmarks. Section \ref{section Invariant Statistical Loss for Time Series Prediction} applies these techniques to time‐series forecasting. Finally, Section \ref{section 7} is devoted to conclusions.

\section{Rank statistics and the invariant loss function} \label{preliminaries}

\subsection{Discrete uniform rank statistics}\label{Discrete uniform rank statistics}

Let $\tilde{y}_1, \ldots, \tilde{y}_K$ be a random sample from a univariate real distribution with pdf $\tilde{p}$ and let $y$ be a single random sample independently drawn from another distribution with pdf $p$. We construct the set,
\begin{align*}
\mathcal{A}_{K} := \Big\{\tilde{y} \in \{\tilde{y}_{k}\}_{k=1}^{K} : \tilde{y} \leq y\Big\},
\end{align*}
and the rank statistic 
\begin{align}\label{eq: rank statistic}
    A_{K} := |\mathcal{A}_{K}|,
\end{align}
i.e., $A_{K}$ is the number of elements in $\mathcal{A}_{K}$. The statistic $A_K$ is a discrete r.v. that takes values in the set $\{0, \ldots, K\}$; and we denote its pmf as $\mathbb{Q}_K: \{0,\ldots,K\} \mapsto [0,1]$. This pmf satisfies the following key result.
\begin{theorem}\label{theorem1}
If $p=\tilde{p}$ then $\mathbb{Q}_{K}(n)=\frac{1}{K+1}$ $\forall n \in \{0,\ldots, K\}$, i.e., $A_{K}$ is a discrete uniform r.v. on  the set $\{0,\ldots,K\}$.
\end{theorem}
\begin{proof}
    See \cite{elvira2021performance} for an explicit proof. This is a basic result that appears under different forms in the literature, e.g., in \cite{rosenblatt1952remarks} or \cite{djuric2010assessment}.
\end{proof}
The following result generalises Theorem 2 in \cite{de2024training}. It establishes the continuity of the rank statistic w.r.t. the \(L^1\) norm.

\begin{theorem}\label{theorem2}
If $\|p-\tilde{p}\|_{L^{1}(\mathbb{R})}\leq \epsilon$ then,
\begin{align*}
    \dfrac{1}{K+1}-\epsilon \leq \mathbb{Q}_{K}(n)\leq \dfrac{1}{K+1}+\epsilon, \;\forall n\in \{0, \ldots, K\}.
\end{align*}
\end{theorem}
\begin{proof}
See Appendix \ref{Proof of Theorem 2}.
\end{proof}

\begin{remark}
If \( p \) and \( \tilde{p} \) have compact support $\mathcal{K}\subset \mathbb{R}$, we can generalise the previous result. By Hölder's inequality \cite[Theorem 4.6]{brezis2011functional}, we readily see that
\begin{align*}
    \|p - \tilde{p}\|_{L^{1}(\mathbb{R})} = \|\mathbb{I}_{\mathcal{K}}\, (p - \tilde{p})\|_{L^{1}(\mathbb{R})} \leq \| \mathbb{I}_{\mathcal{K}} \|_{L^{q}(\mathbb{R})} \| p - \tilde{p} \|_{L^{q'}(\mathbb{R})},
\end{align*}
where \( \frac{1}{q} + \frac{1}{q'} = 1 \), \( q \in [1, \infty] \), \( \mathbb{I}_{\mathcal{K}} \) denotes the indicator function on \( \mathcal{K} \), and \( \| \mathbb{I}_{\mathcal{K}} \|_{L^{q}(\mathbb{R})} = \mathcal{L}(\mathcal{K})^{1/q} \), with \( \mathcal{L}(\mathcal{K}) \) being the Lebesgue measure of \( \mathcal{K} \). Therefore, if \( \|p - \tilde{p}\|_{L^{q'}(\mathbb{R})} \leq \epsilon/\mathcal{L}(\mathcal{K})^{1/q} \), then \( \|p - \tilde{p}\|_{L^{1}(\mathbb{R})} \leq \epsilon \). This implies that Theorem \ref{theorem2} holds for any $L^{q'}$ with $q'\in [1, \infty]$ and not only for $L^{1}$, provided that both $p$ and $\tilde{p}$ have compact support.
\end{remark}

So far, we have shown that if the pdf of the generative model, $\tilde{p}$, is close to the target pdf, $p$, then the statistic $A_{K}$ is close to uniform. A natural question to ask is whether $A_{K}$ displaying a uniform distribution implies that $\tilde{p}=p$. This result is also established in \cite{de2024training} and is reproduced here for convenience.

\begin{theorem}\label{theorem4}
    Let $p$ and $\tilde{p}$ be pdf's of univariate real r.v.'s and let $A_{K}$ be the rank statistic in Eq \ref{eq: rank statistic}. If $A_{K}$ has a discrete uniform distribution on $\{0, \ldots, K\}$ for every $K\in \mathbb{N}$ then $p=\tilde{p}$ almost everywhere (a.e.).
\end{theorem}

\begin{remark}
If $p$ and $\tilde{p}$ are continuous functions then Theorem \ref{theorem4} implies that $p=\tilde{p}$ (everywhere).
\end{remark}

Figure \ref{fig:evolution_ISL} illustrates how the generator density and the empirical rank‐statistic pmf $\mathbb{Q}_K(n)$ evolve over the course of a training process aimed at aligning $\tilde{p}$ with the target distribution $p$. As optimization proceeds, the generator density increasingly resembles the target, and the rank‐statistic distribution transitions from a skewed shape to the uniform distribution expected when $\tilde{p} = p$.

\begin{figure}[hbt!]
    \centering
    \includegraphics[width=\linewidth]{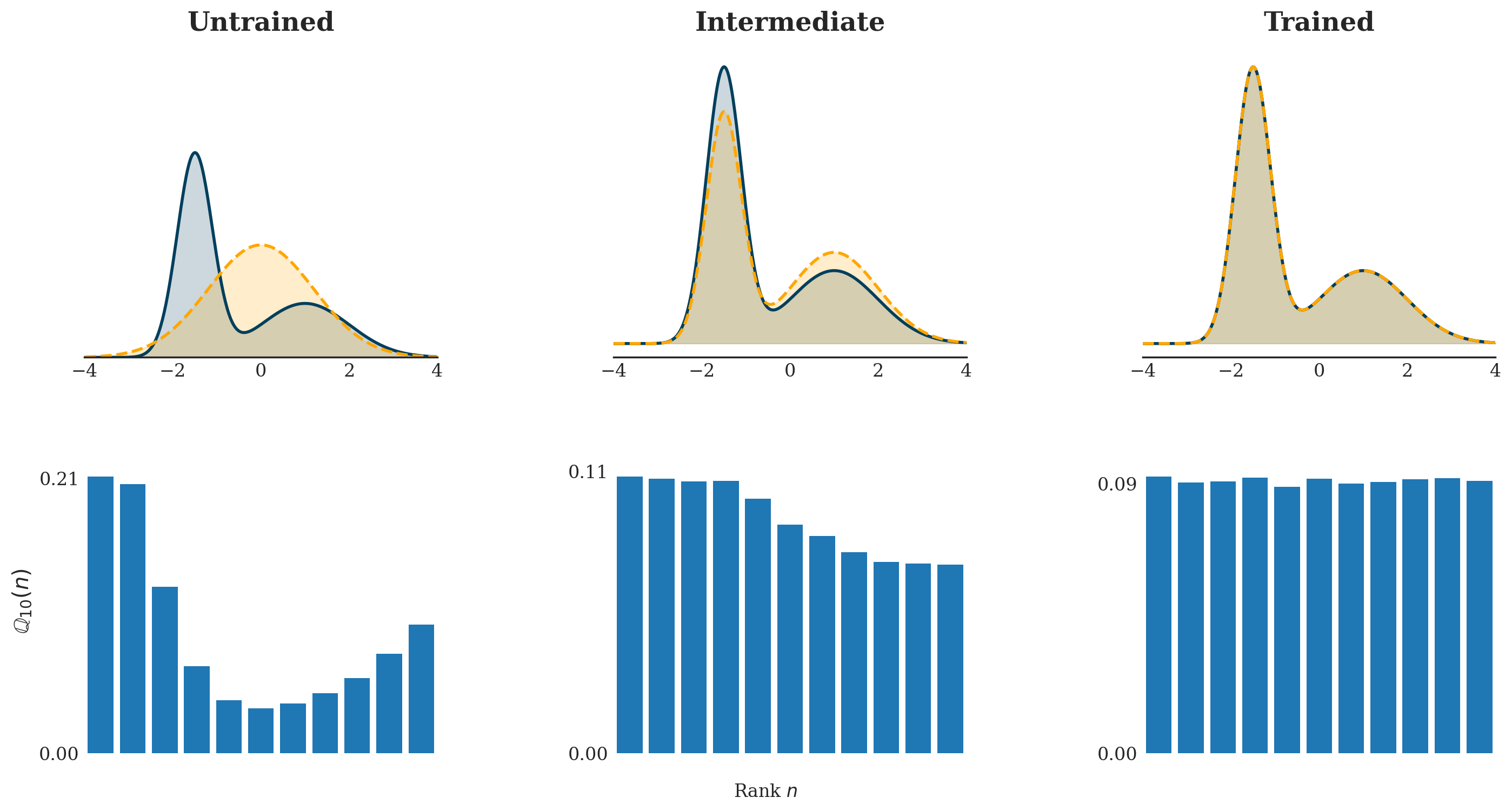}
    \caption{\small Evolution of generator density and rank‐statistic pmf over the course of a training process aimed to align $\tilde{p}$ with the target distribution $p$. Columns (left to right) show the untrained, intermediate, and trained stages.  \emph{Top row:} target density (solid blue) versus model density (dashed orange).  \emph{Bottom row:} empirical pmf \(\mathbb{Q}_K(n)\) of the rank statistic.  As training proceeds, the model density aligns with the target and \(\mathbb{Q}_K(n)\) converges to the uniform distribution.}
    \label{fig:evolution_ISL}
\end{figure}

\subsection{The invariant statistical loss function} \label{The invariant statistical loss function}
In this subsection we build upon previous results to introduce a new discrepancy function \(d_{K}(p, \tilde{p}): C(\mathcal{K}) \times C(\mathcal{K}) \mapsto [0, +\infty)\) between two continuous densities on a compact set $\mathcal{K}\subset \mathbb{R}$. This function measures the $\ell_1$-norm of the difference between the pmf $\mathbb{Q}_{K}$ associated with the statistic $A_{K}$ and the uniform pmf, i.e.,
\begin{align*}
    d_{K}(p, \tilde{p}) = \dfrac{1}{K+1}\norm{\dfrac{1}{K+1}\mathbf{1}_{K+1} - \mathbb{Q}_{K}}_{\ell_1} = \dfrac{1}{K+1}\sum_{n=0}^{K}\abs{\dfrac{1}{K+1}-\mathbb{Q}_{K}(n)}.
\end{align*}

It is clear that \( d_{K}(p, \tilde{p}) \geq 0 \) for any pair of pdf's \( p \) and \( \tilde{p} \). By Theorems \ref{theorem2} and \ref{theorem4}, $d_{K}(p,\tilde p)=0$ for arbitrarily large $K$ if and only if $p=\tilde p$. Thus, \( \lim_{K\to\infty} d_{K}(p, \tilde{p}) \) is a probability divergence \citep{chen2023sample, sugiyama2013direct}. Furthermore, Theorem \ref{theorem2} shows that $d_{K}$ is continuous w.r.t. the $L^{1}(\mathbb{R})$ norm, i.e., whenever $\|p-\tilde p\|_{L^{1}(\mathbb{R})}\le\epsilon$, one has $d_{K}(p,\tilde p)\le\epsilon$. This discrepancy measure serves as the theoretical loss function on which we build up a training procedure for implicit generative models.

The following theorem identifies two key regularity properties—continuity and differentiability—of the divergence $d_{K}(p,\tilde{p}_{\theta})$ w.r.t. the network parameters $\theta$ of $g_{\theta}$.

\begin{theorem} \label{theorem: props of d_k}
    Let $p:\mathcal{X}\to [0,\infty)$ be a pdf, where $\mathcal{X}\subseteq \mathbb{R}$. Let $Z$ be a real r.v. taking values in $\mathcal{Z}\subseteq \mathbb{R}$ and choose a function
    \begin{align*}
        g: \mathcal{Z} \times \mathbb{R}^{d} &\to \mathcal{X},\\
        (z, \theta) &\to g_{\theta}(z).
    \end{align*}
    Let $\tilde p_{\theta}$ denote the pdf of the r.v. $g_{\theta}(Z)$. Then,
    \begin{enumerate}
        \item If $g$ is continuous w.r.t. $\theta$ for almost every $x\in \mathcal{Z}$, then $d_{K}(p, \tilde{p}_{\theta})$ is also continuous w.r.t. $\theta$.
        \item Assume that $g_{\theta}(z)$ satisfies the Lipschitz condition w.r.t. $\theta$, i.e., $\left| g_{\theta}(z) - g_{\theta'}(z) \right| \leq L(z) \| \theta - \theta' \|$, and there is a constant  $L_{\max} < +\infty$ such that $L(z)<L_{\max}$ for almost every $z\in \mathcal{Z}$. If $g_{\theta}(z)$ is differentiable w.r.t. $z$ and there exists $m>0$ such that $\inf_{(z,\theta)\in \mathcal{Z}\times \mathbb{R}^{d}}\left| g_{\theta}'(z) \right| \geq m$, then $d_{K}(p, \tilde{p}_{\theta})$ is Lipschitz continuous w.r.t. $\theta$, and consequently, it is differentiable a.e.
    \end{enumerate}
\end{theorem}

\begin{proof}
    The proof of the Theorem can be found in Appendix \ref{Proof of Theorem 4}. 
\end{proof}

Theorem \ref{theorem: props of d_k} shows that the discrepancy \(d_{K}(p, \tilde{p}_{\theta})\), which measures the difference between a fixed density \(p\) and a parametric family \(\tilde{p}_{\theta}\) generated by \(g_{\theta}(z)\), is continuous whenever \(g_{\theta}(z)\) is continuous in \(\theta\). Additionally, if \(g_{\theta}\) is Lipschitz continuous and monotonic, the discrepancy becomes differentiable a.e. However, since the dependence of the empirical distribution \(\mathbb{Q}_{K}\) on \(\theta\) is unknown, gradient-based methods cannot be directly use to minimise $d_{K}(p, \tilde{p}_{\theta})$ w.r.t. $\theta$.

\subsection{The surrogate invariant statistical loss function}\label{surrogate Invariant-Statistic Loss}

Since the divergence $d_{K}(p,\tilde{p}_{\theta})$ cannot be used in practice, we present a surrogate loss function that is tractable, in the sense that it can be optimised w.r.t. the network parameters $\theta$ using standard methods. The training dataset consists of a set of $N$ i.i.d. samples, $y_1,\ldots,y_N$, from the true data distribution, $p$. For each $y_n$, we generate $K$ i.i.d. samples from the generative model, denoted by $\tilde{\mathbf{y}} = [\tilde{y}_1, \ldots, \tilde{y}_K]^\top$, where each $\tilde{y}_i = g_{\theta}(z_i)$ with $z_i \sim \mathcal{N}(0,1)$. From $y_n$ we obtain one sample of the r.v. \(A_K\), that we denote as \(a_{K,n}\).

We replace $d_{K}(p, \tilde{p}_{\theta})$ by a differentiable approximation and refer to this surrogate function as invariant statistical loss (ISL) \citep{de2024training}. The ISL mimics the construction of a histogram from the statistics \(a_{K,1}, a_{K,2}, \ldots, a_{K,N}\). Given a real data point $y_n$, we can tally how many of the $K$ simulated samples in $\tilde {\bf y}$ are less than the $n$-th observation $y_n$. Specifically, one computes
\begin{align*}
\tilde{a}_{K,n}(y)=\sum_{i=1}^{K}\sigma_\alpha(y_n - \tilde{y}_{i}) = \sum_{i=1}^{K}\sigma_\alpha(y_n - g_{\theta}(z_{i})),
\end{align*}
where $z_i \sim p_z$ is a sample from a univariate distribution, and $\sigma_\alpha(x) :=\sigma(\alpha x)$, with $\sigma(x):=1/\left(1 + \exp(-x) \right)$ being the sigmoid function. As we can see, $\tilde{a}_{K,n}$ is a differentiable (w.r.t. $\theta$) approximation of the actual statistic $A_K$ for the observation $y_n$. The parameter $\alpha$ enables us to adjust the slope of the sigmoid function to better approximate the (discrete) `counting' in the construction of $\tilde{a}_{K,n}$.

A differentiable surrogate histogram is constructed from \(\tilde{a}_{K,1}, \ldots, \tilde{a}_{K,n}\) by leveraging a sequence of differentiable functions. These functions are designed to mimic the bins around \(k \in \{0, \ldots, K\}\), replacing sharp bin edges with functions that replicate bin values at \(k\) and smoothly decay outside the neighborhood of \(k\). In our particular case, we consider radial basis function (RBF) kernels $\{\psi_{k}\}_{k=0}^{K}$ centered at $k\in \{0, \ldots,K\}$ with length-scale $\nu^2$, i.e., $\psi_k(a)=\exp(-(a-k)^2/2\nu^2)$.  Thus, the approximate normalized histogram count at bin $k$ is given by
\begin{align}
    q[k] = \frac{1}{N}\sum_{i=1}^{N}\psi_k(\tilde{a}_{K,i}(y_i)),
\end{align}
for $k=0,\ldots,K$. 
The ISL is computed as the $\ell$-norm distance between the uniform vector $\frac{1}{K+1}\mathbf{1}_{K+1}$ and the vector of empirical probabilities $\mathbf{q}=\left[q[0], q[1],\ldots, q[K]\right]^\top$, namely,
\begin{align}\label{eq:ISL}
    \mathcal{L}_{ISL}(\theta,K)
    := \left|\left|\frac{1}{K+1}\mathbf{1}_{K+1}-\mathbf{q}\right|\right|_{\ell_{1}}.
\end{align}
\begin{remark} 
The ISL is a sum and composition of \(C^{\infty}(\mathbb{R})\) functions. It is smooth w.r.t. the fictitious samples \(\{\tilde{y}_{i}\}_{i=0}^{K}\) and the data \(y_n\). As a result, its regularity aligns with that of the neural network (as a function of both its parameters and the input noise).
\end{remark}
The hyperparameters for the ISL method include the number of samples \(K\), which is tunable, the activation function $\sigma_\alpha$, and the set of basis functions \(\{\psi_k\}_{k=0}^{K}\), specified as radial basis function (RBF) kernels with a length scale of \(\nu^2\). These parameters control the flexibility and behavior of the model during learning.

To better illustrate the transition from discrete to differentiable rank-based histograms, Figure \ref{fig:rank_rbf_adjusted_arrows} provides a visual comparison. In Section \ref{subsection: Comparison of Surrogate and Theoretical Loss Performance}, we present a numerical study assessing how accurately the surrogate loss $\mathcal{L}_{\mathrm{ISL}}$ approximates the true divergence $d_{K}(p, \tilde{p}_{\theta})$. Finally, Figure \ref{fig:isl_1d_pipeline_merged} outlines the full 1D ISL training pipeline, summarizing the steps involved in computing and optimizing the surrogate loss.

\begin{figure}[ht]
  \centering
  \begin{tikzpicture}[
    scale=0.53,                 
    transform shape,           
    every node/.style={font=\sffamily\normalsize},
    numberline/.style={thin},
    bar/.style={fill=blue!50, draw=blue!70},
    dot/.style={circle, fill=blue!50, minimum size=6pt, inner sep=0pt},
    red/.style={circle, fill=red!70, minimum size=8pt, inner sep=0pt},
    arrow/.style={-{Latex[length=1.2mm,width=0.6mm]}, draw=gray!70, semithick}
  ]

    \begin{scope}[shift={(0,0)}]
      \draw[numberline] (-0.2,0) -- (5.2,0);
      \foreach \i/\lbl in {
        0/$\tilde y_1$,
        1/$\tilde y_2$,
        2/$\tilde y_3$,
        3/$\ldots$,
        4/$\tilde y_{K-1}$,
        5/$\tilde y_K$
      } {
        \node[dot] at (\i,0) {};
        \draw (\i,0) -- ++(0,-0.08) node[below] {\lbl};
      }
      \node[red] (num) at (1.4,0) {};
        \node[above=2pt of num] {\small $y$};
    \end{scope}

    \begin{scope}[shift={(7,0)}]
      \foreach \pos/\lbl in {
        0/0, 1/1, 2/2, 3/{\dots}, 4/{$K-1$}, 5/{$K$}
      } {
        \draw (\pos,0) -- ++(0,-0.06) node[below] {\lbl};
      }
      \draw[bar, opacity=0.4] (-0.4,0) rectangle ++(0.8,3.4);
      \draw[bar, opacity=0.4] (0.6,0) rectangle ++(0.8,2.7);
      \draw[bar, opacity=0.4] (1.6,0) rectangle ++(0.8,2.3);
      \draw[bar, opacity=0.4] (2.6,0) rectangle ++(0.8,0.2);
      \draw[bar, opacity=0.4] (3.6,0) rectangle ++(0.8,2.5);
      \draw[bar, opacity=0.4] (4.6,0) rectangle ++(0.8,3.2);
    \end{scope}

    \begin{scope}[shift={(16,0)}]
      \draw[thin] (-1.2,0) -- (6.2,0);
      \foreach \pos/\lbl in {
        0/0, 1/1, 2/2, 3/{\dots}, 4/{$K-1$}, 5/{$K$}
      } {
        \draw (\pos,0) -- ++(0,-0.06) node[below] {\lbl};
      }
      \begin{scope}[yscale=3.5]
        \filldraw[
          fill=blue!20,
          draw=blue!60,
          smooth,
          domain=-1.2:6.2,
          samples=100,
          opacity=0.8
        ] 
          plot (\x, {  
            1.0*exp(-((\x-0)^2)/(2*(0.3)^2))
            + 0.8*exp(-((\x-1)^2)/(2*(0.3)^2))
            + 0.6*exp(-((\x-2)^2)/(2*(0.3)^2))
            + 0.8*exp(-((\x-4)^2)/(2*(0.3)^2))
            + 1.0*exp(-((\x-5)^2)/(2*(0.3)^2))
          })
          -- (6.2,0) -- (-1.2,0) 
          -- cycle;
      \end{scope}
    \end{scope}

    \coordinate (histCent) at (10,1.0);
    \draw[arrow, bend right=20]
      ($(num.north)+(2,-0.8)$) to 
      node[midway, below, font=\sffamily\Large]
        {Differentiable (soft‐count) histogram: $q[n] = \tfrac{1}{N}\sum_{j=1}^N(\psi_n\,\cdot\,\tilde a_K(y_j)),\quad n=0,\dots,K$}
      ++(16, -0.2);

    \draw[arrow, bend left=35]
      ($(num.north)+(0,0.7)$) to 
      node[midway, above, font=\sffamily\Large]
        {Non differentiable (hard-count) histogram: $a_{K}(y)=\{\tilde y_i : \tilde y_i < y\}$}
      ++(7,1);

  \end{tikzpicture}
  \caption{\small
  Hard‐ vs. Soft‐Count Histograms via RBF‐Binning.
   \emph{Left:} Sorted model samples \(\tilde y_1 \le \cdots \le \tilde y_K\), with the observed value \(y\) highlighted in red.
  \emph{Middle:} Standard hard‐count histogram: each bin \(i\) tallies the number of model samples below \(y\), i.e.
  $a_{K}(y)\;=\;\bigl|\{\tilde y_j : \tilde y_j < y\}\bigr|$, yielding discrete, integer counts per bin.  \emph{Right:} Differentiable soft‐count histogram obtained by RBF‐binning, where each \(\tilde y_j\) contributes fractionally to every bin according to $q[n] \;=\;\frac{1}{N}\sum_{j=1}^N\psi_n\bigl(\tilde y_j,y\bigr)$ for each $n\in \{0,\ldots,K\}$, giving a smooth, continuous pmf.}
  \label{fig:rank_rbf_adjusted_arrows}
\end{figure}

\begin{figure}[ht]
  \centering
  \begin{tikzpicture}[
    scale=0.60, transform shape,
    node distance=1.2cm and 0.6cm,
    every node/.style={font=\sffamily\normalsize},
    box/.style={
      rectangle,
      rounded corners,
      draw=gray!70,
      thick,
      inner sep=6pt,
      align=center
    },
    arrow/.style={
      -{Latex[length=2mm,width=1mm]},   
      draw=gray!70,
      semithick,                        
      shorten >=1pt,                    
      shorten <=1pt
    }
  ]

    \node[box, fill=cyan!20] (dataLatent) {
      \textbf{1) Data \& Latent Sampling}\\
      Real: $\{\,y_i \sim p_{\mathrm{target}}\}_{i=1}^N$\\
      Latent: Draw $\{\,z_j \sim p_z\}_{j=1}^K$
    };

    \node[box, fill=yellow!20, right=of dataLatent] (soft) {
      \textbf{2) Soft‐Rank Computation}\\
      For each $y_i$ and $z_j$, compute 
      $\tilde y_{i,j} = g_\theta(z_j)$.\\[4pt]
      Soft count: 
      $\displaystyle \tilde A_K(x_i)
         = \sum_{j=1}^K 
           \sigma_\alpha\bigl(x_i - \tilde y_{i,j}\bigr)
         \;\in\;[0,K].$
    };

    \node[box, fill=teal!20, below=of soft] (rbf) {
      \textbf{3) RBF‐Binning to PMF}\\
      For $n=0,\dots,K$, 
      \\[4pt]
      Soft histogram:
      $\displaystyle 
        q[n] \;=\; \frac{1}{N}\sum_{i=1}^N \psi_n\!\bigl(\tilde A_K(x_i)\bigr).
      $
    };

    \node[box, fill=green!20, right=of soft] (loss) {
      \textbf{4) Compute ISL Loss}\\
      $\displaystyle 
        \mathrm{ISL}(\theta) 
        = 
        \sum_{n=0}^{K}
          \Bigl|\;q[n] - \tfrac{1}{K+1}\Bigr|
        \quad$
    };

    \node[box, fill=purple!20, right=of loss] (update) {
      \textbf{5) Update Generator}\\
      $\theta \;\leftarrow\; \theta - \eta\,\nabla_{\theta}\mathrm{ISL}$
    };

    \draw[arrow] (dataLatent) -- (soft);

    \draw[arrow, draw opacity=0.5] (soft.south) to 
      node[midway] {compute $\tilde A_K$} (rbf.north);

    \draw[arrow, draw opacity=0.5] (rbf.east) -- node[midway] {$\{q[n]\}$} (loss.south west);

    \draw[arrow] (loss) -- (update);

    \draw[arrow, bend right=30]
      (update.north) to node[midway, above] {\small repeat} (dataLatent.north);

    \begin{scope}[on background layer]
      \node[draw=gray!40, thick, fill=gray!5, rounded corners, inner sep=6pt,
            fit=(soft) (rbf)] {};
      \node at ($(soft.north)+(0,0.8)$) {\small \emph{Soft Count \& Binning}};
    \end{scope}

  \end{tikzpicture}
  \caption{\small 1D ISL Training Pipeline. First, real data samples $y_i\sim p_{\mathrm{target}}$ and latent inputs $z_j\sim p_z$ are drawn (Step 1). Next, for each pair $(y_i,z_j)$, the network output $\tilde y_{i,j}=g_\theta(z_j)$ is compared to $y_i$ via a smooth indicator function to produce \emph{soft counts} $\tilde A_K(x_i)\in[0,K]$ (Step 2). These counts are then converted into a \emph{continuous pseudo‐PMF} $q[n]$ over ranks $n=0,\dots,K$ using RBF kernels (Step 3). The ISL loss is computed as the $L^1$-distance between $q[n]$ and the uniform distribution (Step 4). Finally, $\theta$ is updated by gradient descent on this loss (Step 5), and the process repeats. This pipeline yields a fully differentiable surrogate for the invariant statistical loss suitable for end‐to‐end training.}\label{fig:isl_1d_pipeline_merged}
\end{figure}

\subsection{Progressively training by increasing $K$}

The training procedure can be made more efficient by performing it in a sequence of increasing values of \( K \) (see \cite{de2024training}). Specifically, one can select \( K^{(1)} < K^{(2)} < \cdots < K^{(I)} \), where \( I \) is the total number of stages and \( K^{(I)} = K_{\text{max}} \), the maximum admissible value of \( K \). The iterative training scheme is outlined in Algorithm \ref{pseudocode Progressively Training by Increasing K}. The gain in efficiency of the progressive training procedure compared to a scheme with fixed $K$ is illustrated in Appendix \ref{subsection: Efficiency Gains from Progressive $K$ Training vs. Fixed $K$}.

\begin{algorithm}
\setstretch{1.2}
\caption{Progressively training by increasing $K$} \label{pseudocode Progressively Training by Increasing K}
\begin{algorithmic}[1]
\STATE \textbf{Input} Neural network $g_{\theta}$; hyperparameters $N$; number of epochs; training data $\{y_{i}\}_{i=1}^{N}$; $K_{\max}$ maximum admissible value of $K$.
\STATE \textbf{Output} Trained neural network $g_{\theta}$.
\STATE $K=K^{(1)}$
\STATE \textbf{For} $t=1,\ldots, \operatorname{epochs}$ \textbf{do}
    \STATE \hspace{0.15in} Train $g_{\theta}$ using ISL loss function: $\mathcal{L}_{\text{ISL}}(\theta, K^{(i)})$
    \STATE \hspace{0.15in} Compute $\{a_{K^{(i)},1}, \ldots, a_{K^{(i)},N}\}$
    \STATE \hspace{0.15in} Compute Pearson $\chi^2$ test against $\mathbb{Q}_{K^{(i)}}$ using $\{a_{K^{(i)},1}, \ldots, a_{K^{(i)},N}\}$
        \STATE \hspace{0.15in} \textbf{If} hypothesis "$A_K$ is uniform" is accepted \textbf{ and } $K^{(i)} < K_{\max}$\textbf{ do}
            \STATE \hspace{0.30in} Set $K=K^{(i+1)}$
    \STATE \textbf{return} $g_{\theta}$
\end{algorithmic}
\end{algorithm}

\section{Pareto-ISL}  \label{section 4}

As shown in \cite{de2024training}, ISL outperforms other generative methods in learning the central regions of typical 1D distributions. However, Figure \ref{ISL_pareto_cauchy_Mixture} indicates that when standard Gaussian noise is used as an input, NNs struggle to capture the tails of Cauchy mixtures, since compactly supported inputs cannot produce unbounded pdfs (which we will refer to as unbounded distributions). This issue can be addressed by using input noise from a generalized Pareto distribution (GPD). In this section, we introduce Pareto-ISL, which utilizes a GPD for input noise, and demonstrate its effectiveness in learning heavy-tailed distributions.

\subsection{Tail distributions and extreme value theory}

The conditional excess distribution function \( F_u(y) \) provides a key tool for analyzing the tail of a distribution by focusing on exceedances over a specified threshold \( u \). By conditioning on large values, it isolates the behavior of the distribution in its tail, where extreme or rare events are more likely to occur. In the context of extreme value theory, for sufficiently high thresholds \( u \), the conditional excess distribution function converges to the GPD. The GPD, parameterized by the tail index \( \xi \) and scaling parameter \( \sigma \), provides a flexible model for the tail, allowing us to characterize its heaviness and the probability of extreme values.

The following definitions are taken from \cite{huster2021pareto} and provided here for convenience.

\begin{definition}
    The conditional excess distribution function with threshold $u\in \mathbb{R}$ is defined as
    \begin{align*}
        F_{u}(y) = \mathbb{P}(X - u \leq y | X > u ) = \dfrac{F(u+y) - F(u)}{1 - F(u)}.
    \end{align*}
\end{definition}

\begin{definition}
    The GPD, parametrized by tail index $\xi\in \mathbb{R}$ and scaling parameter $\sigma\in \mathbb{R}^{+}$, has the following complementary cumulative distribution function (CCDF)
    \begin{align*}
        S(z;\xi, \sigma) = \begin{cases}
(1 + \xi z/\sigma)^{-1/\xi}, & \text{for } \xi \neq 0, \\
e^{-z/\sigma}, & \text{for } \xi = 0.
\end{cases}
    \end{align*}
\end{definition}

Lipschitz continuous functions map bounded distributions to bounded ones \cite[Chapter 3]{evans2018measure}, limiting the ability of NNs to model heavy-tailed distributions. To address this, unbounded input distributions are required.

To construct unbounded NN generators, recall that piecewise linear (PWL) functions, (which include operations like rectified linear unit (ReLU), leaky ReLU, linear layers, addition, and batch normalization), are closed under composition \cite[Theorem 2.1]{arora2016understanding} and are unbounded, making them ideal for constructing generators that approximate heavy-tailed distributions.

We then define a Pareto-ISL generator, \( g^{PWL} \), as a piecewise linear generator driven by an input from a GPD with tail index \( \xi \), and trained using ISL. Estimators such as Hill’s \citep{resnick1997smoothing} can be used to estimate \( \xi \), aiding in the accurate modeling of heavy-tailed behavior.

\begin{definition}[Pareto-ISL]
Let \( z_{\xi} = (U^{-\xi} - 1)/\xi \), where \( U \sim \unifdist{0}{1} \), be a GPD r.v. with tail index \( \xi \) and CCDF \( S(x;\xi,1) \). 
A Pareto-ISL generator is $ g_{\theta}^{PWL} $, with an input distribution $z_{\xi} $ parameterized by $ \xi $, and output distribution $ y_{\xi} = g_{\theta}^{PWL}(z_{\xi}) $.
\end{definition}

The Pickands-Balkema-de Haan theorem \citep{balkema1974residual} states that for a wide range of probability distributions, the conditional excess distribution function converges to the GPD as the threshold \( u \) increases. This applies to distributions like Gaussian, Laplacian, Cauchy, Lévy, Student-t, and Pareto. Building on this result and \cite[Theorem 2]{huster2021pareto}, a generator constructed using \( g^{PWL} \) with a GPD input can effectively approximate the conditional excess distribution of heavy-tailed distributions. Specifically, if \( y_{\xi} \) has unbounded support, the conditional excess distribution \( F_u(y) \) of \( y_{\xi} \) converges to \( S(y; \xi, \sigma) \) for some \( \sigma \in \mathbb{R}^{+} \). This indicates that ISL-Pareto is especially well-suited for these types of problems, outperforming other implicit methods, including ISL with non-Pareto noise, as we demonstrate in the following subsection.

\subsection{Comparison of Pareto-ISL and standard ISL in learning a Cauchy mixture}

In Figure \ref{ISL_pareto_cauchy_Mixture}, we compare Pareto-ISL against standard ISL where the data distribution is a two-component Cauchy mixture with locations at $-1.0$ and $1.0$, and scales $0.7$ and $0.85$. All generators use a four-layer multilayer perceptron (MLP) with 35 units per layer and ReLU activation. Generators are trained with ISL using $K=20$, $N=1000$, and a learning rate of $10^{-3}$. For Pareto-ISL, the tail parameter is set to $\xi = 1$, aligning GPD noise with the tail index of the Cauchy mixture. Introducing GPD noise improves tail approximation and enhances the modeling of the central part of the target distribution, as demonstrated in the logarithmic-scale (bottom row) and linear-scale (top row) of Figure \ref{ISL_pareto_cauchy_Mixture}. 

In \ref{subsection: Multi-dimensional distributions heavy-tailed}, we present a multidimensional heavy-tailed distribution and compare Pareto-ISL to ISL with Gaussian noise (results shown in Figure \ref{fig:Approximanting_Multi-dimensional_distributions}).

\begin{figure}[h!]
    \centering
    \setlength{\tabcolsep}{2pt} 
    \renewcommand{\arraystretch}{1} 

    \begin{tabular}{c c c c c c}
        \multicolumn{2}{c}{\textbf{ISL}}
        &
        \multicolumn{2}{c}{\textbf{Pareto-ISL}} \\[0.2cm]
        
        \includegraphics[height=2.7cm]{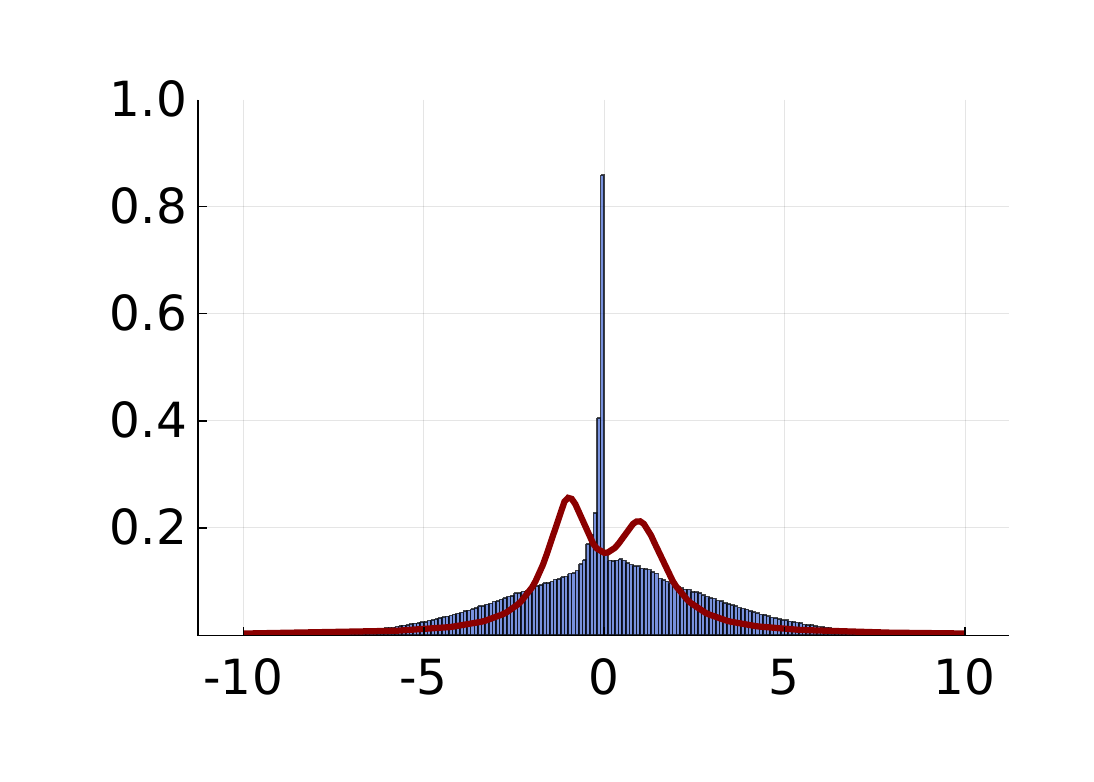} &
        \includegraphics[height=2.7cm]{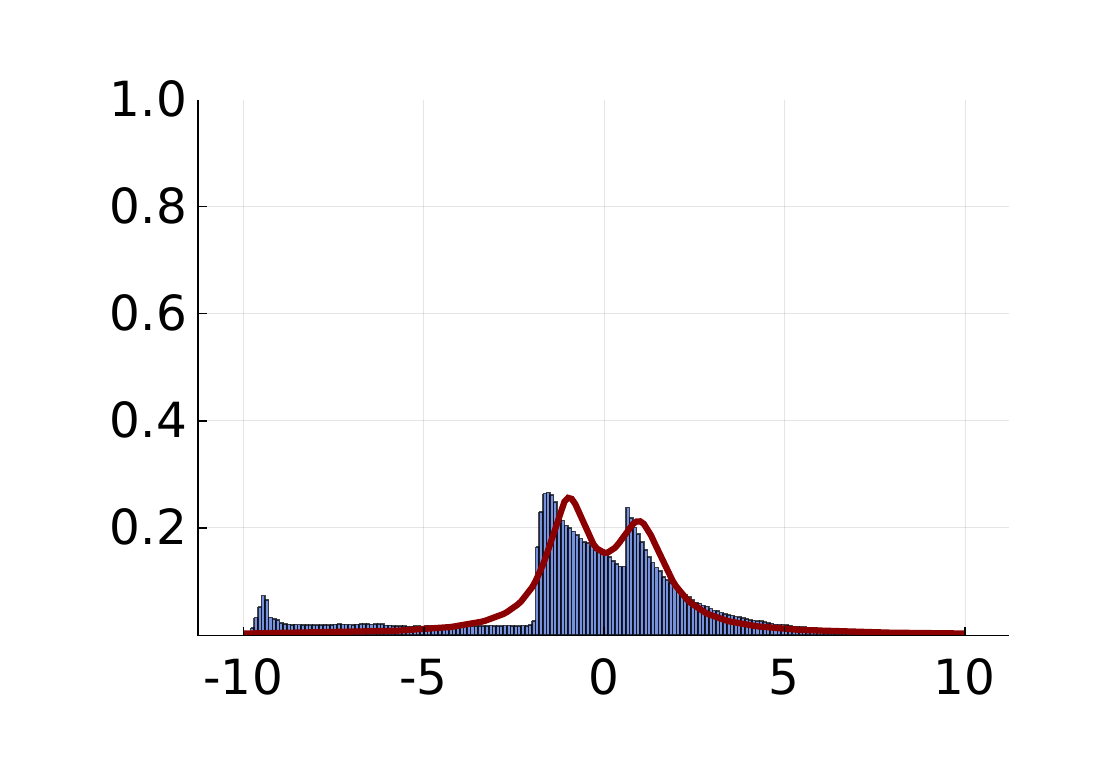} &
        \includegraphics[height=2.7cm]{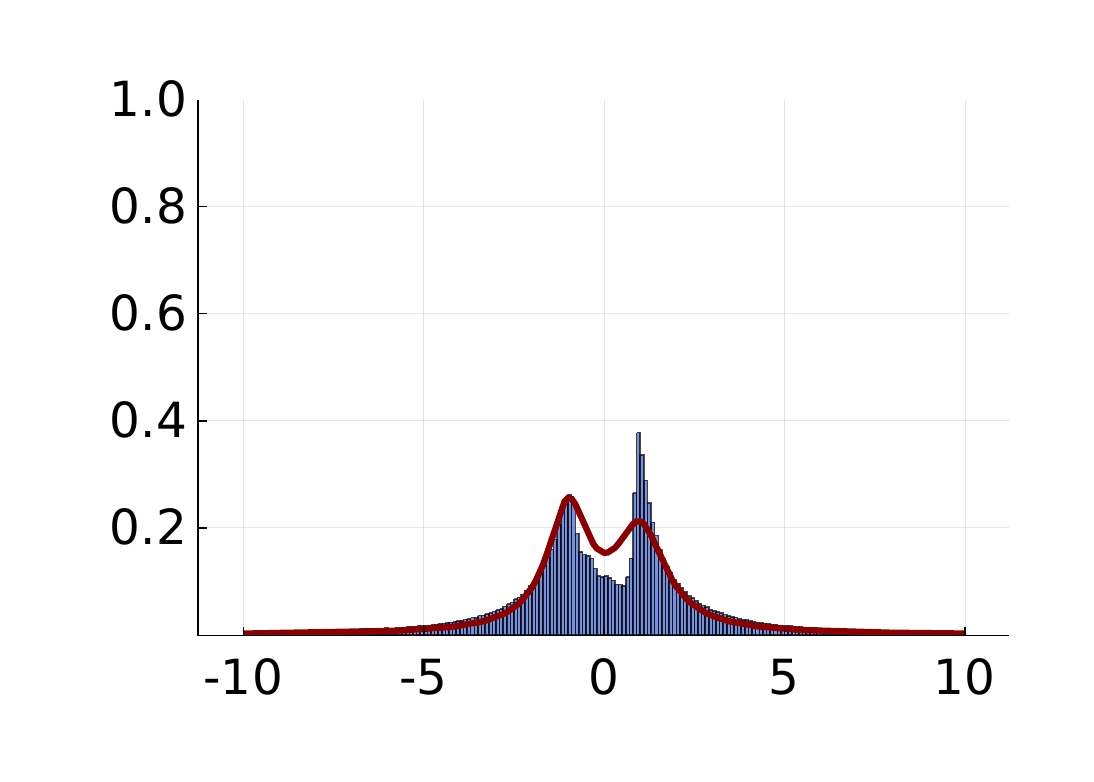} & \\

        \includegraphics[height=2.7cm]{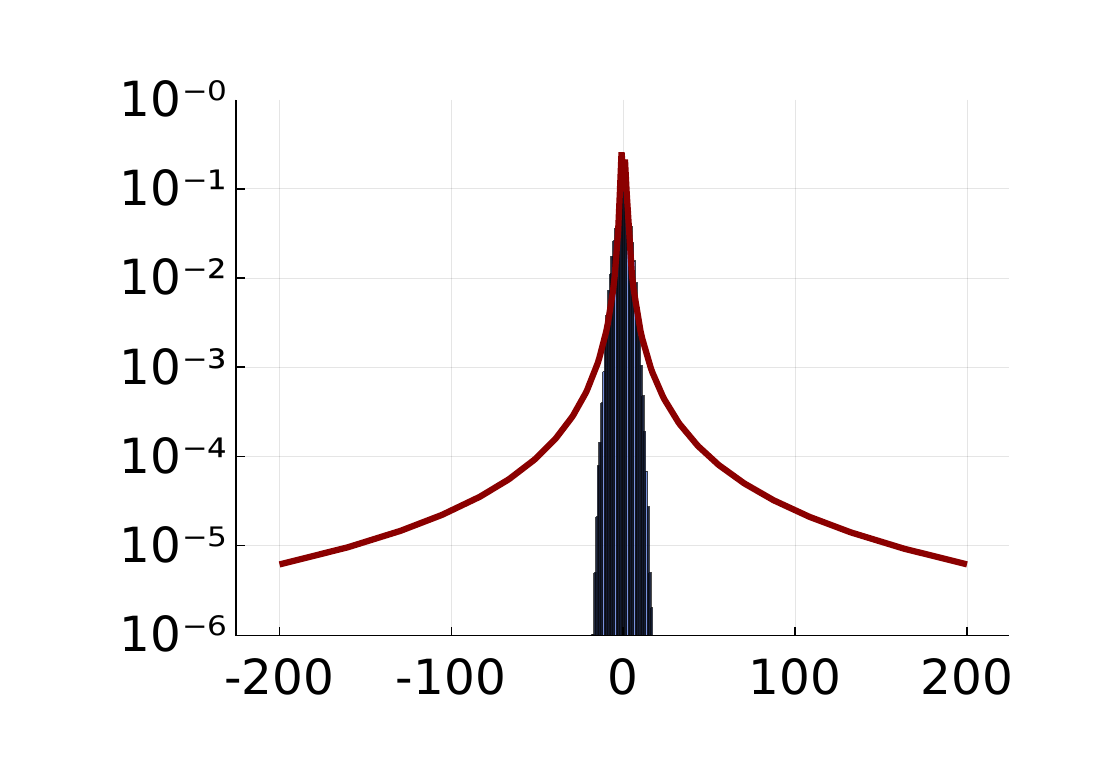} &
        \includegraphics[height=2.7cm]{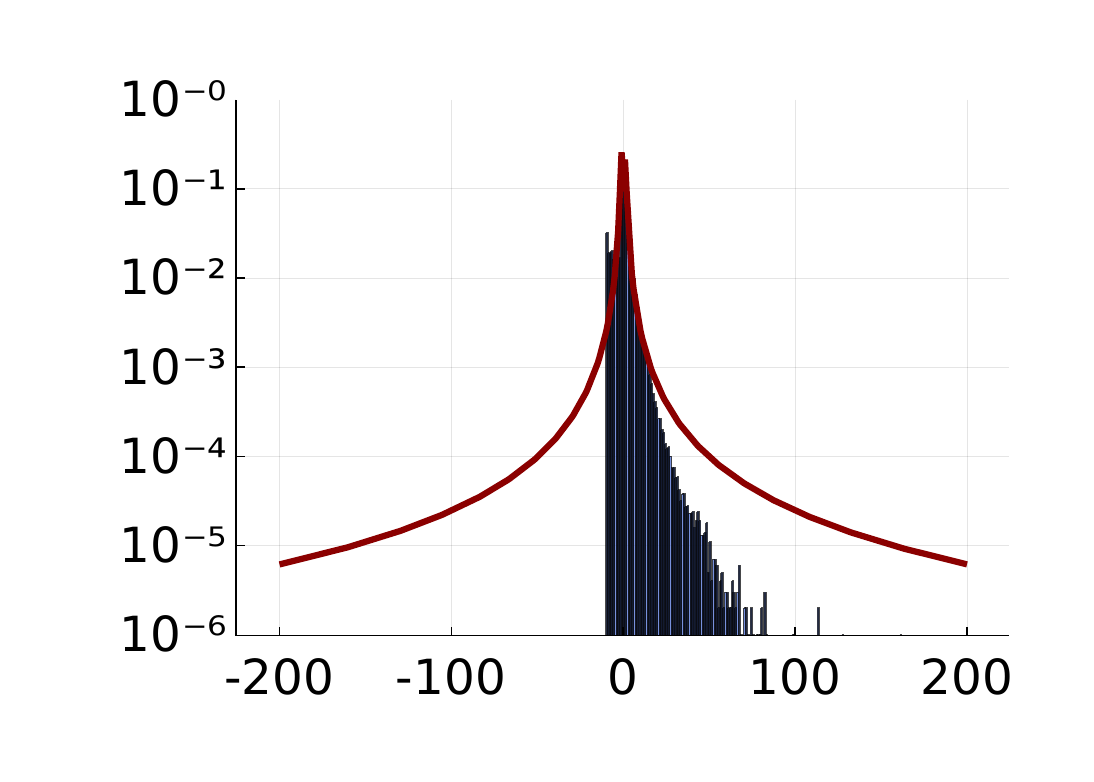} &
        \includegraphics[height=2.7cm]{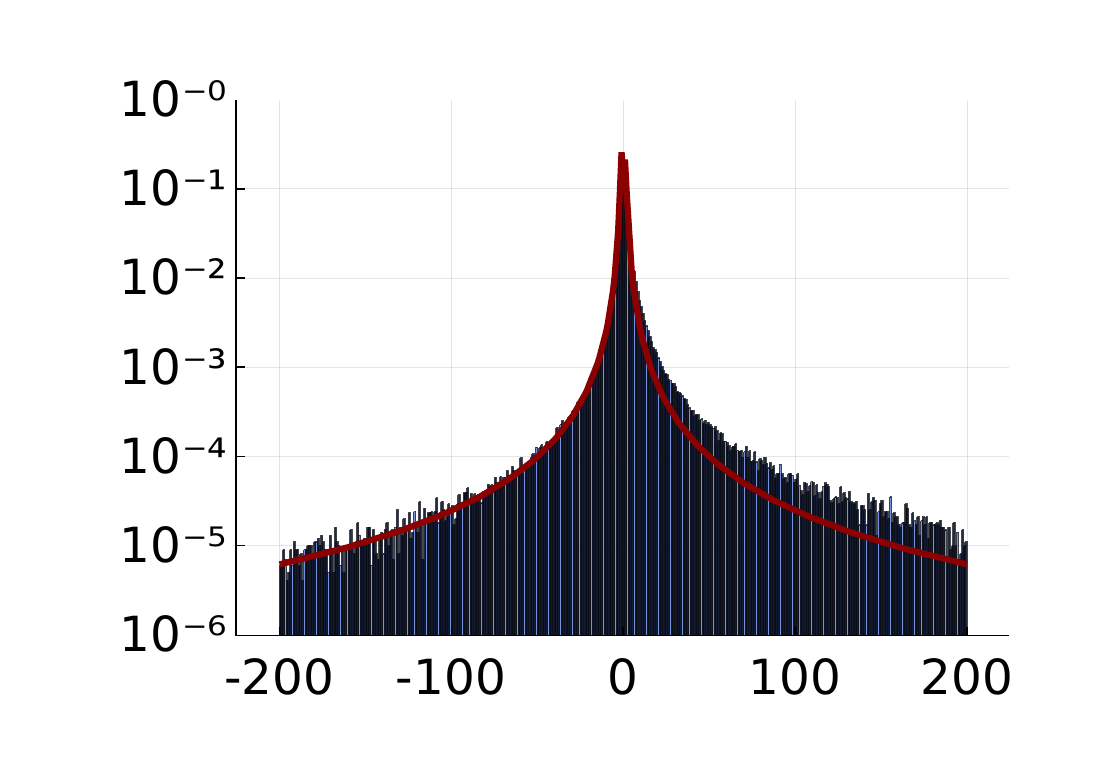} & \\[0.2cm]
        
        \small $\normdist{0}{1}$ & \small $\operatorname{LogNormal}(1, 1)$  & \small $\operatorname{GPD}(1)$ & \\

    \end{tabular}

    \caption{\small Pdfs approximated by generators with various tail behaviors. The noise input for each subplot is given by the subcaption of the corresponding column. Generators are trained on a mixture of Cauchy distributions, with the true density displayed in red. The top row presents the central region of the distributions on a linear scale, while the bottom row highlights the tails on a logarithmic scale.}
    \label{ISL_pareto_cauchy_Mixture}
\end{figure}

\subsection{Performance evaluation of Pareto-ISL compared to other implicit generative models}
In a second experiment, we evaluate the performance of Pareto-ISL as compared to different GANs. For this comparison, we consider four data distributions, including a Cauchy distribution with location parameter $1$ and scale parameter $2$ (labeled $\text{Cauchy}(1,2)$), a Pareto distribution with scale parameter $1$ and shape parameter $1$ (labeled $\text{Pareto}(1,1)$) and two mixture distributions, labeled $\text{Model}_3$ and $\text{Model}_4$. $\text{Model}_3$ is a mixture of a \(\mathcal{N}(-5, 2)\) distribution and a \(\text{Pareto}(5, 1)\) distribution, while $\text{Model}_4$ is a mixture of a $\text{Cauchy}(-1, 0.7)$ distribution and a $\text{Cauchy}(1, 0.85)$ distribution.

As generator NN, we use a $4$-layer MLP with $7, 13, 7$ and $1$ units at the corresponding layers. As activation function we use a ReLU. We train each setting up to $10^{4}$ epochs with $10^{-2}$ learning rate using Adam. We compare Pareto-ISL, GAN, Wasserstein GAN (WGAN) from \citep{arjovsky2017wasserstein}, and maximum mean discrepancy GAN (MMD-GAN) proposed in \cite{li2017mmd} using KSD (Kolmogorov-Smirnov distance), MAE (mean absolute error), and MSE (mean squared error) error metrics, defined as follows
{\footnotesize
\[
  \mathrm{KSD} = \sup_{x\in\mathbb R}\bigl|F(x) - \tilde F(x)\bigr|,\quad
  \mathrm{MAE} = \int_{\mathbb R}\bigl|g(z) - g_{\theta}(z)\bigr|\,p_{z}(z)\,\mathrm{d}z,\quad
  \mathrm{MSE} = \int_{\mathbb R}\bigl(g(z) - g_{\theta}(z)\bigr)^{2}\,p_{z}(z)\,\mathrm{d}z.
\]
}%
where \( F \) and \( \tilde{F} \) represent, respectively, the cdfs of the data distribution and the r.v. generated by the neural network \( g_{\theta} \) with input \( z \sim p_{z} \). Moreover, \( g \) denotes the optimal transformation of the data distribution.
The results are detailed in Table \ref{Results ISL table}.

\begin{table}[h]
  \centering
  \scriptsize
  \setlength{\tabcolsep}{3pt}
  \rowcolors{2}{gray!15}{white}
  \sisetup{detect-weight=true, table-number-alignment=center}
  \begin{tabular}{
    l
    S[table-format=1.2e1, scientific-notation=engineering, round-precision=2] 
    S[table-format=1.2,          round-precision=2]                       
    S[table-format=2.2,          round-precision=2]                       
    S[table-format=1.2,          round-precision=2]                       
    S[table-format=2.2,          round-precision=2]                       
    S[table-format=6.2,          round-precision=2]                       
    S[table-format=1.2,          round-precision=2]                       
    S[table-format=2.2,          round-precision=2]                       
    S[table-format=6.2,          round-precision=2]                       
    S[table-format=1.2,          round-precision=2]                       
    S[table-format=2.2,          round-precision=2]                       
    S[table-format=6.2,          round-precision=2]                       
  }
    \toprule
    \multirow{2}{*}{\textbf{Target}}
      & \multicolumn{3}{c}{\textbf{Pareto-ISL}}
      & \multicolumn{3}{c}{\textbf{GAN}}
      & \multicolumn{3}{c}{\textbf{WGAN}}
      & \multicolumn{3}{c}{\textbf{MMD-GAN}} \\
    \cmidrule(lr){2-4} \cmidrule(lr){5-7} \cmidrule(lr){8-10} \cmidrule(lr){11-13}
      & {KSD} & {MAE} & {MSE}
      & {KSD} & {MAE} & {MSE}
      & {KSD} & {MAE} & {MSE}
      & {KSD} & {MAE} & {MSE} \\
    \midrule
    Cauchy(1,2)
      & {\bfseries 1.90e-3} & {\bfseries 1.42} & {\bfseries 15.78}
      & 0.08              & 8.96             & 8207.00
      & 0.03              & 10.57            & 8127.00
      & 0.03              & 9.68             & 57975.00       \\
    Pareto(1,1)
      & {\bfseries 5.30e-3} & {\bfseries 1.16} & {\bfseries 2.41}
      & 0.10              & 12.64            & 114970.00
      & 0.49              & 7.64             & 7062.00
      & 0.50              & 9.02             & 10674.00      \\
    Model$_3$
      & {\bfseries 0.02}    & 0.47            & {\bfseries 0.33}
      & 0.19              & {\bfseries 0.45} & 1.61
      & 0.30              & 3.04             & 13.13
      & 0.56              & 3.20             & 21.52        \\
    Model$_4$
      & {\bfseries 0.01}    & {\bfseries 0.61} & {\bfseries 1.05}
      & 0.02              & 0.77             & 2.75
      & 0.03              & 2.63             & 66.06
      & 0.02              & 0.66             & 2.60         \\
    \bottomrule
  \end{tabular} 
  \caption{\small Comparison of Pareto-ISL results with vanilla GAN, WGAN, and MMD-GAN when input noise is a standard Gaussian, $K_{\max}=10$, epochs=1000, and $N=1000$. The best result for each metric is highlighted in bold.}\label{Results ISL table}
\end{table}

\subsection{Assessment of Pareto-ISL on heavy-tailed datasets}

In this third experiment, we demonstrate the effectiveness of the Pareto-ISL scheme on two univariate heavy-tailed datasets:
\begin{itemize}
    \item 136 million keystrokes (\texttt{Keystrokes}): this dataset includes inter-arrival times between keystrokes for a variety of users.
    \item Wikipedia web traffic (\texttt{Wiki Traffic}): this dataset includes the daily number of views of Wikipedia articles during $2015$ and $2016$. 
\end{itemize}
Our assessment uses two metrics. First, we compute the KSD to compare the data distribution and the generated distribution. Then, we calculate the area $\text{A}_{CCDF}$ between the log-log plots of the CCDFs of data and generated samples, indicating how well the tails of the distributions match. For $n$ real samples, we have
\begin{align*}
   \text{A}_{CCDF} = \sum_{i=1}^{n} \left[ \log \left( F_{p}^{-1} \left( \frac{i}{n} \right) \right) - \log \left( \tilde{F}_{\tilde{p}}^{-1} \left( \frac{i}{n} \right) \right) \right] \log \left( \frac{i+1}{i} \right),
\end{align*}
where \(F^{-1}_{p}\) and \(\tilde{F}^{-1}_{\tilde{p}}\) are the inverse empirical CCDFs for the data distribution \(p\) and the generated distribution \(\tilde{p}\), respectively.

We use a common network architecture and training
procedure for all experiments. The network consists
of $4$ fully connected layers with $32$ hidden units per layer and ReLU activations. Results are shown in Table \ref{Keystrokes_Wiki Traffic}.

\begin{table}[h]
  \centering
  \rowcolors{2}{gray!15}{white}
  \sisetup{
    detect-weight=true,
    table-number-alignment = center,
    round-mode=places,
  }
  \begin{tabular}{
    l
    S[table-format=1.3, round-precision=3]
    S[table-format=2.1, round-precision=2]
    S[table-format=1.3, round-precision=3]
    S[table-format=2.2, round-precision=2]
  }
    \toprule
    \multirow{2}{*}{\textbf{Method}}
      & \multicolumn{2}{c}{\textbf{Keystrokes}}
      & \multicolumn{2}{c}{\textbf{Wiki Traffic}} \\
    \cmidrule(lr){2-3} \cmidrule(lr){4-5}
      & {KS} & {$A_{\mathrm{CCDF}}$}
      & {KS} & {$A_{\mathrm{CCDF}}$} \\
    \midrule
    Uniform (ISL)
      & 0.087 & 6.2   & 0.025  & 10.3   \\
    Normal (ISL)
      & 0.090 & 2.7   & 0.023  & 8.6    \\
    Lognormal (ISL)
      & 0.096 & 1.7   & 0.019  & 9.5    \\
    Pareto GAN \citep{huster2021pareto}
      & 0.013 & 21.1  & 0.017  & 4.5    \\
    Gamma–Weibull KDE \citep{markovich2016light}
      & 0.050 & 1.7   & 0.075  & 1.5    \\
    \addlinespace
    \textbf{Pareto‐ISL}
      & {\bfseries 0.006} & {\bfseries 1.4}
      & {\bfseries 0.017} & {\bfseries 1.19} \\
    \bottomrule
  \end{tabular}
    \caption{\small Pareto‐ISL outperforms ISL variants (Uniform, Normal, Lognormal), Pareto GAN \citep{huster2021pareto}, and Gamma–Weibull KDE \citep{markovich2016light} in tail estimation ($A_{\mathrm{CCDF}}$), while also achieving lower KS distance.}
  \label{Keystrokes_Wiki Traffic}
\end{table}

\section{ISL-slicing: A random projections-based approach}  \label{section Random Projections}

Machine learning datasets are often multi-dimensional. Building on \citet{bonneel2015sliced} and \citet{kolouri2019generalized}, we extend the one-dimensional loss function $\mathcal{L}_{ISL}$ to a general metric for higher dimensions. We do this by randomly projecting high-dimensional data onto various 1D subspaces, specifically in all possible directions \(s \in \mathbb{S}^{d}\), where \(\mathbb{S}^{d}\) is the unit hypersphere in \(d+1\)-dimensional space.

Let $x$ be a $(d+1)$-dimensional r.v. and let $s\in \mathbb{R}^{d+1}$ be a deterministic vector. We denote by $s\#p$ the pdf of the real r.v. $y=s^\top x$. Using this notation we define the \textit{sliced ISL distance} between distributions with pdfs $p$ and $\tilde{p}$ as 
\begin{align*}
 d^{\;\mathbb{S}^{d}}_{K}(p, \tilde{p}) := \int_{s \in \mathbb{S}^{d}} d_{K}^{s}(p,\tilde{p}) \di s,   
\end{align*}
where $d_{K}^{s}(p,\tilde{p})=d_{K}(s\#p,s\#\tilde{p})$.

Since the expectation in the definition of the sliced ISL distance is computationally intractable, we approximate it using Monte Carlo sampling. Specifically, we choose a pdf \(q\) on \(\mathbb{S}^{d}\) and sample directions \(s_i \sim q\), for \(i = 1, \ldots, m\). Then, the Monte Carlo approximation of the sliced ISL distance is
\begin{align}\label{eq: sliced ISL distance MC}
    \tilde{d}_{K}^{\;\mathbb{S}^{d}}(p, \tilde{p}) = \dfrac{1}{m} \sum_{i=1}^{m} d_{K}(s_{i}\#p, s_{i}\#\tilde{p}).
\end{align}

If $\tilde{p}\equiv \tilde{p}_{\theta}$ is the pdf of the r.v. $y=g_{\theta}(z)$, $z\sim p_{z}$, i.e., the output of a NN with random input $z$ and parameters $\theta$, then one can use $\tilde{d}_{K}^{\;\mathbb{S}^{d}}(p, \tilde{p}_{\theta})$ as a loss function to train the NN.

\begin{remark}
In practice, we use the surrogate loss (see Section \ref{surrogate Invariant-Statistic Loss}) to approximate $d_K(s_{i}\#p, s_{i}\#\tilde{p}_{\theta})$ in Eq.\ref{eq: sliced ISL distance MC}, and sample random vectors from the unit sphere. 
Randomly chosen vectors from the unit sphere in a high-dimensional space are typically almost orthogonal. More precisely, (see \citet{gorban2018blessing})
\begin{align*}
\mathbb{P}\left(\left| \frac{v_{1}^{\top}v_{0}}{\|v_{1}\|_{2}\|v_{0}\|_{2}}\right| < \epsilon \right) > 1 - 2\,e^{-\frac{1}{2}d\epsilon^{2}},
\end{align*}
where \( v_0 \) and \( v_1 \) are uniformly distributed random vectors, and \( \epsilon \) is a small positive constant.
\end{remark}

The pseudocode for training implicit models using the proposed random projection method, referred to as the \textit{ISL-slicing algorithm}, is provided in Algorithm \ref{ISL Slicing Algorithm pseudocode}.

\begin{algorithm}[htb]
\setstretch{1.2}
\caption{ISL-slicing algorithm}\label{ISL Slicing Algorithm pseudocode}

\begin{algorithmic}[1]
\STATE \textbf{Input} Neural network $g_{\theta}$; hyperparameter $K$; number of epochs $N$; batch size $M$; training data $\{y_i\}_{i=1}^N$; number of randomly chosen projections $m$.
\STATE \textbf{Output} Trained neural network $g_{\theta}$.
\STATE \textbf{For} $t=1,\ldots, \operatorname{epochs}$ \textbf{do}
\STATE \hspace{0.15in} \textbf{For} $\text{iteration}=1,\ldots, N/M$ \textbf{do}
\STATE \hspace{0.30in} $L = 0$
\STATE \hspace{0.30in} Sample uniformly distributed random projection directions $\hat{\mathbb{S}}^{d} = \{s_{1:m}\}$
\STATE \hspace{0.30in} Select $M$ samples from $\{y_{j}\}_{j=1}^{N}$ at random
\STATE \hspace{0.30in} \textbf{For} each $s\in \hat{\mathbb{S}}^{d}$ \textbf{do}
\STATE \hspace{0.45in} $\{z_{i}\}^{K}_{i=1} \sim \normdist{\mathbf{0}}{I}$
\STATE \hspace{0.45in} $\mathbf{q} = \frac{1}{M}\sum_{j=1}^{M}\psi_k\Big(\sum_{i=1}^{K}\sigma_{\alpha}(s^{\top} y_{j} - s^{\top} g_{\theta}(z_i))\Big)$

\STATE \hspace{0.45in} $\operatorname{L} \leftarrow 
 \operatorname{L} + \nabla_{\theta}\left|\left|\frac{1}{K+1}\mathbf{1}_{K+1}-\mathbf{q}\right|\right|_2$
\STATE \hspace{0.30in} $\text{Backpropagation}(g_{\theta}, \nabla_{\theta}\operatorname{L}/m)$
\STATE \textbf{return} $g_{\theta}$
\vspace{0.15in}
\end{algorithmic}
\end{algorithm}

\subsection{Random projections vs. marginals on high-dimensional data}

Following \cite{saatci2017bayesian}, we conduct experiments on a multi-modal synthetic dataset. Specifically, we generate \(D\)-dimensional synthetic data from the model

\begin{equation*}
\begin{cases}
\begin{aligned}
\bm{z} &\sim \mathcal{N}(0, 10 \cdot I_{d}), \quad \bm{A} \sim \mathcal{N}(0, I_{D \times d}), \quad \epsilon \sim \mathcal{N}(0, 0.01 \cdot I_{D}), \\
\bm{x} &= \bm{A}\bm{z} + \epsilon, \quad d \ll D.
\end{aligned}
\end{cases}
\end{equation*}

In these experiments, we fit a standard GAN to a dataset where \( D = 100 \) and \( d = 2 \). The generator is a 3-layer fully connected neural network with $10$, $1000$, and $100$ units, respectively, using ReLU activations throughout. 

Figure \ref{fig:Synthetic Dataset res combined} illustrates the performance of random projections (ISL-slicing) when compared with that of the marginal pdf-based method from \cite{de2024training}. Our results show that even with a limited number of projections, ISL-slicing achieves a lower Jensen-Shannon divergence (\texttt{JS-Divergence}) w.r.t. the true distribution, as estimated by kernel density. This approach also significantly reduces computation time per iteration compared to ISL marginals, since the number of projections is a fraction of the total marginals. Specifically, it took $10$x and $5$x less execution time, respectively, to obtain the results shown in Figures \ref{fig:Synthetic Dataset res a} and \ref{fig:Synthetic Dataset res b} using the slicing method compared to the marginals method.

\begin{figure}[h]
    \centering
    \begin{subfigure}[t]{0.32\textwidth} 
        \centering
        \includegraphics[width=\textwidth]{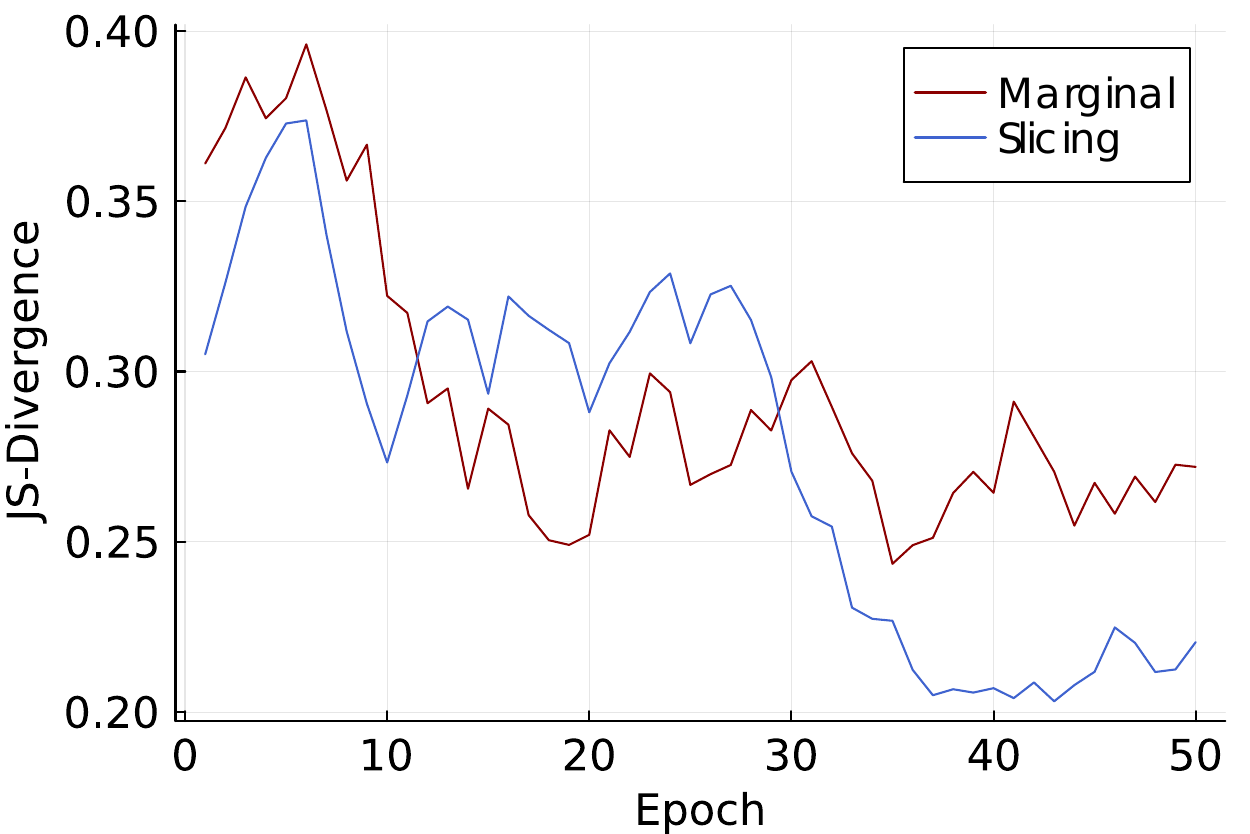}
        \caption{\small $m=10$ random projections.}
        \label{fig:Synthetic Dataset res a}
    \end{subfigure}
    \hfill
    \begin{subfigure}[t]{0.32\textwidth} 
        \centering
        \includegraphics[width=\textwidth]{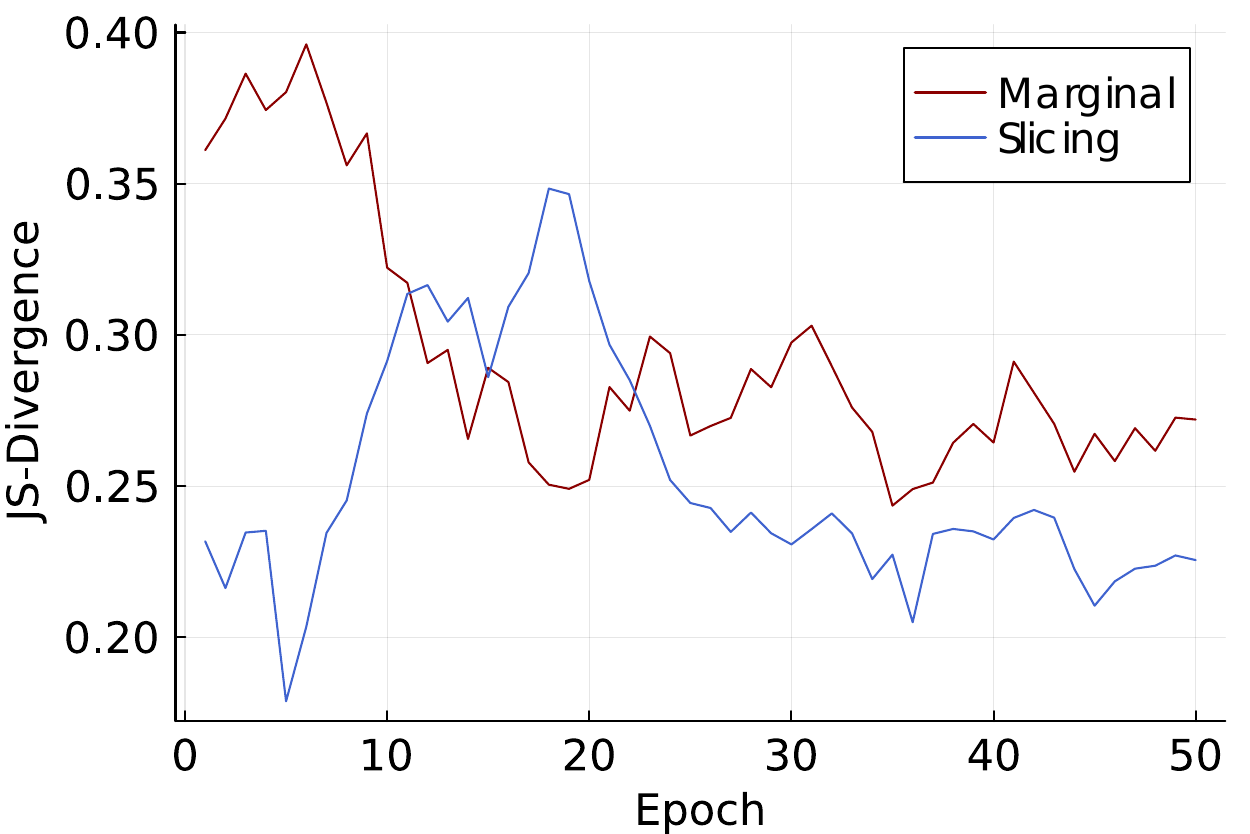}
        \caption{\small $m=20$ random projections.}
        \label{fig:Synthetic Dataset res b}
    \end{subfigure}
    \hfill
    \begin{subfigure}[t]{0.32\textwidth} 
        \centering
        \includegraphics[width=\textwidth]{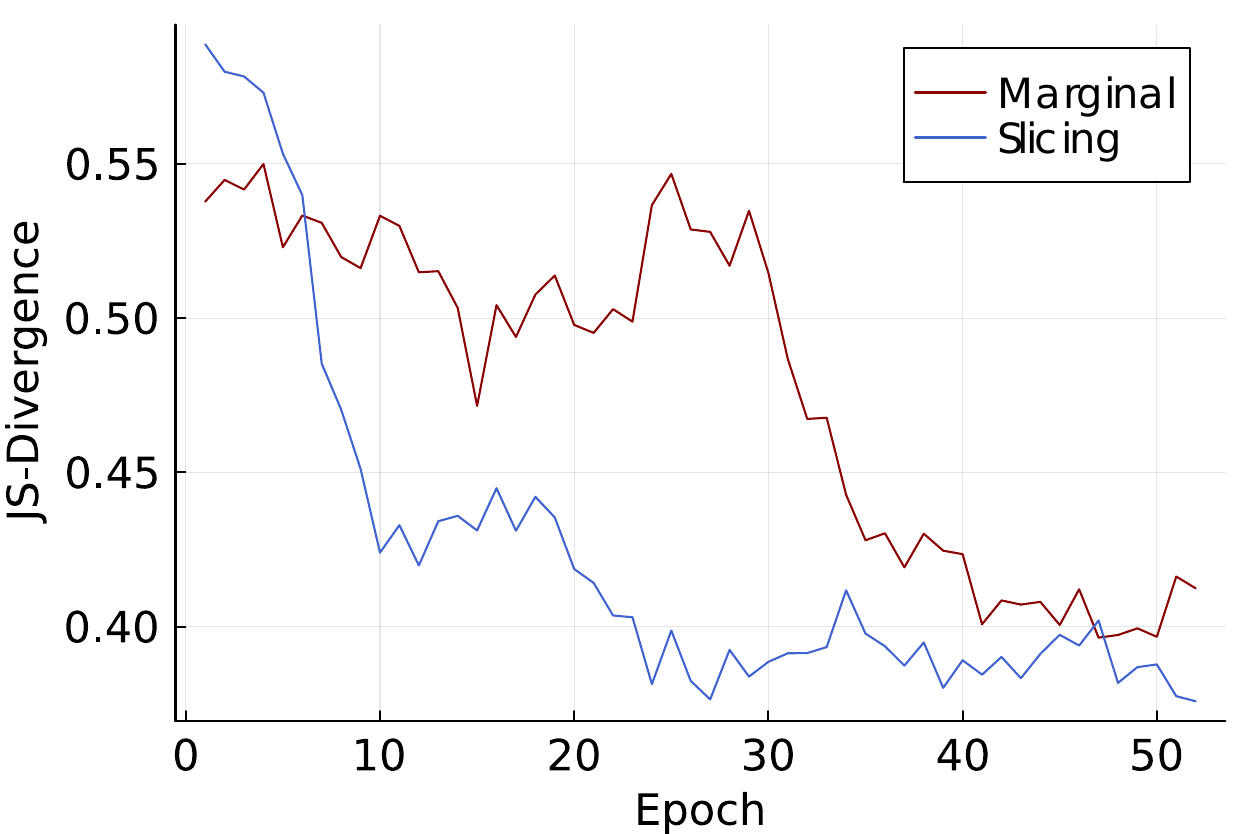}
        \caption{\small $D=500$, $m=10$ random projections.}
        \label{fig:image1}
    \end{subfigure}
    
    \caption{\small Performance evaluation of ISL-slicing vs ISL marginal methods on a synthetic dataset. Plots (a) and (b) correspond to \( D=100 \) with different random projections (\(m=10\) and \(m=20\)). Plot (c) corresponds to $D=500$ and $m=10$. The hyperparameters are \(K=10\), \(N=1000\) and learning rate of $10^{-4}$.}
    \label{fig:Synthetic Dataset res combined}
\end{figure}

\subsection{Experiments on 2D distributions} \label{Experiments on 2-D distributions}

We begin by examining simple 2D distributions characterized by different topological structures: one distribution with two modes, another with eight modes, and a third featuring two rings. Our objective is to assess the ability of the ISL-slicing method to fully capture the support of these distributions. We compare our approach to normalizing flows and GANs, using KL-divergence and visual assessment as metrics.

For GAN, WGAN, and ISL, we use a 4-layer MLP generator with a 2D input sampled from a standard normal distribution. Each layer has 32 units and uses the hyperbolic tangent activation. The discriminator is an MLP with 128 units per layer, using ReLU activations except for the final sigmoid layer. We used a batch size of 1000 and optimized the critic-to-generator update ratio over \(\{1\!:\!1,\,2\!:\!1,\,3\!:\!1,\,4\!:\!1,\,5\!:\!1\}\)
 for GAN and WGAN. The learning rate was chosen from \(\{10^{-2}, 10^{-3}, 10^{-4}, 10^{-5}\}\). For ISL, we set \(K=10\), \(N=1000\), $m=10$ random projections, and a learning rate of $10^{-3}$. For the normalizing flow model, we used the RealNVP architecture from \cite{dinh2016density}, with 4 layers of affine coupling blocks, parameterized by MLPs with two hidden layers of 32 units each. The learning rate was set to $5\cdot 10^{-5}$, using the implementation from \cite{Stimper2023}. All methods were trained for 1000 epochs, with optimization performed using the ADAM algorithm.

Figure \ref{experiments in 2d distributions figures} highlights the challenges of training GANs, particularly their susceptibility to mode collapse. In contrast, normalizing flow methods preserve topology via invertibility constraints but struggle to model complex structures, often forming a single connected component due to density filaments. Our method overcomes this by capturing the full distribution support and distinguishing between connected components, as seen in the \texttt{Dual Moon} example. However, ISL can fill regions between modes, as seen in the \texttt{Circle of Gaussians}, an issue mitigated by increasing \(K\). Alternatively, combining a pretrained network with ISL (trained for 100 epochs) and a GAN yielded the best results (method denoted as ISL+GAN), capturing full support while excluding zero-density regions. This approach is detailed further in Section \ref{ISL-Pretrained GANs for Robust Mode Coverage in 2D Grids}. We also estimate the KL-divergences between the target and model distributions, as listed in Table \ref{table: experiments in 2d distributions figures}. In all cases, the ISL method, and particularly the ISL+GAN approach, outperform the respective baselines.

\begin{figure}[!h]
  \centering
  \footnotesize
  \setlength{\tabcolsep}{1.5pt}       
  \renewcommand{\arraystretch}{0.95}  
  \setkeys{Gin}{height=1.5cm}

  \begin{tabular}{@{} r | *{6}{c} @{}}
    \toprule
    \rowcolor{gray!30}
      & \textbf{Target}
      & \textbf{Real NVP}
      & \textbf{WGAN}
      & \textbf{GAN}
      & \textbf{ISL}
      & \textbf{ISL+GAN} \\
    \midrule

    \rotatebox{90}{\scriptsize\textbf{Moon}}
      & \includegraphics{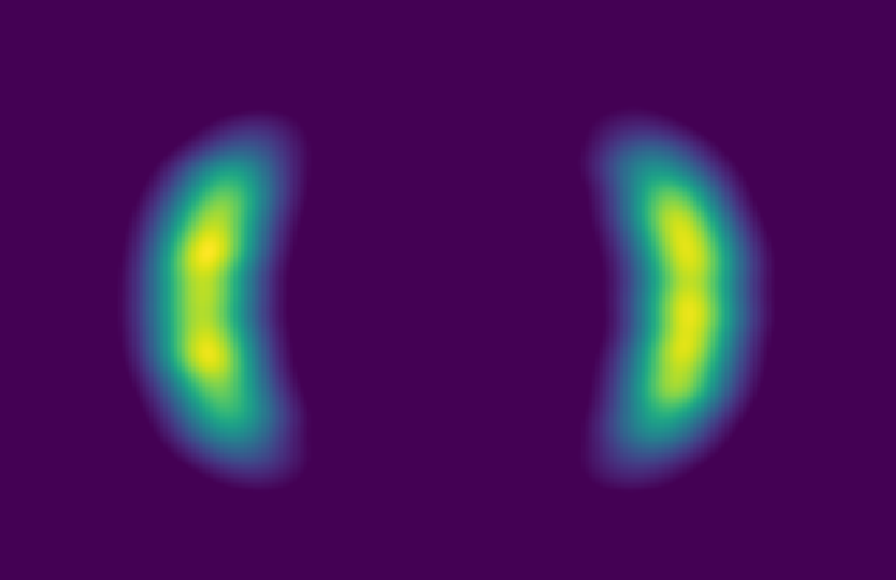}
      & \includegraphics{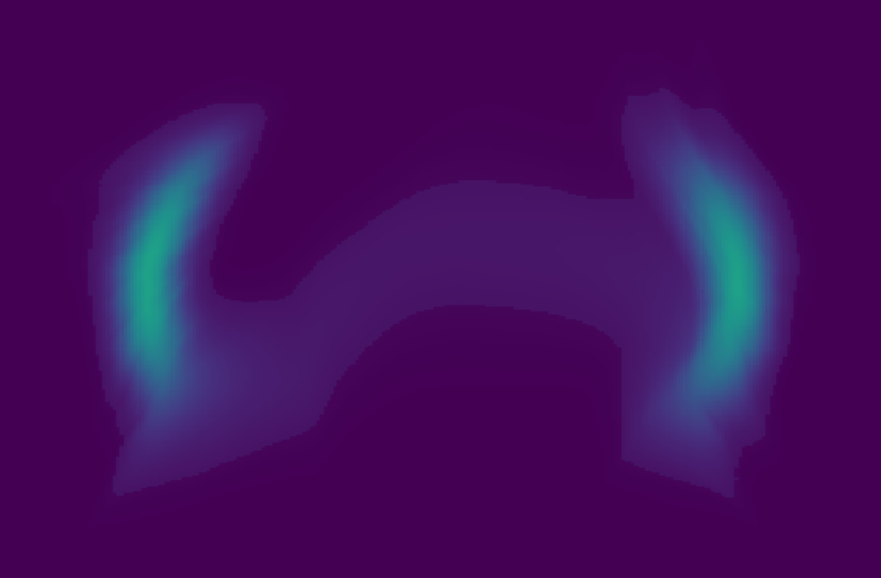}
      & \includegraphics{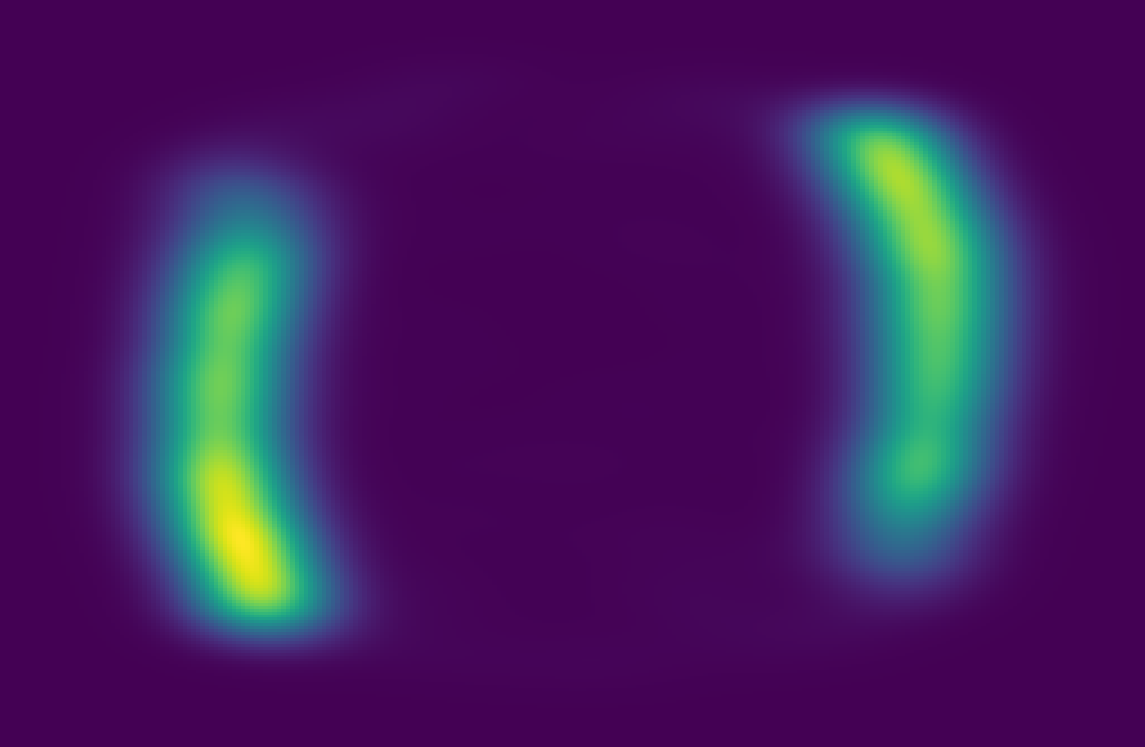}
      & \includegraphics{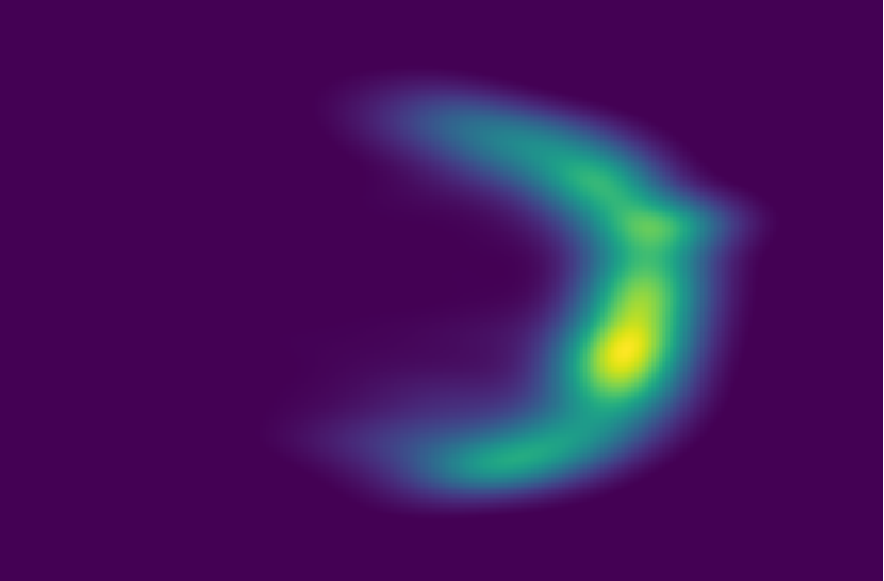}
      & \includegraphics{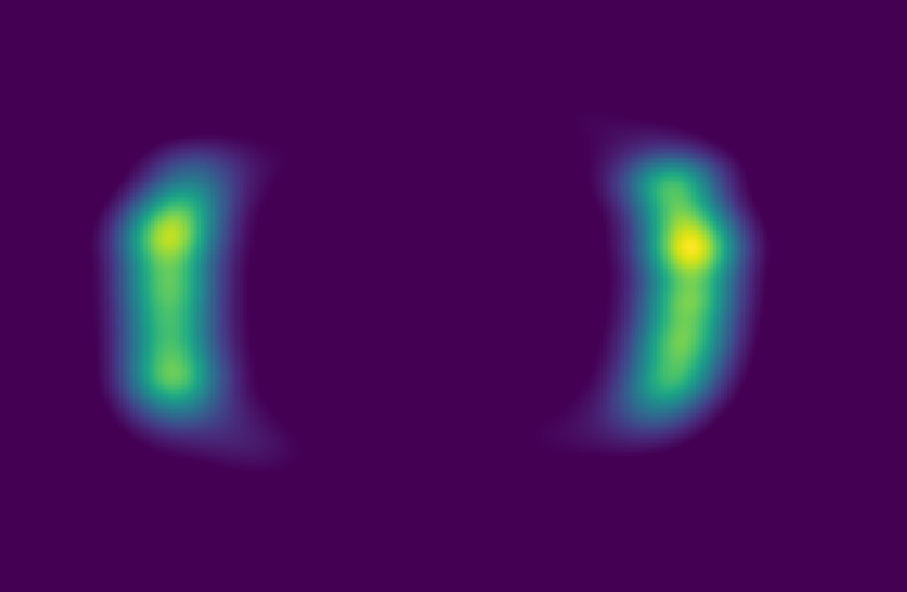}
      & \includegraphics{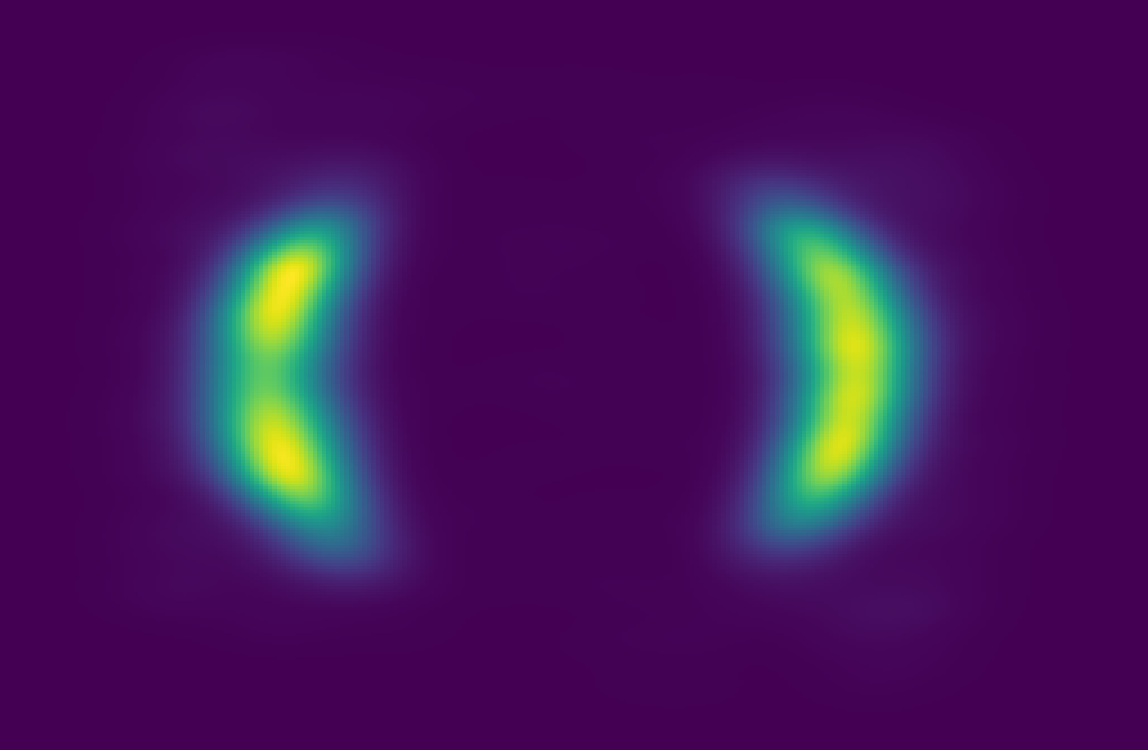}
    \\[0pt]

    \rotatebox{90}{\scriptsize\textbf{Gauss}}
      & \includegraphics{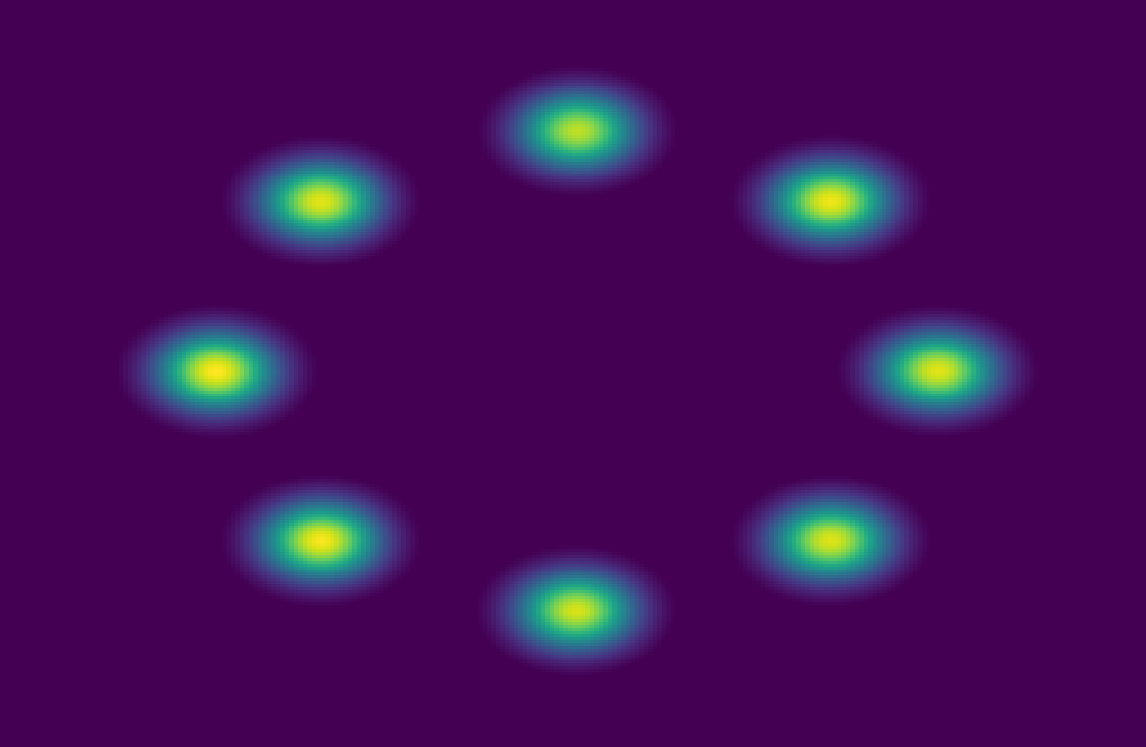}
      & \includegraphics{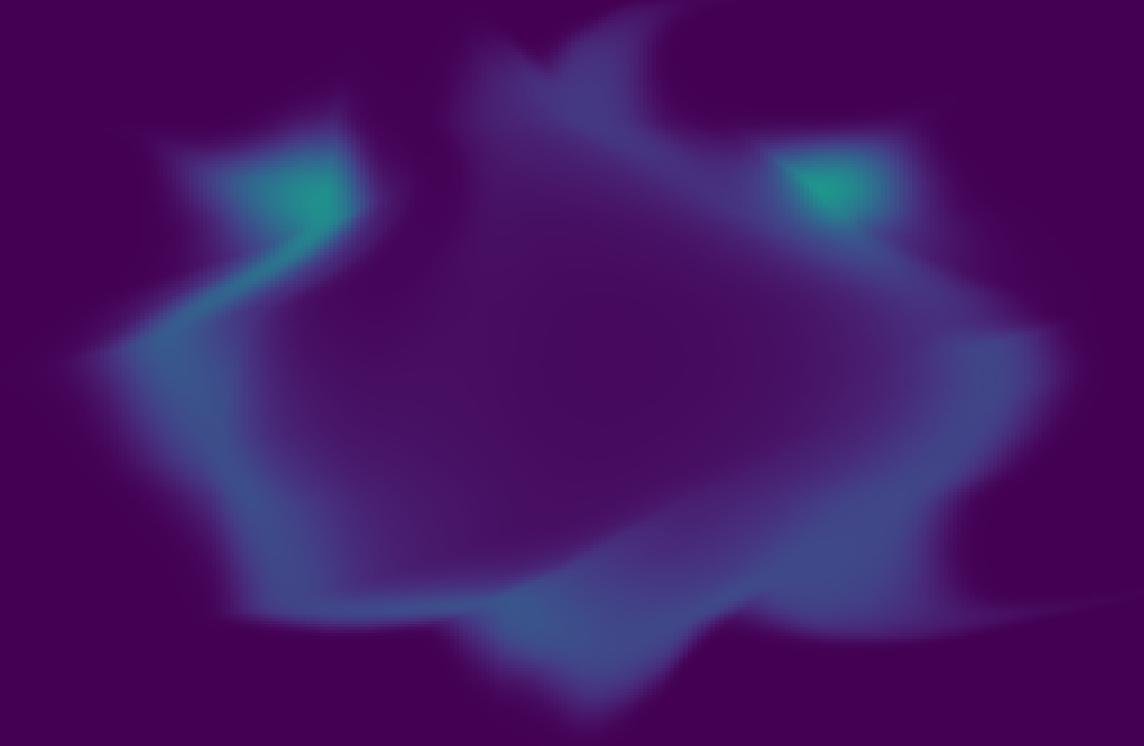}
      & \includegraphics{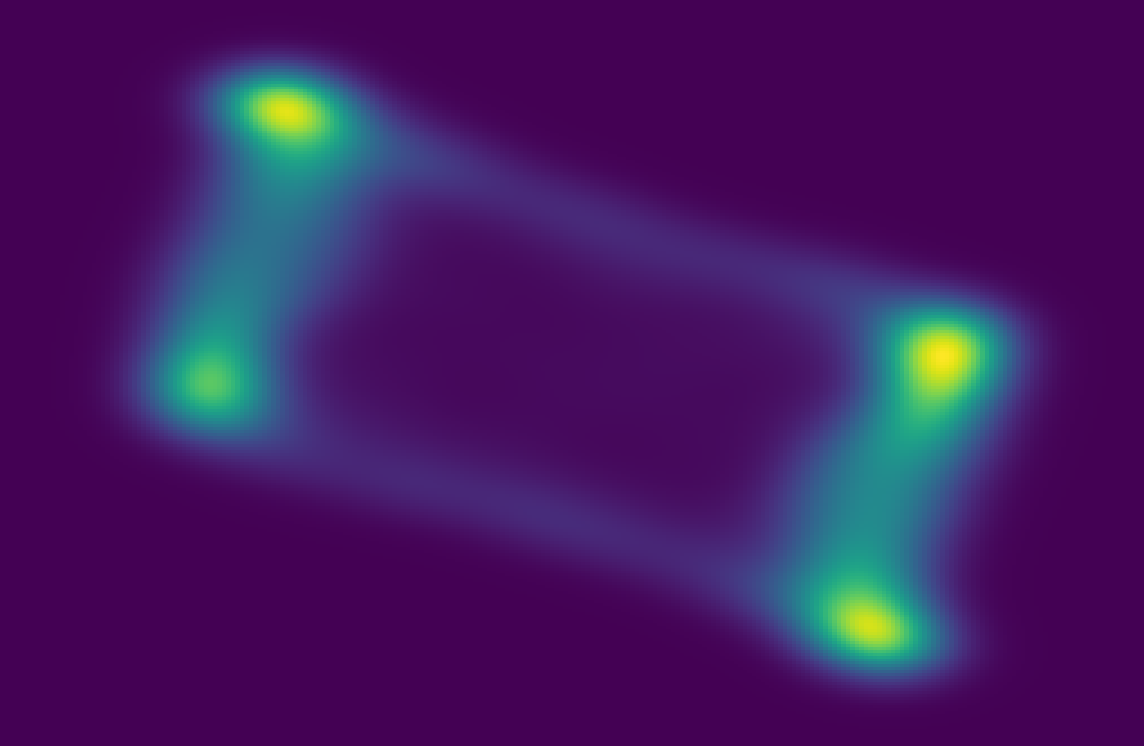}
      & \includegraphics{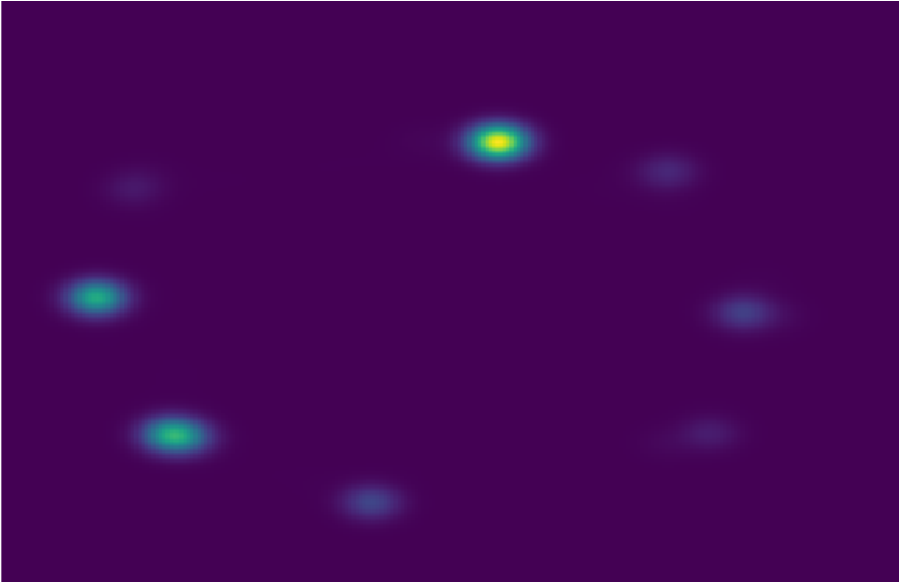}
      & \includegraphics{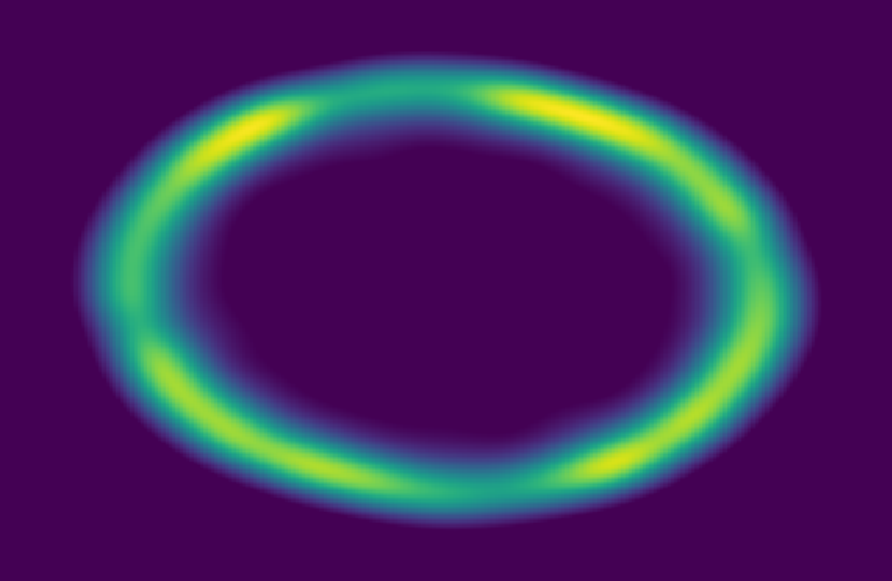}
      & \includegraphics{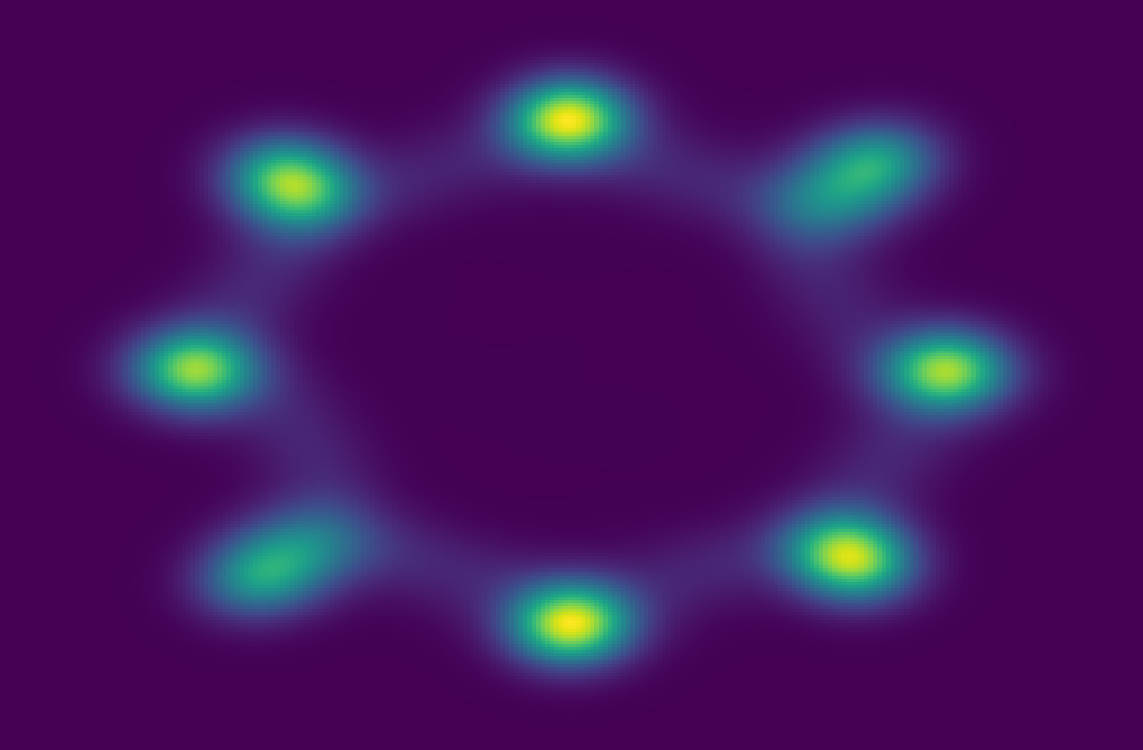}
    \\[2pt]

    \rotatebox{90}{\scriptsize\textbf{Rings}}
      & \includegraphics{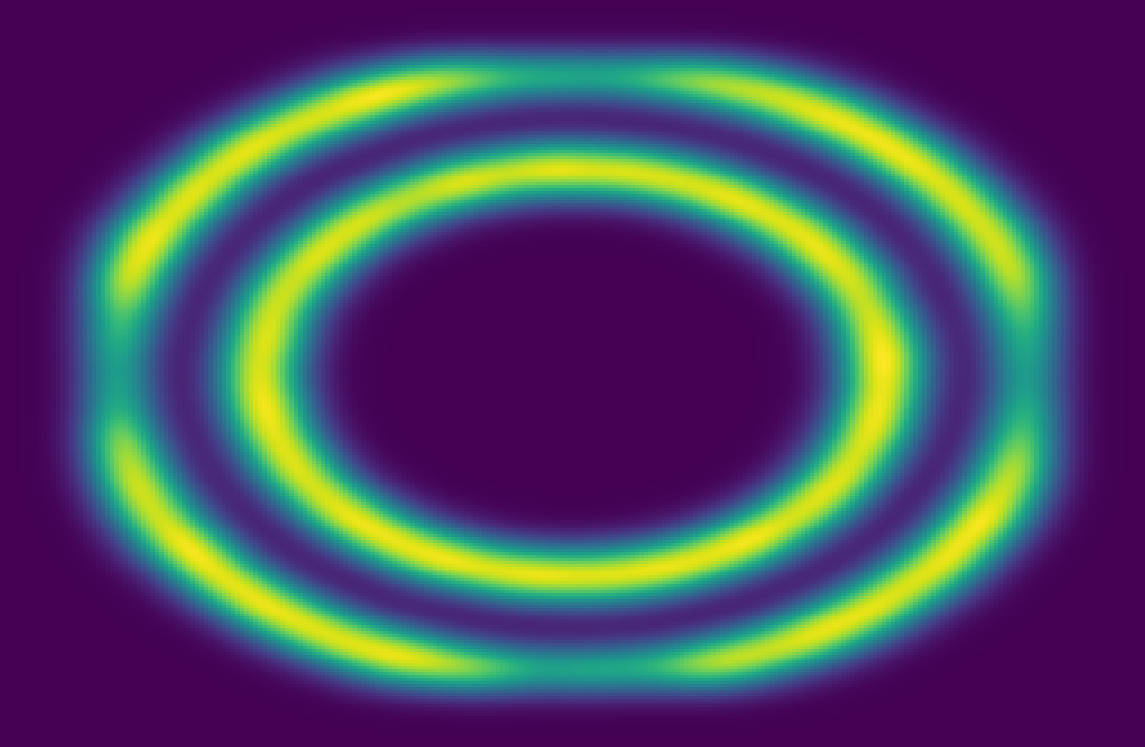}
      & \includegraphics{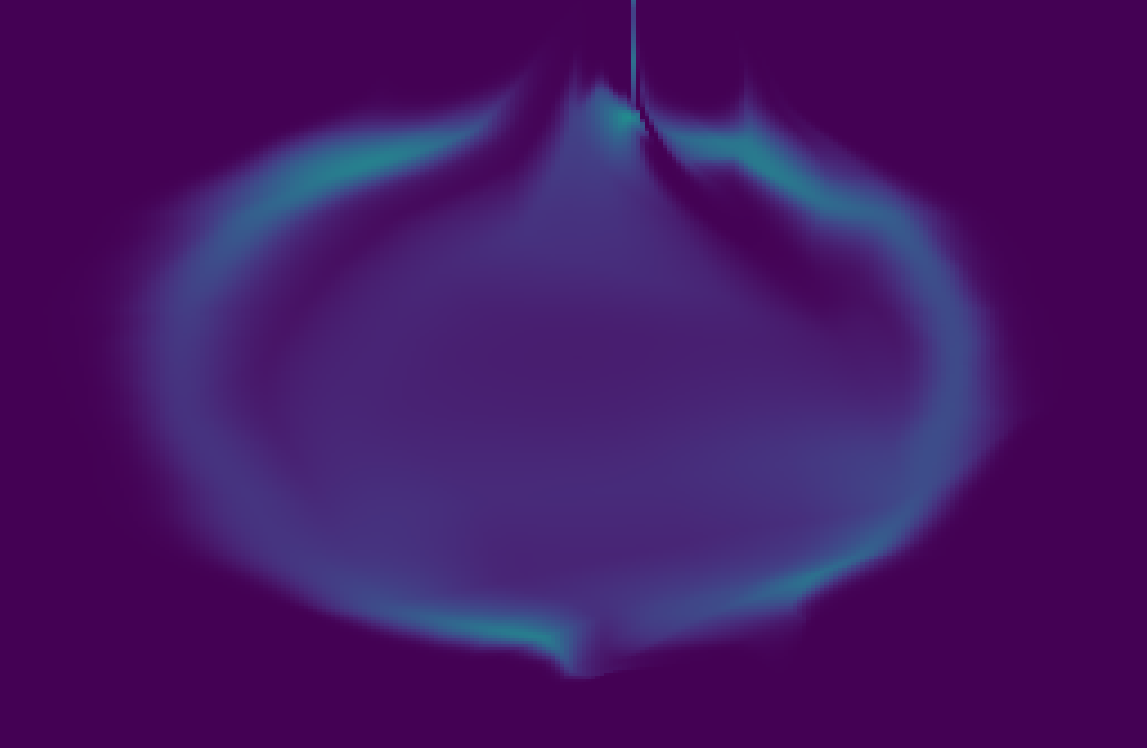}
      & \includegraphics{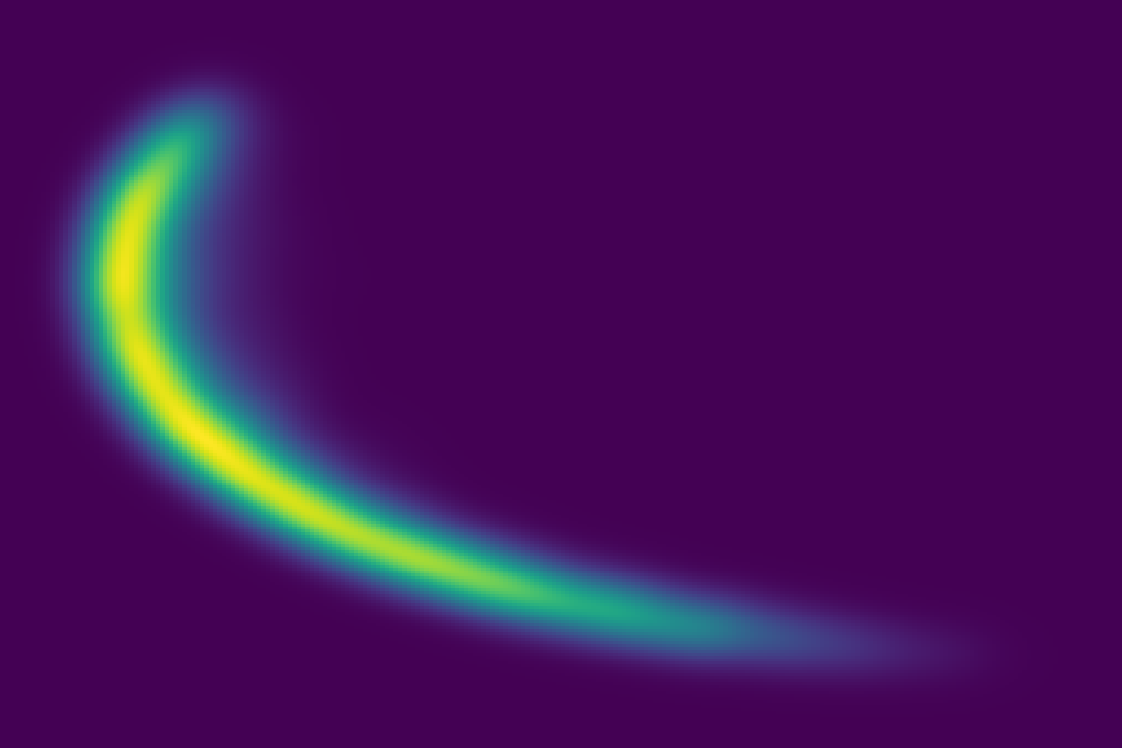}
      & \includegraphics{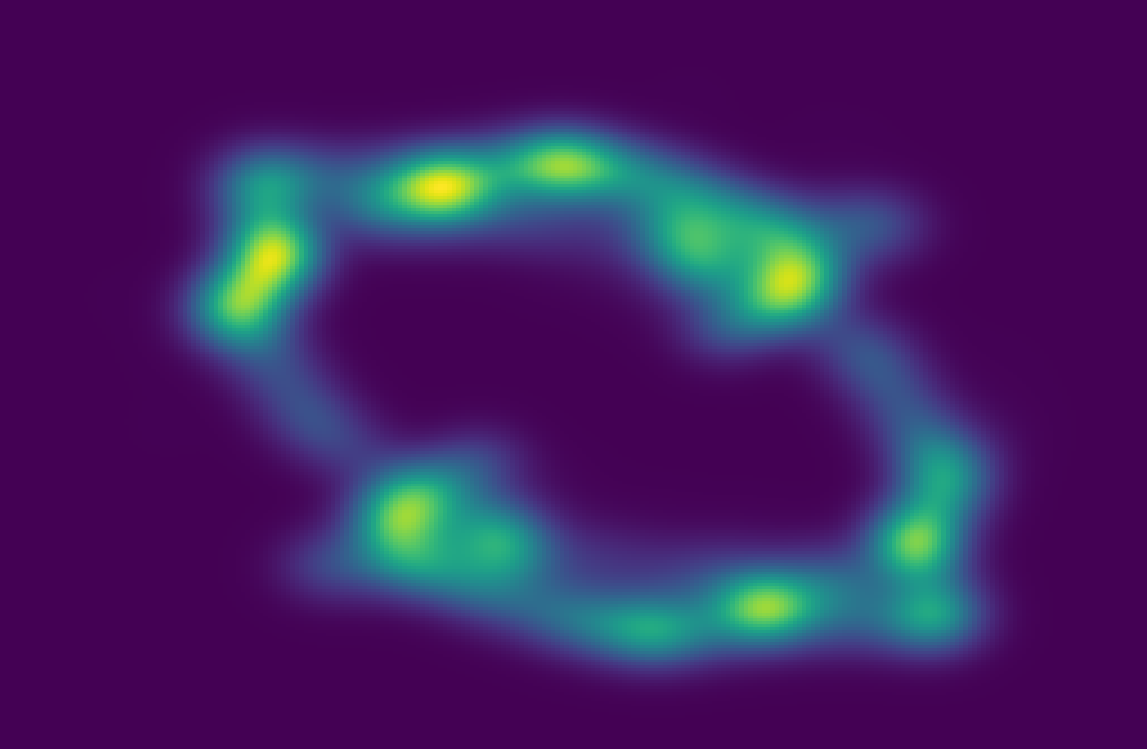}
      & \includegraphics{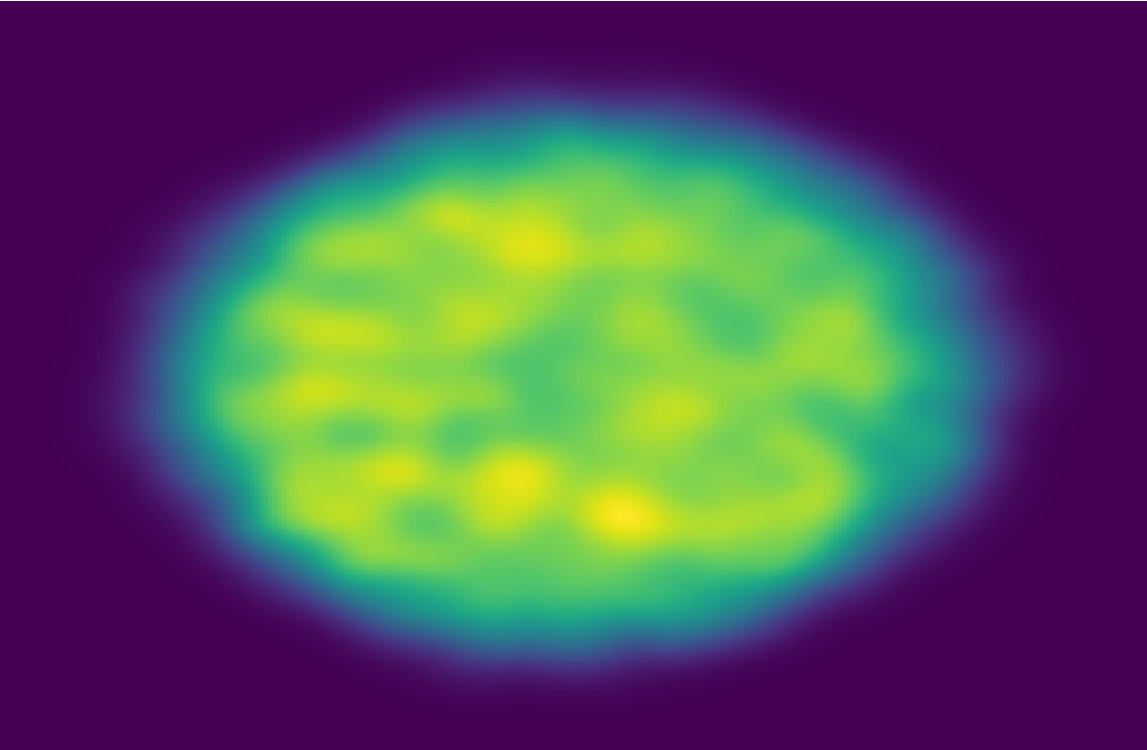}
      & \includegraphics{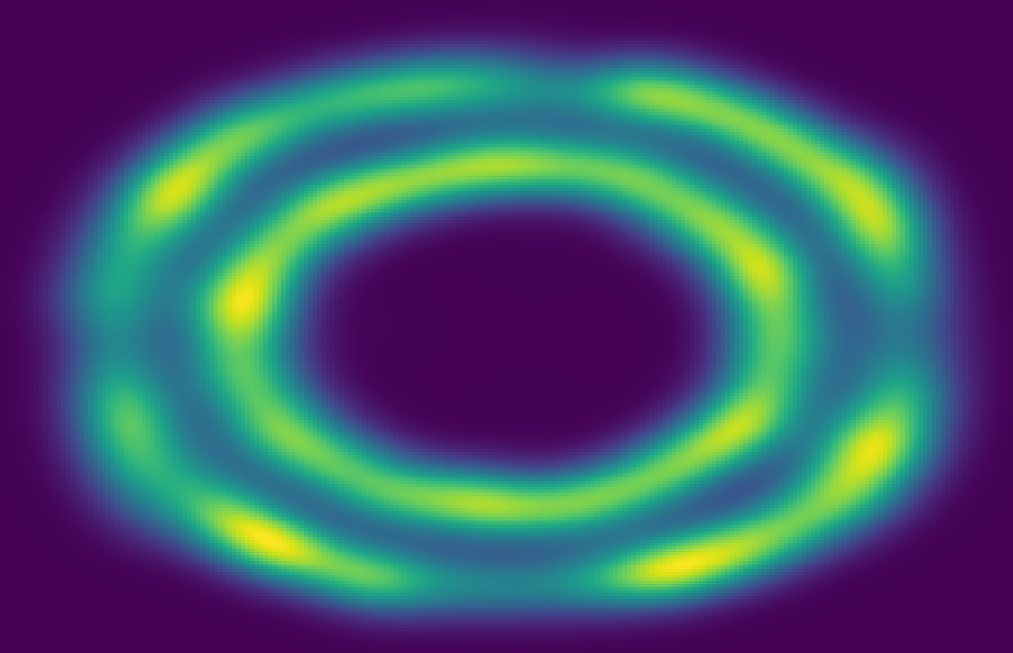}
    \\
    \bottomrule
  \end{tabular}
    \caption{\small Comparison of the performance of various generative methods—Real NVP, WGAN, GAN, ISL-slicing, and GAN pretrained with ISL to approximate three complex 2D distributions: the \texttt{Moon}, \texttt{Gauss}, and \texttt{Rings} datasets— referred to as Dual Moon, Circle Gaussian, and Two Rings, respectively. For a detailed description of the dataset, refer to \cite{Stimper2023}. The first column shows the target distributions, while the subsequent columns display the results generated by Real NVP, WGAN, GAN, ISL-slicing, and GAN pretrained with ISL, in that order.} \label{experiments in 2d distributions figures}
\end{figure}

\begin{table}[!h]
  \centering
  \setlength{\tabcolsep}{4pt}
  \rowcolors{2}{gray!15}{white}
  \sisetup{
    detect-weight = true,
    table-number-alignment = center,
    round-mode       = places,
    round-precision  = 2
  }
  \begin{tabular}{
    l
    S[table-format=1.2]
    S[table-format=1.2]
    S[table-format=1.2]
    S[table-format=1.2]
    S[table-format=1.2]
  }
    \toprule
    \textbf{Method / Dataset}
      & \textbf{Real NVP}
      & \textbf{GAN}
      & \textbf{WGAN}
      & \textbf{ISL}
      & \textbf{ISL+GAN} \\
    \midrule
    Dual moon
      & 1.77
      & 1.23
      & 1.02
      & 0.43
      & {\bfseries 0.30} \\
    Circle of Gaussians
      & 2.59
      & 2.24
      & 2.38
      & 1.61
      & {\bfseries 0.46} \\
    Two rings
      & 2.69
      & 1.46
      & 2.74
      & 0.56
      & {\bfseries 0.38} \\
    \bottomrule
  \end{tabular}
  \caption{\small Performance comparison of generative methods on 2D distributions. The table summarizes KL-divergence results for Real NVP, GAN, WGAN, and ISL on \texttt{Dual Moon}, \texttt{Circle of Gaussians}, and \texttt{Two Rings}. Lower values indicate better performance.}
\label{table: experiments in 2d distributions figures}
\end{table}

\subsection{ISL-pretrained GANs for robust mode coverage in 2D grids} \label{ISL-Pretrained GANs for Robust Mode Coverage in 2D Grids}

In this experiment, we explore how ISL-pretrained GANs improve mode coverage on the \texttt{2D-Ring} and \texttt{2D-Grid} datasets, benchmarks commonly used in generative model evaluation. The \texttt{2D-Ring} dataset consists of eight Gaussian distributions arranged in a circle, while the \texttt{2D-Grid} dataset has twenty-five Gaussians on a grid. We first pretrain with ISL-slicing to ensure comprehensive mode coverage, then fine-tune with a GAN. Performance was compared to other GANs using two metrics: the number of covered modes (\texttt{\#modes}) and the percentage of high-quality samples (\(\texttt{\%HQ}\)). A mode is covered if generated samples fall within three standard deviations of the Gaussian center, which also defines high-quality samples. The second metric represents their proportion among all generated samples.

We follow the experimental settings for GAN and WGAN as outlined by \cite{luo2024dyngan}. The generator is a 4-layer MLP with 128 units per layer and ReLU activation, while the discriminator is an MLP with 128 units, using ReLU except for a sigmoid in the final layer. We use a batch size of $128$, a 1:1 critic-to-generator update ratio, and learning rates of $10^{-4}$ for GAN and $10^{-5}$ for WGAN, both optimized using ADAM. ISL was run for $250$ epochs with \(K=10\), \(N=1000\), and $m=5$ random projections, capturing the full support of the distribution. The GAN is then trained for $30000$ epochs. The results presented in Table \ref{table:mode collapse exp} and Figure \ref{fig:combined_gan_comparison} compare our method with other techniques for addressing mode collapse in GANs. They demonstrate that our approach performs on par with more advanced methods, such as DynGAN (a semi-supervised technique, see \cite{luo2024dyngan}) and BourGAN (computationally intensive, see \cite{xiao2018bourgan}), while offering greater simplicity. 

Finally, we conduct a robustness experiment comparing the performance of a WGAN with suboptimal hyperparameters to the same WGAN after ISL pretraining. Both use the same generator and discriminator architectures described in Section \ref{Experiments on 2-D distributions}. The WGAN hyperparameters include a batch size of 1000, a critic-to-generator update ratio optimized over \(\{1\!:\!1,\,2\!:\!1,\,3\!:\!1,\,4\!:\!1,\,5\!:\!1\}\), and a learning rate of $10^{-5}$. For ISL, we use \(K=10\), \(N=1000\), $m=10$ random projections, and a learning rate of $10^{-3}$. Numerical results are displayed in Table \ref{table:comparison_wgan_isl_combined}. Additional experiments with batch sizes \(\{64, 128, 252, 512\}\) and optimized critic-to-generator ratios are shown in Table \ref{table:comparison_wgan_batchsize}. ISL-pretrained WGAN consistently outperforms standard WGAN across all ratios and batch sizes, detecting all 8 modes and producing high-quality samples. This shows that ISL pretraining significantly enhances stability and performance, leading to more robust GAN training. Once all modes are captured by ISL, a low learning rate allows the generator to improve high-quality metrics while maintaining full distribution coverage.

\begin{table}[h]
  \centering
  \rowcolors{2}{gray!15}{white}
  \sisetup{
    detect-weight = true,
    table-number-alignment = center,
    round-mode       = places,
    round-precision  = 1
  }
  \begin{tabular}{l *{4}{S}}
    \toprule
    \multirow{2}{*}{\textbf{Method}}
      & \multicolumn{2}{c}{\textbf{2D‐Ring}}
      & \multicolumn{2}{c}{\textbf{2D‐Grid}} \\
    \cmidrule(lr){2-3} \cmidrule(lr){4-5}
      & {\#modes} & {\%HQ}
      & {\#modes} & {\%HQ} \\
    \midrule
    GAN \citep{goodfellow2014generative}
      &  6.3 &  98.2
      & 17.1 &  92.5 \\
    VEEGAN \citep{srivastava2017veegan}
      &  {\bfseries 8.0} &  86.8
      & 24.4 &  74.2 \\
    Pointwise \citep{zhong2019rethinking}
      &  {\bfseries 8.0} &  87.5
      & {\bfseries 25.0} &  76.7 \\
    BourGAN \citep{xiao2018bourgan}
      &  {\bfseries 8.0} & {\bfseries 99.9}
      & {\bfseries 25.0} &  94.9 \\
    WGAN \citep{arjovsky2017wasserstein}
      &  7.7 &  86.4
      & 24.8 &  83.7 \\
    DynGAN \citep{luo2024dyngan}
      &  {\bfseries 8.0} &  99.5
      & {\bfseries 25.0} &  96.0 \\
    ISL + GAN
      &  {\bfseries 8.0} &  97.9
      & 24.6 & {\bfseries 96.8} \\
    ISL + WGAN
      &  {\bfseries 8.0} &  97.4
      & {\bfseries 25.0} &  96.0 \\
    \bottomrule
  \end{tabular}
    \caption{\small Quantitative results on synthetic 2D‐mode benchmarks. \#modes is the number of modes covered; \%HQ is the percentage of high‐quality samples (higher is better).}\label{table:mode collapse exp}
\end{table}

\begin{figure}[h]
    \centering
    \begin{subfigure}[t]{0.32\textwidth}
        \centering
        \includegraphics[width=\textwidth]{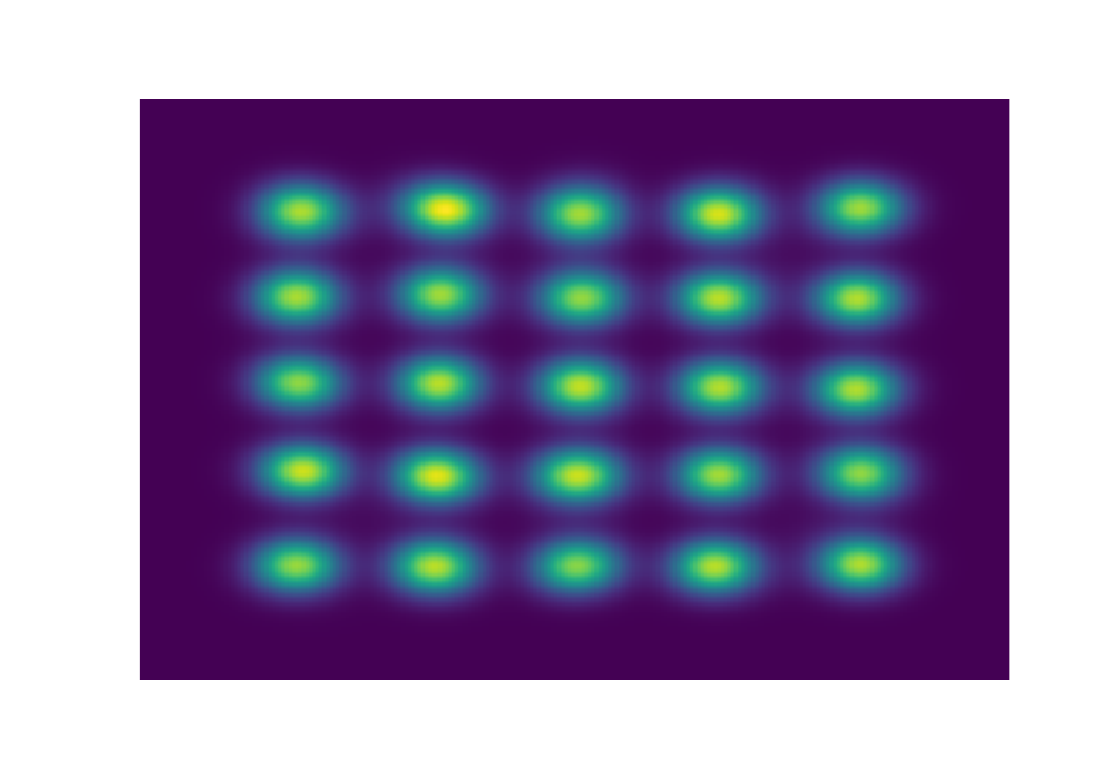}
        \caption{\small ISL + GAN, 25 modes dataset}
        \label{fig:covered_modes}
    \end{subfigure}
    \hfill
    \begin{subfigure}[t]{0.32\textwidth}
        \centering
        \includegraphics[width=\textwidth]{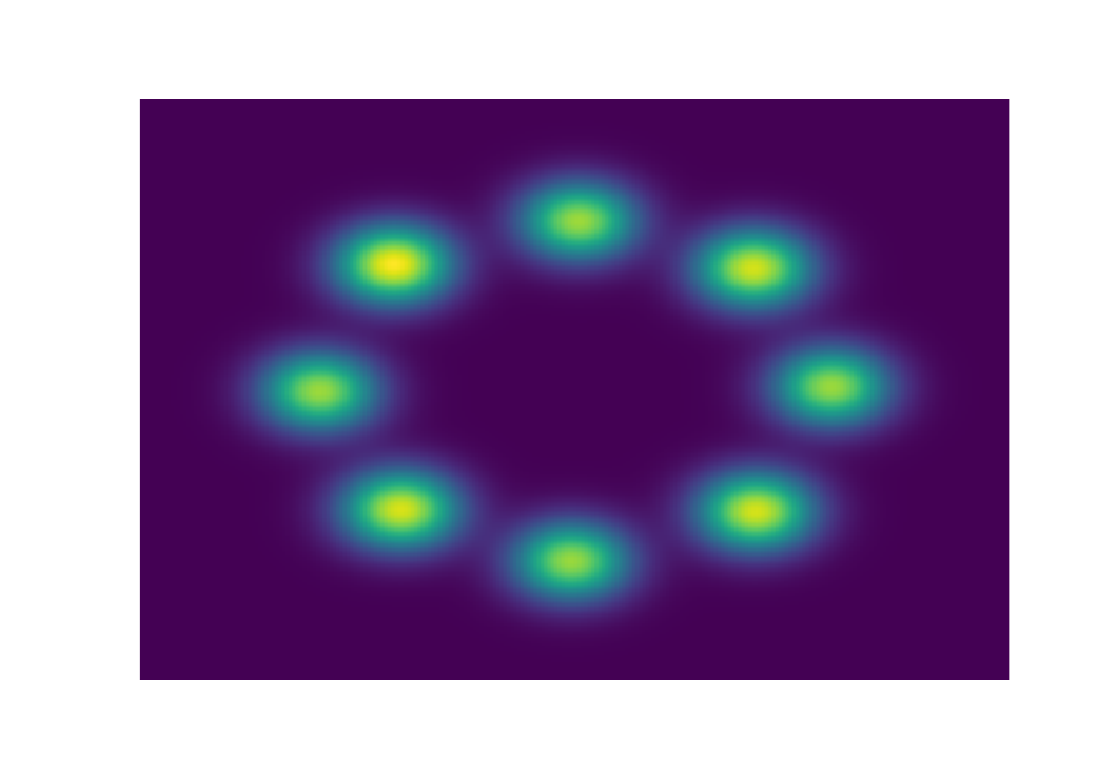}
        \caption{\small ISL + GAN, 8 modes dataset}
        \label{fig:high_quality_samples}
    \end{subfigure}
    \hfill
    \begin{subfigure}[t]{0.32\textwidth}
        \centering
        \includegraphics[width=\textwidth]{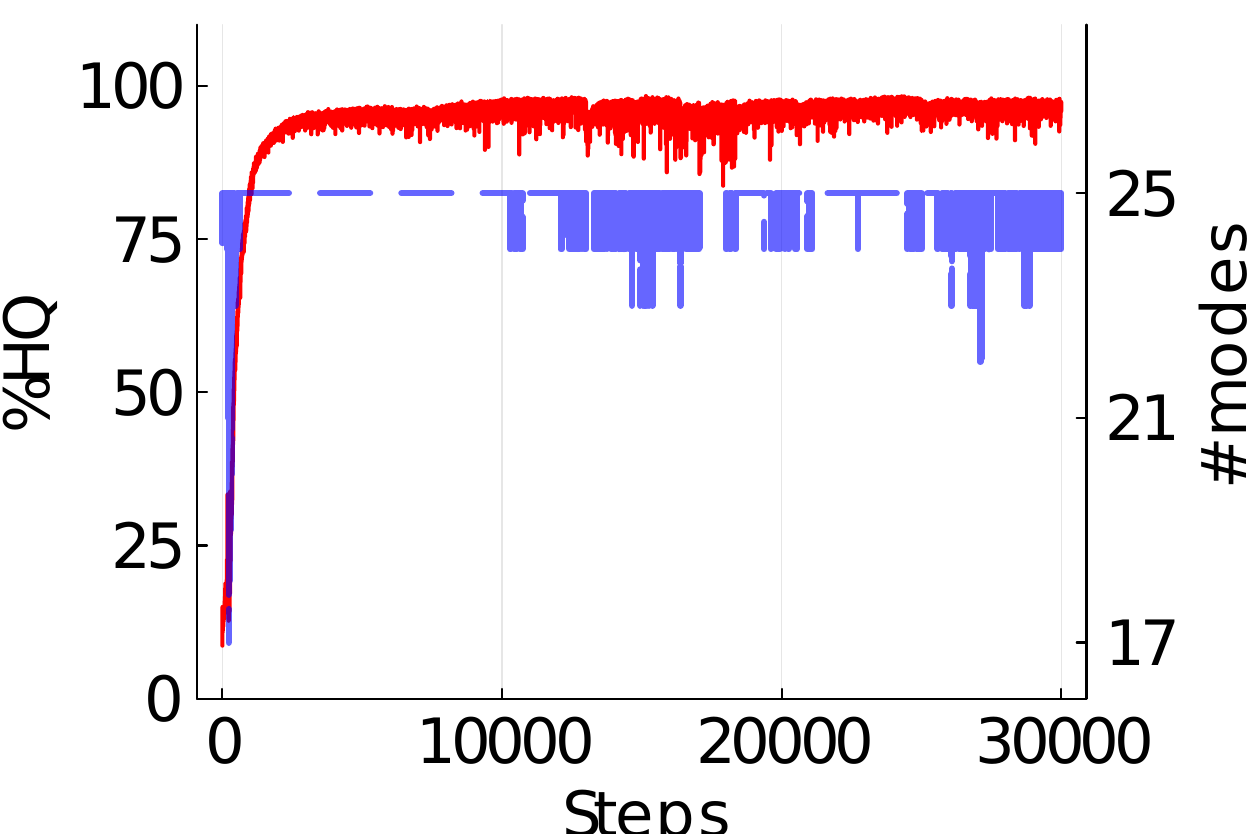}
        \caption{\small $25$ modes dataset, ISL + GAN}
        \label{fig:covered_modes_convergence_25}
    \end{subfigure}
    
    \caption{\small ISL + GAN on $25$ and $8$ modes datasets (plots (a) and (b), respectively) with \(K=10\), \(N=1000\), ISL learning rate of $10^{-3}$ for $250$ epochs, and GAN with a batch size of 128, 1:1 critic-generator ratio, and learning rate of $10^{-4}$ for $30000$ epochs. Plot (c) shows the coverage of the \texttt{2D-Grid} dataset, highlighting the number of modes covered (\texttt{\#modes}, blue) and high-quality samples (\texttt{\%HQ}, red) during GAN training after ISL pretraining.}
    \label{fig:combined_gan_comparison}
\end{figure}

\begin{table}[h]
  \centering
  \footnotesize
  \setlength{\tabcolsep}{1.4pt}
  \rowcolors{2}{gray!15}{white}
  \sisetup{
    detect-weight          = true,
    table-number-alignment = center,
    round-mode             = places,
    round-precision        = 2
  }

  \begin{adjustbox}{width=0.88\textwidth,center}
    \begin{minipage}[t]{0.48\textwidth}
      \centering
      \renewcommand{\arraystretch}{1.0}
      \begin{tabular}{l *{3}{S[table-format=1.1]} *{3}{S[table-format=2.2]}}
      \rowcolor{gray!15}
        \toprule
        \multirow{2}{*}{\textbf{Critics}}
          & \multicolumn{3}{c}{\textbf{WGAN}}
          & \multicolumn{3}{c}{\textbf{ISL+WGAN}} \\
        \cmidrule(lr){2-4}\cmidrule(lr){5-7}
          & {\#modes} & {\%HQ} & {KL\,div}
          & {\#modes} & {\%HQ} & {KL\,div} \\
        \midrule
        1:1 & 3.7 & 40.99 & 33.99 & {\bfseries 8.0} & {\bfseries 97.74} & {\bfseries 0.51} \\
        1:2 & 7.0 &  0.07 & 59.58 & {\bfseries 8.0} & {\bfseries 98.43} & {\bfseries 0.47} \\
        1:3 & 8.0 &  0.63 & 57.05 & {\bfseries 8.0} & {\bfseries 97.39} & {\bfseries 0.55} \\
        1:4 & 8.0 &  0.44 & 59.18 & {\bfseries 8.0} & {\bfseries 98.07} & {\bfseries 0.78} \\
        1:5 & 8.0 &  0.19 & 56.83 & {\bfseries 8.0} & {\bfseries 98.24} & {\bfseries 0.72} \\
        \bottomrule
      \end{tabular}
      \caption{\small Comparison of WGAN and ISL + WGAN on the \texttt{2D-Ring} dataset.}\label{table:comparison_wgan_isl_combined}
    \end{minipage}%
    \hspace{1cm}
    \begin{minipage}[t]{0.48\textwidth}
      \centering
      \renewcommand{\arraystretch}{1.2}
      \begin{tabular}{l *{3}{S[table-format=1.1]} *{3}{S[table-format=2.2]}}
        \toprule
        \multirow{2}{*}{\textbf{Batch Size}}
          & \multicolumn{3}{c}{\textbf{WGAN}}
          & \multicolumn{3}{c}{\textbf{ISL+WGAN}} \\
        \cmidrule(lr){2-4}\cmidrule(lr){5-7}
          & {\#modes} & {\%HQ} & {KL\,div}
          & {\#modes} & {\%HQ} & {KL\,div} \\
        \midrule
        64  & 6.1 & 30.01 & 40.52 & {\bfseries 8.0} & {\bfseries 99.31} & {\bfseries 0.82} \\
        128 & 4.5 & 19.94 & 47.29 & {\bfseries 8.0} & {\bfseries 98.22} & {\bfseries 1.20} \\
        252 & 3.1 & 19.98 & 46.57 & {\bfseries 8.0} & {\bfseries 89.70} & {\bfseries 2.24} \\
        512 & 3.5 & 29.97 & 41.19 & {\bfseries 8.0} & {\bfseries 95.18} & {\bfseries 1.71} \\
        \bottomrule
      \end{tabular}
      \caption{\small Comparison of WGAN and ISL + WGAN with different batch sizes.}\label{table:comparison_wgan_batchsize}
    \end{minipage}
  \end{adjustbox}
\end{table}

\subsection{Addressing mode collapse on MNIST and FMNIST}

In this section, we assess the effectiveness of the ISL-slicing for generating high-dimensional images. We integrate the ISL-slicing loss into the deep convolutional GAN (DCGAN) generator’s training \citep{radford2015unsupervised} and evaluate its performance on the MNIST and Fashion-MNIST benchmarks. Following \cite{sajjadi2018assessing}, we characterize the quality and diversity of generated images by reporting precision (as a proxy for fidelity) and recall (as a proxy for diversity). We train each model for 40 epochs with a batch size of 128. For the pretrained variants, we first ran 20 epochs under the sliced dual-ISL objective (using 20 random projections) and then continued with 40 epochs of standard DCGAN training.

In Table \ref{table:real word data 1}, we compare our results to those of other implicit generative models. On MNIST, our straightforward ISL-based model matches the recall of five-discriminator GANs \citep{durugkar2016generative, choi2022mcl} while using only one-third the number of parameters. Moreover, pretraining the generator with sliced-ISL and then fine-tuning under the standard adversarial loss delivers state-of-the-art precision and recall.

On Fashion-MNIST, our model almost ties the recall of MCL-GAN’s and, although the recall of GMAN is higher, we outperform it in precision. This shows that despite relying on a simpler architecture, ISL can achieve competitive recall and precision across diverse image-generation benchmarks.

\begin{table}[ht]
  \centering
  \small
  \label{tab:quant-results}
  \rowcolors{2}{gray!15}{white}
  \setlength{\tabcolsep}{5pt}
  \begin{tabular}{
      l 
      l 
      S[table-format=2.2(2)] 
      S[table-format=2.2(2)] 
      S[table-format=2.2(2)] 
      S[table-format=2.2(2)] 
    }
    \toprule
    \multirow{2}{*}{\textbf{Dataset}}
      & \multirow{2}{*}{\textbf{Method}}
      & \multicolumn{2}{c}{\textbf{F-score}}
      & \multicolumn{2}{c}{\textbf{P\&R}} \\
    \cmidrule(lr){3-4} \cmidrule(lr){5-6}
      & 
      & {\(F_{1/8}\!\uparrow\)} 
      & {\(F_{8}\!\uparrow\)} 
      & {Precision\(\uparrow\)} 
      & {Recall\(\uparrow\)} \\
    \midrule
    \textbf{MNIST}
      & ISL (m=20)
      & ${85.00 \pm 0.32}$   & ${95.17 \pm 1.76}$
      & ${84.85 \pm 1.20}$   & ${95.35 \pm 1.39}$ \\
      & ISL (m=50)
      & ${85.69 \pm 0.29}$   & ${95.81 \pm 1.24}$
      & ${85.55 \pm 1.11}$   & ${96.23 \pm 1.98}$ \\
      & DCGAN
      & ${93.58 \pm 0.64}$   & ${75.66 \pm 1.46}$
      & ${93.85 \pm 1.45}$   & ${75.43 \pm 2.56}$ \\
      & ISL + DCGAN
      & ${93.58 \pm 0.84}$   & ${95.82 \pm 1.61}$
      & ${94.03 \pm 1.82}$   & ${96.68 \pm 2.42}$ \\
      & GMAN
      & ${97.60 \pm 0.70}$   & ${96.81 \pm 1.71}$
      & ${97.60 \pm 1.82}$   & ${96.80 \pm 2.42}$ \\
      & MCL-GAN
      & $\mathbf{97.71 \pm 0.19}$ & $\mathbf{98.49 \pm 1.57}$
      & $\mathbf{97.70 \pm 1.33}$ & $\mathbf{98.50 \pm 2.15}$ \\
    \addlinespace
    \addlinespace
    \textbf{FMNIST}
      & ISL (m=20)
      & ${81.84 \pm 0.11}$   & ${91.08 \pm 1.83}$
      & ${81.48 \pm 1.43}$   & ${91.49 \pm 2.15}$ \\
      & ISL (m=50)
      & ${83.90 \pm 0.09}$   & ${91.18 \pm 1.57}$
      & ${84.08 \pm 1.31}$   & ${92.92 \pm 1.23}$ \\
      & DCGAN
      & ${86.14 \pm 0.11}$   & ${88.92 \pm 1.51}$
      & ${86.60 \pm 1.58}$   & ${88.97 \pm 1.33}$ \\
      & ISL + DCGAN
      & ${91.43 \pm 0.19}$   & ${91.87 \pm 1.57}$
      & ${91.88 \pm 1.35}$   & ${92.42 \pm 1.47}$ \\
      & GMAN
      & ${90.97 \pm 0.09}$   & $\mathbf{95.43 \pm 1.12}$
      & ${90.90 \pm 1.33}$   & $\mathbf{95.50 \pm 2.25}$ \\
      & MCL-GAN
      & $\mathbf{97.62 \pm 0.09}$ & ${92.97 \pm 1.28}$
      & $\mathbf{97.70 \pm 1.33}$ & ${92.90 \pm 2.31}$ \\
    \bottomrule
  \end{tabular}\vspace{0.2cm}
    \caption{\small Quantitative results on MNIST and Fashion-MNIST (28×28), reporting $F_{1/8}$ and $F_{8}$ ($\beta$-weighted harmonic means of precision and recall), precision, and recall (mean±std, \%). We compare ISL (m=20, 50), standard DCGAN (with and without ISL pretraining), GMAN, and MCL-GAN. Bold entries mark the best score per column; higher is better.} \label{table:real word data 1}
\end{table}

Figure \ref{fig:digit_distribution} compares the class-frequency distributions produced by the sliced ISL model (trained for 40 epochs with $m=50$ random projections) against those generated by a conventional DCGAN. Our model generates all ten digit classes with nearly uniform frequency whereas the DCGAN displays marked class imbalances. To estimate these frequencies, we sampled $10,000$ images from each model and labeled them using a pretrained digit classifier. Under a one‐sample Kolmogorov–Smirnov test for uniformity on $10,000$ samples, the sliced-ISL model yielded a p-value $p = 0.062$, whereas the DCGAN produced $p = 0.742$—reflecting a much closer match to the ideal uniform distribution.

\begin{figure}[ht]
  \centering
  \resizebox{0.7\textwidth}{!}{%
    \begin{tikzpicture}
      \begin{axis}[
        ybar,
        bar width=8pt,
        xmin=-0.5,            
        xmax=9.5,           
        ymin=0, ymax=45,
        grid=major,
        grid style={dashed,gray!30},
        xtick={0,1,2,3,4,5,6,7,8,9},
        xlabel={Digit},
        ylabel={Percentage (\%)},
        tick align=inside,
        nodes near coords,
        every node near coord/.append style={font=\scriptsize, yshift=1pt},
        cycle list={{blue!60!black, fill=blue!30!white},
                    {orange!80!black, fill=orange!40!white}},
        width=0.8\textwidth,
        height=0.5\textwidth,
        legend style={
          at={(0.95,0.95)},
          anchor=north east,
          draw=none,
          font=\small,
          /tikz/every odd column/.append style={column sep=5pt}
        },
        legend cell align=left
      ]
        \addplot+[bar shift=-3pt] coordinates {
          (0,6.7)  (1,40.0) (2,2.2)  (3,3.7)  (4,7.7)
          (5,2.1)  (6,3.3)  (7,14.7) (8,2.7)  (9,16.9)
        };
        \addplot+[bar shift=3pt] coordinates {
          (0,10.5) (1,12.5) (2,7.3)  (3,12.4) (4,10.3)
          (5,8.0)  (6,7.6)  (7,13.2) (8,7.4)  (9,10.9)
        };
        \draw[dashed, opacity=0.4] (rel axis cs:0,0.222) -- (rel axis cs:1,0.222);
        \legend{DCGAN, ISL+DCGAN}
      \end{axis}
    \end{tikzpicture}%
  }
  \caption{\small Digit-frequency comparison for DCGAN vs.\ ISL+DCGAN. The dashed line marks the ideal 10\% for each class.}
  \label{fig:digit_distribution}
\end{figure}

Figure \ref{fig:ablation studies on p and r metrics vs m} shows the precision and recall of the ISL model on MNIST in an ablation study with $m=20$ and $m=100$ random projections. With just $m=20$ projections, the model delivers robust performance; increasing to $m=100$ yields only marginal improvements in these metrics. Monte Carlo sampling error for the sliced-ISL estimator scales as $1/\sqrt{m}$, so initial increases in $m$ deliver significant gains but further increases produce only marginal improvements, explaining the observed performance plateau.

\begin{figure}[htb]
  \centering
  \begin{subfigure}[t]{0.48\textwidth}
    \centering
    \includegraphics[width=\linewidth, height=5cm]{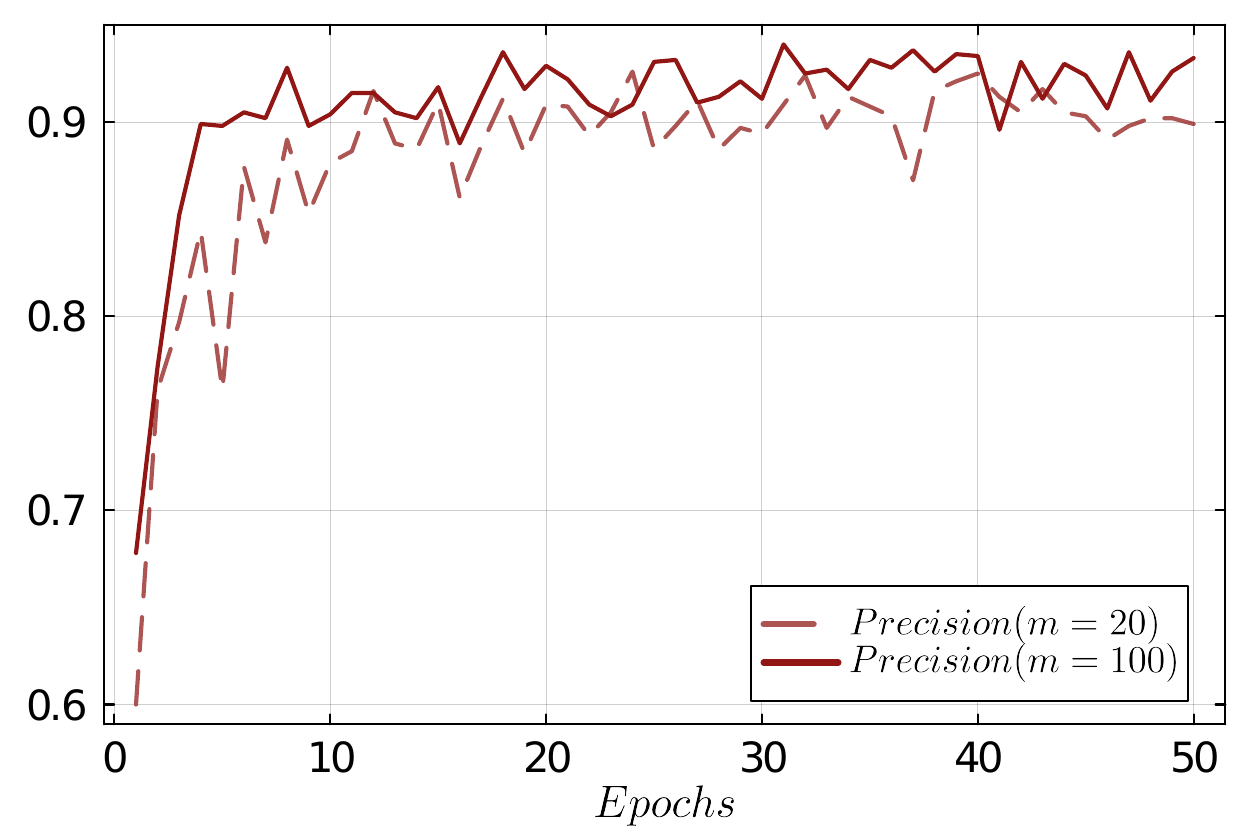}
    \label{fig:precision_m20_m100_MNIST}
  \end{subfigure}
  \hfill
  \begin{subfigure}[t]{0.48\textwidth}
    \centering
    \includegraphics[width=\linewidth, height=5cm]{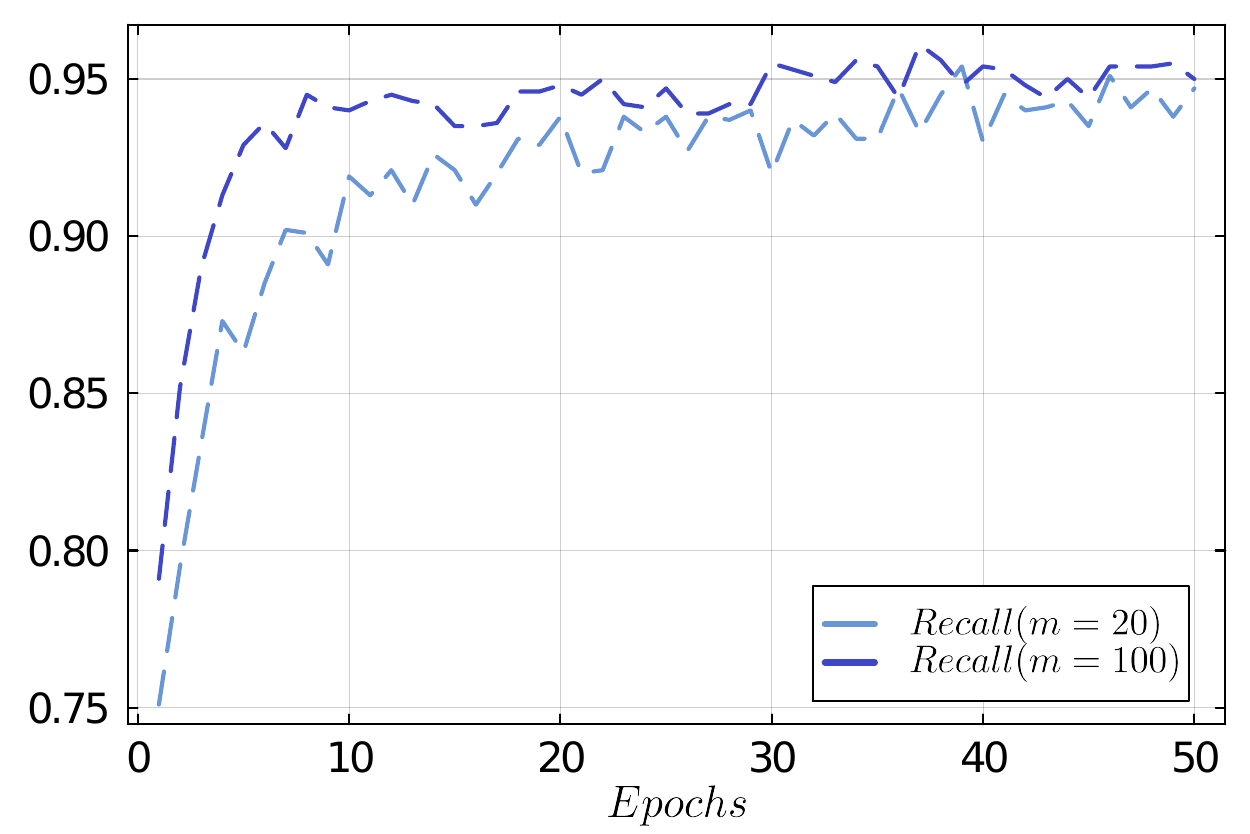}
    \label{fig:recall_m20_m100_MNIST}
  \end{subfigure}

  \vspace{0.5ex}
  \caption{\small Precision (left) and recall (right) on the MNIST test set after 50 training epochs, using $m=20$ and $m=100$ random projections.}
  \label{fig:ablation studies on p and r metrics vs m}
\end{figure}

Finally, Figure \ref{fig:both_figures} compares the sample diversity of a standard DCGAN Figure \ref{fig:figure1} against a DCGAN pretrained with our ISL method Figure \ref{fig:figure2}. In Figure \ref{fig:figure1}, the red boxes highlight multiple occurrences of the digit “1,” a clear sign of mode collapse where the generator repeatedly produces the same class. After pretraining with ISL (Figure \ref{fig:figure2}), the generator exhibits far fewer repeated “1”s, instead producing more diverse handwritten digits. This visual evidence demonstrates that ISL pretraining effectively mitigates mode collapse and encourages the generator to cover a broader range of the MNIST data distribution.  

\begin{figure}[ht]
  \centering
  \resizebox{1.0\textwidth}{!}{%
  \begin{subfigure}[b]{0.40\textwidth}
    \begin{overpic}[width=\textwidth]{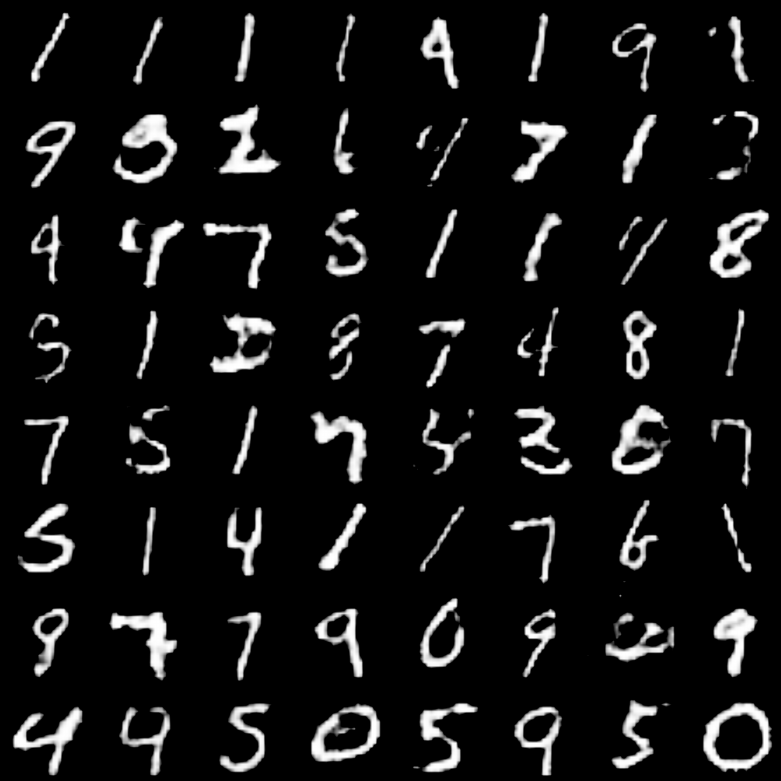}
    \linethickness{2pt}
      \put(0,88){\color{red}\framebox(50,12){}}
      \put(63,88){\color{red}\framebox(12,12){}}
      \put(76,75){\color{red}\framebox(12,12){}}
      \put(50,62){\color{red}\framebox(25,12){}}
      \put(11,51){\color{red}\framebox(12,12){}}
      \put(88,50){\color{red}\framebox(12,12){}}
      \put(24,38){\color{red}\framebox(12,12){}}
      \put(11,25){\color{red}\framebox(12,12){}}
      \put(37,25){\color{red}\framebox(24,12){}}
      \put(88,26){\color{red}\framebox(12,12){}}
    \end{overpic}
    \caption{DCGAN}
    \label{fig:figure1}
  \end{subfigure}
  \hfill
  \begin{subfigure}[b]{0.40\textwidth}
    \begin{overpic}[width=\textwidth]{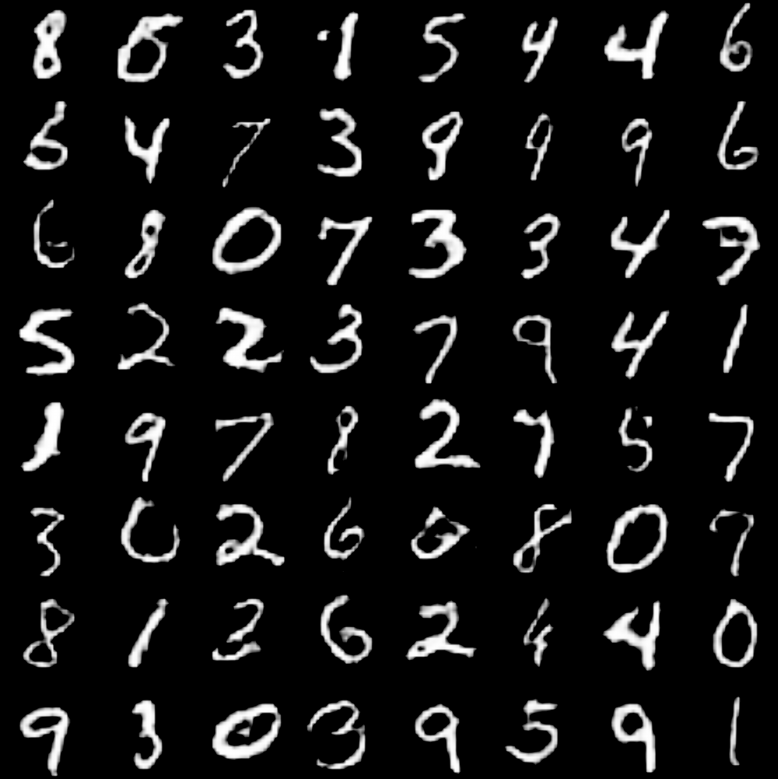}
     \linethickness{2pt}
      \put(88,49){\color{red}\framebox(12,12){}}
      \put(0,37){\color{red}\framebox(12,12){}}
      \put(12,13){\color{red}\framebox(12,12){}}
      \put(88,1){\color{red}\framebox(12,12){}}
    \end{overpic}
    \caption{DCGAN pretrained with ISL}
    \label{fig:figure2}
  \end{subfigure}
    }
  \caption{\small Generated samples from the MNIST dataset: plot (a) shows numbers generated by a DCGAN, while the plot (b) shows numbers generated by a DCGAN pretrained with ISL. Red squares around the repeated “1”’s.}
  \label{fig:both_figures}
\end{figure}

\subsection{Improving diversity on CelebA with ISL pretraining}

In this section, we evaluate ISL pretraining on the CelebA dataset under a unified 50-epoch regime. We first pretrain a DCGAN with ISL for 10 epochs (learning rate $10^{2}$, $K = 40$), then fine-tune it adversarially for 40 epochs (batch size 128, learning rate $2\times10^{-4}$). For a fair comparison, we trained every baseline using 50 epochs, batch size 128, and a grid search over learning rates $\{1\times10^{-2},\,1\times10^{-3},\,2\times10^{-4},\,1\times10^{-4}\}$. The baselines we evaluated were
\begin{itemize}
    \item Generative multi-adversarial networks (GMAN) \citep{durugkar2016generative}: uses multiple discriminators to improve mode coverage.
    \item Least-aquares DCGAN (LS-DCGAN) \citep{mao2017least}: swaps the usual cross-entropy loss for a least-squares objective.
    \item Wasserstein GAN with gradient penalty (W-DCGAN-GP) \citep{gulrajani2017improved}: enforces the 1-Lipschitz constraint via a gradient penalty.
    \item Spectral normalization DCGAN (SN-DCGAN) \citep{miyato2018spectral}: stabilizes training by normalizing discriminator weights.
    \item Dynamic GAN (DynGAN) \citep{luo2024dyngan}: detects collapsed modes during training and uses conditional generators to recover them.
\end{itemize}
From each model’s single run, we report the best recall, Fréchet Inception Distance (FID), and precision scores. Full details are provided in Table \ref{tab:celeba-results}.


Remarkably, our ISL-pretrained DCGAN achieves the highest recall of all models, despite its relative simplicity. It even outperforms multi-discriminator frameworks such as GMAN—while using three times fewer parameters—and specialized techniques like DynGAN, which introduce semi-supervised overhead to recover collapsed modes. Moreover, in terms of both FID and precision, the ISL-pretrained network remains highly competitive, despite fitted only on 40 adversarial training epochs (10 fewer than every baseline). These results underscore the power of ISL pretraining to capture diverse image modes without extra architectural complexity.

\begin{table}[ht]
  \centering
  \small
  \rowcolors{2}{gray!15}{white}
  \setlength{\tabcolsep}{6pt}      
  \renewcommand{\arraystretch}{1.2}
  \begin{tabular}{
      >{\raggedright\arraybackslash}p{1.5cm}  
      >{\raggedright\arraybackslash}p{6.0cm}  
      S[table-format=2.2(2)]                  
      S[table-format=2.2(2)]                  
      S[table-format=2.2(2)]                  
    }
    \toprule
    \textbf{Dataset}
      & \textbf{Method}
      & {\(FID\downarrow\)}
      & {\(Precision\uparrow\)}
      & {\(Recall\uparrow\)} \\
    \midrule
    \textbf{CelebA}
      & ISL + DCGAN
      & {$31.64$}
      & {$0.887$}
      & {$\mathbf{0.954}$} \\
      & DCGAN \citep{radford2015unsupervised}
      & {$30.93$}
      & {$0.839$}
      & {$0.834$} \\
      & LS-DCGAN \citep{mao2017least}
      & {$\mathbf{22.99}$}
      & {$\mathbf{0.997}$}
      & {$0.324$} \\
      & W-DCGAN-GP \citep{gulrajani2017improved}
      & {$80.30$}
      & {$0.982$}
      & {$0.374$} \\
      & SN-DCGAN \citep{miyato2018spectral}
      & {$32.94$}
      & {$0.974$}
      & {$0.887$} \\
      & DynGAN \citep{luo2024dyngan}
      & {$48.06$}
      & {$0.955$}
      & {$0.718$} \\
      & GMAN \citep{durugkar2016generative}
      & {$31.66$}
      & {$0.873$}
      & {$0.888$} \\
    \bottomrule
  \end{tabular}
  \caption{\small Performance comparison on CelebA. We evaluate ISL‐pretrained DCGAN (ISL+DCGAN), standard DCGAN, Least Squares DCGAN (LS-DCGAN), Spectral Normalization DCGAN (SN-DCGAN), DynGAN, and Generative Multi-Adversarial Networks (GMAN) using Fréchet Inception Distance (FID↓), precision (↑), and recall (↑). Lower FID and higher precision/recall indicate better sample quality and diversity.} \label{tab:celeba-results}
\end{table}

Our two‐stage training pipeline (first pretraining the generator by minimizing ISL, then fine‐tuning it with standard adversarial updates) parallels the strategy of the sliced‐Wasserstein generator (SWG) from \citet{nadjahi2021fast}. In SWG, the authors begin by minimizing a Monte Carlo estimate of the sliced‐Wasserstein distance in feature space before introducing a discriminator (see \citet[Sec. 3.2]{nadjahi2021fast}). They motivate this design by noting that naive slicing in very high dimensions demands an impractically large number of random projections; adding a learned discriminator effectively reduces the problem’s dimensionality to the most discriminative subspace, thereby cutting the projection requirement and restoring strong FID performance. We build on exactly this insight: by replacing the Monte Carlo sliced‐Wasserstein objective with ISL during the initial phase, we still “warm up” the generator into a space close to the real‐data manifold, but with a discrepancy measure that more directly reduces mode collapse. Under identical hyperparameter settings and evaluation protocols, ISL pretraining achieves substantially higher recall on every benchmark, while matching the precision and FID of SWG. These results confirm that ISL is a more effective statistical divergence than sliced‐Wasserstein for preserving diversity without sacrificing sample fidelity. Table \ref{tab:celeba-results_sw} provides the full comparison, highlighting the superior recall of ISL alongside equivalent precision and FID.

\begin{table}[ht]
  \centering
  \scriptsize
  \rowcolors{2}{gray!15}{white}
  \setlength{\tabcolsep}{5.0pt}
  \renewcommand{\arraystretch}{1.4}
  \begin{tabular}{
      c
      S[table-format=2.2,table-number-alignment=center]
      S[table-format=1.2]
      S[table-format=1.2]
      S[table-format=2.2,table-number-alignment=center]
      S[table-format=1.2]
      S[table-format=1.2]
      S[table-format=2.2,table-number-alignment=center]
      S[table-format=1.2]
      S[table-format=1.2]
    }
    \toprule
    \multirow{2}{*}{\(\mathbf{m}\)}
      & \multicolumn{3}{c}{\textbf{ISL + DCGAN}}
      & \multicolumn{3}{c}{\textbf{SWG}}
      & \multicolumn{3}{c}{\textbf{SWG-2}} \\
    \cmidrule(lr){2-4} \cmidrule(lr){5-7} \cmidrule(lr){8-10}
      & {\(FID\downarrow\)}
      & {\(Precision\uparrow\)}
      & {\(Recall\uparrow\)}
      & {\(FID\downarrow\)}
      & {\(Precision\uparrow\)}
      & {\(Recall\uparrow\)}
      & {\(FID\downarrow\)}
      & {\(Precision\uparrow\)}
      & {\(Recall\uparrow\)} \\
    \midrule
    100      & $\mathbf{31.64}$ & $\mathbf{0.89}$ & $\mathbf{0.95}$ & 32.82 & 1.00 & 0.16 & 34.57 & 1.00 & 0.14 \\
    1\,000   & $\mathbf{30.41}$ & $\mathbf{0.91}$ & $\mathbf{0.94}$ & 37.27 & 0.98 & 0.68 & 35.83 & 1.00 & 0.56 \\
    10\,000  & $\mathbf{30.33}$ & $\mathbf{0.92}$ & $\mathbf{0.93}$ & 37.29 & 0.98 & 0.63 & 37.20 & 0.99 & 0.74 \\
    \bottomrule
  \end{tabular}
  \caption{\small Comparison for CelebA between ISL + DCGAN, the Sliced-Wasserstein Generator (SWG), and its Sliced-Wasserstein-2 variant (SWG-2) as a function of the number of random projections \(m\). Lower FID and higher precision/recall indicate better sample quality and diversity.}
  \label{tab:celeba-results_sw}
\end{table}

\section{Time series prediction}  \label{section Invariant Statistical Loss for Time Series Prediction}

Our approach applies to both univariate and multivariate time series. 
Hereafter we briefly describe the methodology for both types for series and conclude the section with experiments comparing these techniques on various datasets.



\subsection{Univariate time series prediction}

Let \( y[0], y[1], \ldots, y[T] \) represent a realization of a discrete-time random process. We assume that the process begins at \( t=0 \), \( Y[t] \) denotes the r.v. at time \( t \), and \( p_t = p(Y[t] | Y[0], \ldots, Y[t-1]) \) is the unknown conditional distribution. Given the sequence \( y[0], y[1], \dots, y[t-1] \), we aim to train an autoregressive conditional implicit generator network, \( g_\theta(z_t, \mathbf{h}[t]) \), to approximate \( p_t \). Here, \( z_t \) is a sequence of i.i.d. Gaussian r.v.s, and \( \mathbf{h}[t] \) is an embedding of the sequence \( y[0], \ldots, y[t-1] \) via a NN, such as a simple RNN connected to the generator (see Figure \ref{fig:my_tikz}). At time \( t \), the observation \( y[t] \) is fed into the RNN, compressing the history into a hidden state \( \mathbf{h}[t] \). The generator \( g_\theta(z, \mathbf{h}[t]) \) uses this hidden state and noise \( z \) to predict \( \tilde{y}[t+1] \). During testing, \( \tilde{y}[t+1] \) is fed back into the RNN for forecasting.

\tikzstyle{ts neural network style}=[rounded corners=4]
\tikzset{
    rnn cell/.style={rectangle, rounded corners=5, fill=blue!20, minimum width=1.5cm, minimum height=1.2cm}
}
\tikzset{
    gen cell/.style={rectangle, rounded corners=5, fill=red!20, minimum width=1.5cm, minimum height=1.2cm}
}
\begin{figure}[h] \label{architecture ts}
    \centering
        \begin{tikzpicture}[ts neural network style]
            \node[rnn cell] at (1.4,3.5) {RNN};
            \node[gen cell] at (1.4,1.5) {$g_\theta$};
        
            \node[rnn cell] at (3.9,3.5) {RNN};
            \node[gen cell] at (3.9,1.5) {$g_\theta$};
        
            \node[rnn cell] at (6.5,3.5) {RNN};
            \node[gen cell] at (6.5,1.5) {$g_\theta$};
            
            \draw[->] (1.4,4.7) -- (1.4,4.2) node[midway, right] {$y[t]$};
            \draw[->] (1.4,2.8) -- (1.4,2.2) node[midway, right] {\footnotesize $\mathbf{h}[t]$};
            \draw[->] (1.4,0.8) -- (1.4,0.4) node[midway, right] {$\tilde{y}[t+1]$};
        
            \draw[->] (4.0,4.7) -- (4.0,4.2) node[midway, right] {$y[t+1]$};
        
            \draw[->] (6.5,4.7) -- (6.5,4.2) node[midway, right] {$y[t+2]$};
            \draw[->] (3.9,2.8) -- (3.9,2.2) node[midway, right] {\footnotesize $\mathbf{h}[t+1]$};
            \draw[->] (3.9,0.8) -- (3.9,0.4) node[midway, right] {$\tilde{y}[t+2]$};
        
            \draw[->] (6.5,2.8) -- (6.5,2.2) node[midway, right] {\footnotesize $\mathbf{h}[t+2]$};
            \draw[->] (6.5,0.8) -- (6.5,0.4) node[midway, right] {$\tilde{y}[t+3]$};

            \draw[->] (0.0,3.5) -- (0.6,3.5) node[midway, above] {\hspace{-4mm}\footnotesize $\mathbf{h}[t-1]$};
            \draw[->] (2.3,3.5) -- (3.0,3.5) node[midway, above] {\footnotesize $\mathbf{h}[t]$};
            \draw[->] (0.2,1.2) -- (0.6,1.2) node[midway, above] {\hspace{-1mm}\footnotesize $z_{t}$};
        
            \draw[->] (4.8,3.5) -- (5.6,3.5) node[midway, above] {\footnotesize $\mathbf{h}[t+1]$};
            \draw[->] (2.7,1.2) -- (3.1,1.2) node[midway, above] {\footnotesize $z_{t+1}$};
        
            \draw[->] (5.3,1.2) -- (5.7,1.2) node[midway, above] {\hspace{-1mm}\footnotesize $z_{t+2}$};
        
            \draw[->] (7.3,3.5) -- (8.0,3.5) node[midway, above] {\hspace{3mm}\footnotesize $\mathbf{h}[t+2]$};
        \end{tikzpicture}
\caption{\small Conditional implicit generative model for time-series prediction.}\label{architecture ts}
\label{fig:my_tikz}
\end{figure}

As noted in \cite{de2024training}, all results from previous sections remain valid in this temporal setup. The sequence of observations \(\mathbf{y} = [y[0], y[1], \ldots, y[T]]\) is used to construct a sequence of statistics \(a_K[0], a_K[1], \ldots, a_K[T]\), whose empirical distribution should be approximately uniform if \(\tilde{p}_{t, \theta}\approx p_t\) for \(t=0, \ldots, T\), where $\tilde{p}_{t, \theta}$ is the pdf of the output r.v. $\tilde{y}[t]=g_{\theta}(z_{t}, \mathbf{h}[t])$. To build the ISL, we follow the same procedure described in Section \ref{preliminaries} to obtain a differentiable surrogate.

\subsection{Multivariate time series prediction}

For a multivariate time series the data has the form $\mathbf{y}=[\mathbf{y}[0], \ldots, \mathbf{y}[T]]$ where each element of the series is an $N$-dimensional vector, i.e., $\mathbf{y}[t]=[y_1[t], \ldots, y_N[t]]^{\top}$. A NN $g_{\theta}(\cdot, \mathbf{h}[t])$ can be trained using the same scheme as in Figure \ref{architecture ts} and an ISL loss function constructed from the $N$ marginals $p(y_{i}[t]|\mathbf{y}[0],\ldots, \mathbf{y}[t-1])$, as suggested in \cite{de2024training}, or using the ISL-slicing (Algorithm \ref{ISL Slicing Algorithm pseudocode}) method introduced in Section \ref{section Random Projections}.

\subsection{Experiments}

This subsection presents the results of multivariate long-sequence time-series forecasting on the ETTh2, ETTm1, and ETTm2 datasets using Autoformer \citep{wu2021autoformer}, Informer\citep{zhou2021informer}, LogTrans\citep{nie2022logtrans} and various ISL-based methods. The Electricity Transformer Temperature (ETT) dataset comprises four distinct subsets: two with hourly resolutions (ETTh) and two with 15-minute resolutions (ETTm), each containing seven features related to oil and load characteristics of electricity transformers collected between July 2016 and July 2018. For more details on the methods and datasets, refer to \cite{zeng2023transformers}. The evaluation spans multiple forecasting horizons (\(\tau\)) and is assessed using two key metrics: mean squared error (MSE) and mean absolute error (MAE). Detailed results are provided in Table \ref{table:combined_multivariate_time_series}.

For these experiments, we use a 1-layer RNN with 5 units, followed by batch normalization and a 2-layer MLP with 10 units per layer. The MLP has ReLU in the first layer, identity activation in the last, and 5\% dropout in the first layer. Despite this simple architecture, ISL outperforms state-of-the-art transformers. We also observe that the ISL-slicing method generally performs better than both the marginal fitting technique and other state-of-the-art methods, with improvements as the number of projections increases. Notably, the ISL models (with \(\approx\) $25$K parameters), using an RNN and a simple MLP, achieves better forecasting accuracy than many transformer models with millions of parameters \citep{zeng2023transformers}.

\begin{table}[ht]
  \centering
  \scriptsize
  \setlength{\tabcolsep}{2pt}
  \renewcommand{\arraystretch}{1.1}
  {%
    \sisetup{
      detect-weight          = true,
      table-number-alignment = center,
      round-mode             = places,
      round-precision        = 2
    }%

    \rowcolors{2}{gray!15}{white}%
    \begin{tabular}{@{} c c *{7}{S S} @{}}
      \toprule
      \textbf{DB} & \(\tau\)
        & \multicolumn{2}{c}{\textbf{Autoformer}}
        & \multicolumn{2}{c}{\textbf{Informer}}
        & \multicolumn{2}{c}{\textbf{LogTrans}}
        & \multicolumn{2}{c}{\textbf{ISL-M}}
        & \multicolumn{2}{c}{\textbf{ISL-S7}}
        & \multicolumn{2}{c}{\textbf{ISL-S10}}
        & \multicolumn{2}{c}{\textbf{ISL-S20}} \\
      \cmidrule(lr){3-4}\cmidrule(lr){5-6}\cmidrule(lr){7-8}%
      \cmidrule(lr){9-10}\cmidrule(lr){11-12}\cmidrule(lr){13-14}\cmidrule(lr){15-16}
      & 
        & {MSE} & {MAE}
        & {MSE} & {MAE}
        & {MSE} & {MAE}
        & {MSE} & {MAE}
        & {MSE} & {MAE}
        & {MSE} & {MAE}
        & {MSE} & {MAE} \\
      \midrule

      \multirow{4}{*}{\texttt{ETTh2}}
        &  96 
          & 0.36 & {\bfseries 0.40}   & 3.76 & 1.53   & 2.12 & 1.20 
          & 0.42 & 0.49   & 0.31 & 0.45   & 0.28 & 0.43   
          & {\bfseries 0.26} & 0.41   \\
       \texttt{ETTh2} & 192 
          & {\bfseries 0.46} & {\bfseries 0.45}   & 5.60 & 1.93   & 4.32 & 1.63 
          & 0.68 & 0.69   & 0.61 & 0.64   & 0.54 & 0.57   
          & 0.55 & 0.60   \\
        & 336 
          & 0.48 & 0.49   & 4.72 & 1.84   & 1.12 & 1.60 
          & 0.61 & 0.64   & 0.41 & 0.55   
          & {\bfseries 0.33} & {\bfseries 0.43}
          & 0.44 & 0.51   \\
        & 720 
          & {\bfseries 0.52} & {\bfseries 0.51}   & 3.65 & 1.63   & 3.19 & 1.54 
          & 0.60 & 0.62   & 0.75 & 0.68   
          & 0.56 & 0.57   
          & 0.65 & 0.64   \\

      \midrule

      \multirow{4}{*}{\texttt{ETTm1}}
        &  96 
          & 0.67 & 0.57   & 0.54 & 0.51   & 0.60 & 0.55 
          & 0.41 & 0.51   & 0.43 & 0.55   
          & 0.28 & 0.43   
          & {\bfseries 0.19} & {\bfseries 0.37}   \\
       \texttt{ETTm1}  & 192 
          & 0.80 & 0.67   & 0.56 & 0.54   & 0.84 & 0.70 
          & 0.57 & 0.62   & 0.67 & 0.65   
          & 0.35 & {\bfseries 0.36}   
          & {\bfseries 0.30} & 0.49   \\
        & 336 
          & 1.21 & 0.87   & 0.75 & 0.66   & 1.12 & 0.83 
          & 0.70 & 0.67   & 0.79 & 0.78   
          & 0.67 & 0.68   
          & {\bfseries 0.54} & {\bfseries 0.61}   \\
        & 720 
          & 1.17 & 0.82   & 0.91 & 0.72   & 1.15 & 0.82 
          & 0.77 &  {\bfseries 0.69}   & 0.92 & 0.73   
          & 0.87 & 0.77   
          & {\bfseries 0.76} & 0.74   \\

      \midrule

      \multirow{4}{*}{\texttt{ETTm2}}
        &  96 
          & 0.26 & {\bfseries 0.34}   & 0.37 & 0.45   & 0.77 & 0.64 
          & 0.28 & 0.44   & {\bfseries 0.21} & 0.38   
          & 0.22 & 0.37   
          & 0.22 & 0.36   \\
        \texttt{ETTm2} & 192 
          & 0.28 & 0.34   & 0.53 & 0.56   & 0.99 & 0.76 
          & 0.34 & 0.46   & 0.33 & 0.43   
          & 0.26 & 0.40   
          & {\bfseries 0.21} & {\bfseries 0.34}   \\
        & 336 
          & 0.34 & 0.37   & 1.36 & 0.89   & 1.33 & 0.87 
          & 0.31 & 0.46   & 0.37 & 0.47   
          & 0.31 & 0.45   
          & {\bfseries 0.15} & {\bfseries 0.31}   \\
        & 720 
          & 0.43 & 0.43   & 3.38 & 1.34   & 3.05 & 1.33 
          & 0.43 & 0.55   & 0.37 & 0.50   
          & 0.36 & 0.49   
          & {\bfseries 0.24} & {\bfseries 0.40}   \\

      \bottomrule
    \end{tabular}
  }
\caption{\small Forecasting results for multivariate time-series on ETTh2, ETTm1, and ETTm2 datasets using Autoformer, Informer, LogTrans and ISL. ISL-M denotes the marginal approach, and ISL-S$n$ is the ISL-slicing with \(n\) random projections.}
\label{table:combined_multivariate_time_series}
\end{table}

\section{Conclusions}  \label{section 7}

We have investigated the construction of loss functions based on rank statistics with an invariant distribution in order to train implicit generative models that can reproduce heavy-tailed and multivariate data distributions. As a result, we have introduced two novel methodologies, Pareto-ISL and ISL-slicing. Pareto-ISL is used to fit a NN with a Pareto-distributed random input. The resulting model can represent accurately both the central features and the tails of heavy-tailed data distributions. ISL-slicing can be used to train NN with a random input and a high-dimensional output. We have shown through simulations with several data sets (including time series data) that the proposed methods are competitive with (and often outperform) state of the art models. The computational cost of training generative models with ISL-based methods is usually low compared to that of state of the art schemes.


\acks{This work has been partially supported by the Office of Naval Research (award N00014-22-1-2647) and Spain’s Agencia Estatal de Investigación (ref. PID2021-125159NB-I00 TYCHE and PID2021-123182OB-I00 EPiCENTER) funded by MCIN/\-AEI/10.13039/501100011033 and by “ERDF A way of making Europe". Also funded by Comunidad de Madrid IND2022/TIC-23550, IDEA-CM project (TEC-2024/COM-89) and the ELLIS Unit Madrid (European Laboratory for Learning and Intelligent Systems). Pablo M. Olmos was also supported by the 2024 Leonardo Grant for Scientific Research and Cultural Creation from the BBVA Foundation.}


\newpage

\appendix
\counterwithin{equation}{section}
\section{Proof of Theorem 2} \label{Proof of Theorem 2}

\begin{proof} 
    For clarity, we introduce the following notation: for a real function \( f: \mathbb{R} \to \mathbb{R} \) and the pdf \( p \) of a univariate r.v. we define
    \[
        (f, p) := \int_{\mathbb{R}} f(x) p(x) \, dx.
    \]
    Let \( B_1(\mathbb{R}) := \{ f: \mathbb{R} \to \mathbb{R} \mid \sup_{x \in \mathbb{R}} |f(x)| \leq 1 \} \) be the set of real functions bounded by 1. We note that the assumption $\sup_{f \in B_1(\mathbb{R})} \left| (f, p) - (f, \tilde{p}) \right|$ in \cite[Theorem 2]{de2024training} is equivalent to the assumption of Theorem \ref{theorem4}. Indeed,
    \begin{align*}
        \norm{p - \tilde{p}}_{L^{1}(\mathbb{R})} = \int_{\mathbb{R}} \abs{p(x) - \tilde{p}(x)} \, dx &= \int_{\mathbb{R}} \operatorname{sgn}(p(x) - \tilde{p}(x)) (p(x) - \tilde{p}(x)) \, dx \\&= \sup_{f \in B_1(\mathbb{R})} \left| (f, p) - (f, \tilde{p}) \right|.
    \end{align*}
    The remainder of the proof follows directly from \cite[Theorem 2]{de2024training}.
\end{proof}

\section{Proof of Theorem 4} \label{Proof of Theorem 4}
\setcounter{theorem}{0}
\subsection{Preliminary results}

Let, \( \tilde{F}_{\theta}(y_0) \) represents the cdf of the transformed r.v. \( y = g_{\theta}(z) \), evaluated at \( y_0 \), where \( g_{\theta} \) is a function parameterized by \( \theta \). The r.v. $y$ has a pdf denoted $\tilde{p}_{\theta}$ and $z\sim p_{z}$.

Let $\tilde{y}_1, \ldots, \tilde{y}_k$ be $K$ iid draws from $\tilde{p}_{\theta}$. Using \( \tilde{F}_{\theta}(y_0) \), we define the probability of observing exactly \( n \) successes in \( K \) independent Bernoulli trials. The $i$-th trial is considered a success when $\tilde{y}_{i}\leq y_{0}$, hence the success probability is \( \tilde{F}_{\theta}(y_0) \). We write of the resulting binomial distribution as
\begin{align}\label{eq1:h_n}
h_{n,\theta}(y_0) = \binom{K}{n} \left[ \tilde{F}_{\theta}(y_0) \right]^n \left[ 1 - \tilde{F}_{\theta}(y_0) \right]^{K - n}, \quad \text{for}\;\; n\in\{0, \ldots K\}.
\end{align}

Using \ref{eq1:h_n}, the pmf of the rank statistic $A_{K}=\abs{\left\{\tilde{y}\in\{\tilde{y}_{i}\}^{K}_{i=1} : \tilde{y}\leq y_{0}\right\}}$ when $y_{0}\sim \tilde{p}_{\theta}$ can be constructed as
\begin{align}
\mathbb{Q}_{K, \tilde{p}_{\theta}}(n) = \int_{\mathbb{R}} h_{n,\theta}(y) \, \tilde{p}_{\theta}(y) \, dy, \quad \text{for}\;\; n\in\{0,\ldots, K\}.
\label{eqConstructQ}
\end{align}

\begin{lemma}\label{lemma1}
The cdf \( \tilde{F}_{\theta}(y_0) = \mathbb{P}(g_{\theta}(z) \leq y_0) \) is continuous in \( \theta \), for every fixed \( y_0 \in \mathbb{R} \).
\end{lemma}
\begin{proof}
    The cdf $\tilde{F}_{\theta}(y_0)$ can be expressed as
    \begin{align*}
        \tilde{F}_{\theta}(y_{0}) = \mathbb{P}(y\leq y_{0})= \mathbb{P}(g_{\theta}(z)\leq y_{0})=\int_{\mathbb{R}}\mathbb{I}_{S_{y_{0}, \theta}}(z)\, p_{z}(z) \,dz,
    \end{align*}
    where \( p_z \) is the pdf of the input noise, $S_{y_{0}, \theta}=\left\{u\in \mathbb{R}: g_{\theta}(u)\leq y_{0}\right\}$ and $\mathbb{I}_{S_{y_{0}, \theta}}(z)$ is the indicator function. We need to prove that for any sequence \( \{\theta_m\}_{m=1}^{\infty} \) such that \( \theta_m \to \theta \) as \( m \to \infty \), we have \( \lim_{m\to \infty}\left| \tilde{F}_{\theta_m}(y_0) - \tilde{F}_{\theta}(y_0) \right| = 0 \). We can write
    \begin{align*}
    \left| \tilde{F}_{\theta_m}(y_0) - \tilde{F}_{\theta}(y_0) \right| &\leq  \int_{\mathbb{R}} \left| \mathbb{I}_{S_{y_{0}, \theta_{m}}}(z) - \mathbb{I}_{S_{y_{0}, \theta}}(z) \right| p_z(z) \, dz.
    \end{align*}

    \noindent If we define the set $A_{m}=\left\{u\in \mathbb{R}: \left( g_{\theta}(u) \leq y_0 < g_{\theta_m}(u) \right) \text{ or } \left( g_{\theta_m}(u) \leq y_0 < g_{\theta}(u) \right)\right\}$ then is clear that 
    \begin{align*}
    \left| \mathbb{I}_{S_{y_{0}, \theta_{m}}}(z) - \mathbb{I}_{S_{y_{0}, \theta}}(z) \right|= \begin{cases}
    1 & \text{if } z \in A_{m}, \\
    0 & \text{otherwise}.
    \end{cases}
     \end{align*}
     Hence, we can write 
     \begin{align}\label{eq1:lemma1}
     \left| \tilde{F}_{{\theta}_{m}}(y_0) - \tilde{F}_{\theta}(y_0)\right| = \int_{\mathbb{R}} \mathbb{I}_{A_{m}}(z)\, p_{z}(z)\,dz = \int_{A_{m}} p_{z}(z)\,dz,
     \end{align}
     and there is a constant $C<\infty$ such that $\int_{A_{m}}p_{z}\,dz\leq C\int_{A_{m}}dz$. However $g_{\theta}(z)$ is continuous in $\theta$ for almost every $z\in \mathbb{R}$, i.e., $\lim_{m\to \infty}g_{\theta_{m}}(z)=g_{\theta}(z)$, which implies that $\lim_{m\to \infty}\int_{A_{m}}dz=0$. Therefore $\lim_{m\to \infty}\abs{\tilde{F}_{\theta_{m}}(y_0)-\tilde{F}_{\theta}(y_0)} \leq C\int_{A_{m}}dz = 0$.
\end{proof}

\begin{lemma} \label{lemma2}
    For each \( n \in \{0, 1, \dots, K\} \), there is a constant $C_{n,K}<\infty$ such that 
    \begin{align*}
    \left| h_{n,\theta}(y_0) - h_{n,\theta'}(y_0) \right| \leq \binom{K}{n}\, C_{n,K}\, \left| \tilde{F}_{\theta}(y_0) - \tilde{F}_{\theta'}(y_0) \right|.
    \end{align*}
\end{lemma}
\begin{proof}
    We can bound the difference between \(h_{n,\theta}(y_0)\) and \(h_{n,\theta'}(y_0)\) using \ref{eq1:h_n} as 
    \begin{align*}
        &\left|h_{n,\theta}(y_0) - h_{n,\theta'}(y_0)\right| \leq \binom{K}{n} \left| \tilde{F}_{\theta}(y_0)^n (1 - \tilde{F}_{\theta}(y_0))^{K-n} - \tilde{F}_{\theta'}(y_0)^n (1 - \tilde{F}_{\theta'}(y_0))^{K-n} \right|.
    \end{align*}
    Let us define \( f(q) = q^n (1 - q)^{K-n} \), which is continuously differentiable for \( q \in [0,1] \). By the mean value theorem, for some value \( \tilde{F}_{\theta^*}(y_0) \in \left[\tilde{F}_{\theta}(y_0) \land \tilde{F}_{\theta'}(y_0), \tilde{F}_{\theta}(y_0)\lor\tilde{F}_{\theta'}(y_0)\right] \), we have
    \begin{align*}
        \left| f(\tilde{F}_{\theta}(y_0)) - f(\tilde{F}_{\theta'}(y_0)) \right| &\leq \left| f'(\tilde{F}_{\theta^{*}}(y_0)) \right| \left| \tilde{F}_{\theta}(y_0) - \tilde{F}_{\theta'}(y_0) \right|.
    \end{align*}
    Since \( f(q) = q^n (1 - q)^{K-n} \) is a polynomial in \( q \), its derivative \( f'(q) \) is continuous and bounded on the interval \( q \in [0, 1] \). Given that \( \tilde{F}_{\theta^*}(y_0) \in [0, 1] \) (as it is a cdf), there exists a constant \( C_n \) such that, $
    \left| f'(\tilde{F}_{\theta^*}(y_0)) \right| \leq C_n$ for any $\tilde{F}_{\theta^*}(y_0)$.
\end{proof}

\begin{lemma}\label{lemma3}
    Let \( \tilde{F}_{\theta}(y_0) \) and \( \tilde{F}_{\theta'}(y_0) \) be the cdfs of the transformed r.v.s \( y = g_{\theta}(z) \) and \( y' = g_{\theta'}(z) \), respectively, where \( g_{\theta}(z) \) is differentiable w.r.t. \( z \) and satisfies the Lipschitz condition \( \left| g_{\theta}(z) - g_{\theta'}(z) \right| \leq L_{\max} \|\theta - \theta'\| \) for some Lipschitz constant \(L_{\max}<\infty\), and there exists $m>0$ such that \( \inf_{(z,\theta)} \left| g_{\theta}'(z) \right| \geq m  \). Then
    \begin{align*}
    \left| \tilde{F}_{\theta}(y_0) - \tilde{F}_{\theta'}(y_0) \right| \leq L_1 \|\theta - \theta'\|,
    \end{align*}
    where \( L_1 = \|p_z\|_{L^\infty(\mathbb{R})} \, \dfrac{2 L_{\max}}{m} \) and \( p_z(z) \) is the pdf of the input variable \( z \).
\end{lemma}
\begin{proof}
    Let $A=\left\{u\in \mathbb{R}: \left( g_{\theta}(u) \leq y_0 < g_{\theta'}(u) \right) \text{ or } \left( g_{\theta'}(u) \leq y_0 < g_{\theta}(u) \right)\right\}$. By the same argument as in the proof of Lemma \ref{lemma1} (see Eq. \ref{eq1:lemma1}) we have
    \begin{align}\label{eq1:lemma3}
    \left| \tilde{F}_{\theta}(y_0) - \tilde{F}_{\theta'}(y_0) \right|\leq \int_{A} p_z(z) \, dz.
    \end{align}
    
    We now prove that $A\subseteq B= \left\{ u \in \mathbb{R} \ \bigg| \ \left| g_{\theta}(u) - y_0 \right| \leq L_{\max} \, \|\theta - \theta'\| \right\}$. For any $z \in A$, there are two possible cases to consider: either \(g_{\theta}(z) \leq y_0 < g_{\theta'}(z)\) or $g_{\theta'}(z) \leq y_0 < g_{\theta}(z)$. In the first case, we can see that \( |g_{\theta}(z) - y_0| \leq |g_{\theta}(z) - g_{\theta'}(z)| \leq L_{\max} \, \| \theta - \theta' \| \) by the Lipschitz assumption. An analogous argument holds for the second case.  Thus, we have shown that for any \( z \in A \), $ \left| g_{\theta}(z) - y_0 \right| \leq L_{\max} \, \|\theta - \theta'\|$ and, therefore, \( A \subseteq B \).

    Next, we estimate the Lebesgue measure $\mathcal{L}(B)$ of the set $B$. By assumption, $g_{\theta}(z)$ is differentiable w.r.t. $z$, with \(\inf_{(z, \theta)\in \mathcal{Z}\times \mathbb{R}^{d}}\left| g_{\theta}'(z) \right| \geq m > 0\). Hence, we can perform a change of variable \(y = g_{\theta}(z)\), and obtain \(z = g_{\theta}^{-1}(y)\), where $g^{-1}_{\theta}$ is the local inverse function. We note that $z\in B$ when  $y\in[ y_0 - \delta, y_0 + \delta ]$, where \(\delta = L_{\max} \, \|\theta - \theta'\|\). The Lebesgue measure of \(B\) can be upper bounded as
    \begin{align*}
        \mathcal{L}(B) = \int_{B} dz \leq \int_{y_0 - \delta}^{y_0 + \delta} \left| \frac{dz}{dy} \right| dy \leq \frac{1}{m} \int_{y_{0}-\delta}^{y_{0}+\delta}\, dy = \frac{2 \delta}{m} = \frac{2 L_{\max}}{m} \, \|\theta - \theta'\|,
    \end{align*}
    since \( \frac{dz}{dy} = \frac{1}{g_{\theta}'(z)} \) and, by assumption, \( \inf_{(z,\theta)} \left| g_{\theta}'(z) \right| \geq m \), hence
    $\left| \frac{dz}{dy} \right| = \frac{1}{\left| g_{\theta}'(z) \right|} \leq \frac{1}{m}
    $.

    \noindent Finally, we proceed to bound the absolute difference between the cdfs. In particular,
    \begin{align*}
    \left| \tilde{F}_{\theta}(y_0) - \tilde{F}_{\theta'}(y_0) \right| \leq \int_{A} p_z(z) \, dz \leq \int_{B} p_z(z) \, dz &\leq \norm{p_{z}}_{L^{\infty}(\mathbb{R})} \, \mathcal{L}(B) \leq L_1 \norm{\theta - \theta'},
    \end{align*}
    the first inequality follows from \ref{eq1:lemma3}, the second is given by $A\subseteq B$, and the last inequality follow from the upper bound on $\mathcal{L}(B)$ with $L_{1} = \norm{p_{z}}_{L^{\infty}(\mathbb{R})} \, \dfrac{2 L_{\max}}{m}$.
\end{proof}

\subsection{Proof of Theorem 4}

\begin{proof}\noindent\underline{Continuity (Part 1)}\\[3pt]

    To prove that \( d_{K}(p, \tilde{p}_{\theta}) \) is continuous w.r.t. \( \theta \), we need to show that $$\lim_{m\to\infty} \left| d_{K}(p, \tilde{p}_{\theta_m}) - d_{K}(p, \tilde{p}_{\theta}) \right| = 0$$ for any sequence $\{ \theta_m \}_{m=1}^\infty$ such that $\theta_m\to \theta$.
    
    We begin by examining the difference \( \left| d_{K}(p, \tilde{p}_{\theta_m}) - d_{K}(p, \tilde{p}_{\theta}) \right| \). Let us denote ${\bf q}_{K,\tilde p_\theta}=\left[ \mathbb{Q}_{K, \tilde{p}_\theta}(0), \ldots, \mathbb{Q}_{K, \tilde{p}_\theta}(K) \right]^\top$. We readily arrive at the bound
    \begin{align}
        \abs{d_{K}(p, \tilde{p}_{\theta_m}) - d_{K}(p, \tilde{p}_{\theta})} &= \frac{1}{K+1} \left|\left\|\frac{1}{K+1} \mathbf{1}_{K+1} - {\bf q}_{K, \tilde{p}_{\theta_m}}\right\|_{\ell_1} - \left\|\frac{1}{K+1}\mathbf{1}_{K+1} - {\bf q}_{K, \tilde{p}_{\theta}}\right\|_{\ell_1}\right| \nonumber\\
        &\leq \frac{1}{K+1} \left\|{\bf q}_{K, \tilde{p}_{\theta_m}} - {\bf q}_{K, \tilde{p}_{\theta}}\right\|_{\ell_1} \label{eqReverseTriangle}\\
        &\leq \dfrac{1}{K+1}\sum_{n=0}^{K} \int_{\mathbb{R}} \abs{h_{n,\theta_m}(y)-h_{n,\theta}(y)}\, \max\left\{\tilde{p}_{\theta_m}(y), \tilde{p}_{\theta}(y)\right\}\,dy,
        \label{eq1:Proof of Theorem 4 part 1}
    \end{align}
    where \eqref{eqReverseTriangle} follows from the reverse triangle inequality and \eqref{eq1:Proof of Theorem 4 part 1} is obtained from the construction of the pmf $\mathbb{Q}_{K,\tilde p_\theta}(n)$ in Eq. \eqref{eqConstructQ}. It is easy to see that 
    $$
    \left| h_{n,\theta_m}(y) - h_{n,\theta}(y) \right| \max\{\tilde{p}_{\theta_m}(y), \tilde{p}_{\theta}(y)\} \le 2\,( \tilde{p}_{\theta}(y) + \tilde{p}_{\theta_m}(y) ),
    $$
    hence the dominated convergence theorem yields
    \begin{align}
        \lim_{m\to \infty}\abs{d_{K}(p, \tilde{p}_{\theta_m})- d_{K}(p, \tilde{p}_{\theta})} \leq \sum_{n=0}^{K} \int_{\mathbb{R}} \lim_{m\to \infty} \abs{h_{n,\theta_m}(y)-h_{n,\theta}(y)}\, \max\left\{\tilde{p}_{\theta_m}(y), \tilde{p}_{\theta}(y)\right\}\,dy.
        \label{eqBoundDkLim}
    \end{align}
    However, \( h_{n,\theta_m}(y) \) is continuous in $\theta$ because it depends continuously on the cdf \( \tilde{F}_{\theta_{m}}(y) \) which, in turn, is continuous by Lemma \ref{lemma1}. Therefore, 
    \begin{equation}
    \lim_{m\to \infty} \abs{h_{n,\theta_m}(y)-h_{n,\theta}(y)} = 0
    \label{eqConvergenceH}
    \end{equation}
    for any sequence $\{ \theta_m \}_{m=1}^\infty$ such that $\theta_m\to \theta$. Combining \eqref{eqConvergenceH} with the inequality \eqref{eqBoundDkLim} yields $\lim_{m\to\infty} \abs{d_{K}(p, \tilde{p}_{\theta_m})- d_{K}(p, \tilde{p}_{\theta})} = 0$ whenever $\theta_m \to \theta$ and completes the proof of Part 1.
    \end{proof}

    \begin{proof}\noindent\underline{Differentiability (Part 2).}\\[3pt]

    By Rademacher's theorem (see \cite[Theorem 3.2]{evans2018measure}), if $d_K(p,\tilde p_\theta)$ is Lipschitz then it is differentiable almost everywhere. Hence, we aim to prove that $d_{K}(p, \tilde{p}_{\theta})$ is Lipschitz continuous w.r.t. $\theta$.

    Lemma \ref{lemma2} yields the upper bound
    \begin{align}
        \abs{h_{n,\theta}(y_0) - h_{n,\theta'}(y_0)}\leq \binom{K}{n}\, C_{n,K}\, \abs{\tilde{F}_{\theta}(y_0)-\tilde{F}_{\theta'}(y_0)},
        \label{eqPart2_x1}
    \end{align}
    where $C_{n,K}<\infty$ is a constant that depends on $n$ and $K$. Combining \eqref{eqPart2_x1} with Lemma \ref{lemma3} yields the Lipschitz continuity of $h_{n,\theta}(y_0)$, namely
    \begin{align}\label{eq1:part2 proof}
        \abs{h_{n,\theta}(y_0) - h_{n,\theta'}(y_0)}\leq  \binom{K}{n}\, C_{n,K}\, L_{1} \norm{\theta - \theta'},
    \end{align}
    where $L_1<\infty$ is a constant. Moreover, we have a bound for the pdf $\tilde p_\theta$ of the form
    \begin{equation}
    \sup_{\theta,y} \tilde{p}_{\theta}(y) = \sup_{\theta,y} \frac{p_z(g_{\theta}^{-1}(y))}{\left| g'_{\theta}\left( g_{\theta}^{-1}(y) \right) \right|} \leq \frac{\|p_z\|_{L^\infty(\mathbb{R})}}{m},
    \label{eqPart2_x2}
    \end{equation}
    where $g_\theta^{-1}(y)$ is the local inverse of $g_\theta(y)$ and we have used the assumption $\inf_{z,\theta} \abs{g_\theta'(z)} \ge m > 0$. Substituting the upper bounds \eqref{eq1:part2 proof} and \eqref{eqPart2_x2} back into \eqref{eq1:Proof of Theorem 4 part 1} (with $\theta_m=\theta'$) yields
    \begin{equation}
    \left|d_{K}(p, \tilde{p}_{\theta}) - d_{K}(p, \tilde{p}_{\theta'})\right| \leq 
    \dfrac{1}{K+1}\dfrac{\norm{p_{z}}_{L^{\infty}(\mathbb{R})}}{m} \, \sum_{n=0}^{K}\binom{K}{n} \, C_{n,K} L_1 \| \theta - \theta' \|.
    \nonumber
    \end{equation}
    Finally, we note that $\sum_{n=0}^{K} \binom{K}{n} = 2^K$ to obtain
    \begin{align*}
        \left| d_{K}(p, \tilde{p}_{\theta}) - d_{K}(p, \tilde{p}_{\theta'}) \right| \leq \frac{\|p_z\|_{L^\infty(\mathbb{R})}}{m\, (K+1)} \, 2^K \left(\max_{0 \leq n \leq K} C_{n,K} \right) L_1 \|\theta - \theta'\|.
    \end{align*}
    and complete the proof.
\end{proof}

\section{Experimental results} \label{Appendix Experimental results}

\subsection{Comparison of the surrogate and theoretical loss functions} \label{subsection: Comparison of Surrogate and Theoretical Loss Performance}

Table \ref{Learning1D table} compares the error rates between the surrogate and theoretical losses. The metrics displayed include the \(L_1\) and \(L_{\infty}\) norms of the values obtained from the surrogate loss and the theoretical loss at each epoch, along with their respective percentage errors. 
In this setup the noise is drawn from a standard Gaussian, \(\mathcal{N}(0,1)\), the number of training epochs is $1000$, \(K=10\), the sample size is \(N=1000\), and the results are averaged over 100 trials.
The final three rows in the table describe mixture models with equal component weights: Model 1 is a mixture of $\mathcal{N}(5, 2)$ and $\mathcal{N}(-1, 1)$, Model 2 is a mixture of $\mathcal{N}(5, 2)$, $\mathcal{N}(-1, 1)$, and $\mathcal{N}(-10, 3)$, and Model 3 is a mixture of $\mathcal{N}(-5, 2)$ and $\text{Pareto}(5, 1)$.

\begin{table*}[!htbp]
  \centering
  \begingroup
    \sisetup{
      separate-uncertainty   = true,
      uncertainty-separator  = \pm,
      table-number-alignment = center,
      group-digits           = false
    }
    \scriptsize
    \setlength{\tabcolsep}{3pt}
    \rowcolors{2}{gray!15}{white}

    \begin{adjustbox}{width=0.8\textwidth,center}
      \begin{tabular}{
        l
        l
        S[table-format=1.6(6)]
        S[table-format=1.6(6)]
        S[table-format=1.6]
        S[table-format=1.6]
      }
        \rowcolor{gray!30}
        \toprule
        & \textbf{Target}
        & {$L_1$}
        & {$L_{\infty}$}
        & {\% error~$L_1$}
        & {\% error~$L_{\infty}$}
        \\
        \midrule
        & $\mathcal{N}(4,2)$ 
          & 0.005307(1505) & 0.254805(34452) 
          & 0.019625 & 0.930903 \\
        & $\mathcal{U}(-2,2)$  
          & 0.015454(1450) & 0.244568(28237) 
          & 0.015454 & 0.734944 \\
        & Cauchy(1,2)      
          & 0.005192(1329) & 0.242361(28494) 
          & 0.015636 & 0.729832 \\
        & Pareto(1,1)       
          & 0.003290(4527) & 0.137377(183803)
          & 0.000547 & 0.022837 \\
        & Model$_1$        
          & 0.004701(2198) & 0.175290(25778) 
          & 0.003925 & 1.471159 \\
        & Model$_2$        
          & 0.004991(1895) & 0.173241(30297) 
          & 0.033870 & 0.117551 \\
        & Model$_3$        
          & 0.009817(2167) & 0.348440(68293) 
          & 0.016706 & 0.592932 \\
        \bottomrule
      \end{tabular}
    \end{adjustbox}
  \endgroup

  \caption{\small Comparison of error between the surrogate and theoretical losses. Noise is
    $\sim\mathcal{N}(0,1)$, $K=10$, $\text{epochs}=1000$, and $N=1000$. Entries are mean~$\pm$~std over 100 trials; last two columns show percentage error.}\label{Learning1D table}
\end{table*}

Figure \ref{Comparison error Surrogate Loss vs Theoretical figures} illustrates the surrogate loss versus the theoretical one across several distributions, presented in log-scale to highlight performance differences.

\begin{figure}[H]
    \centering
    \begin{subfigure}[b]{0.32\linewidth} 
        \centering
        \includegraphics[width=\linewidth]{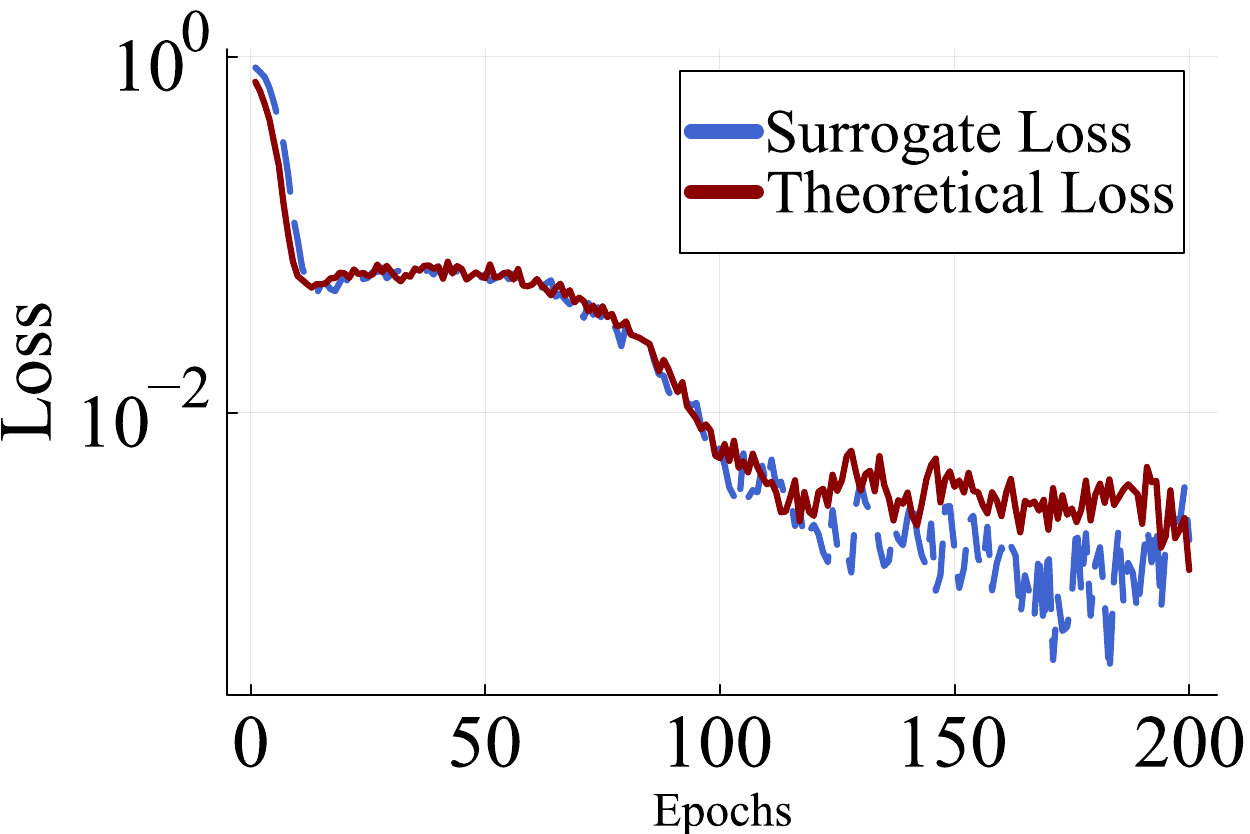}
        \caption{$\normdist{4}{2}$}
    \end{subfigure}
    \hfill
    \begin{subfigure}[b]{0.32\linewidth} 
        \centering
        \includegraphics[width=\linewidth]{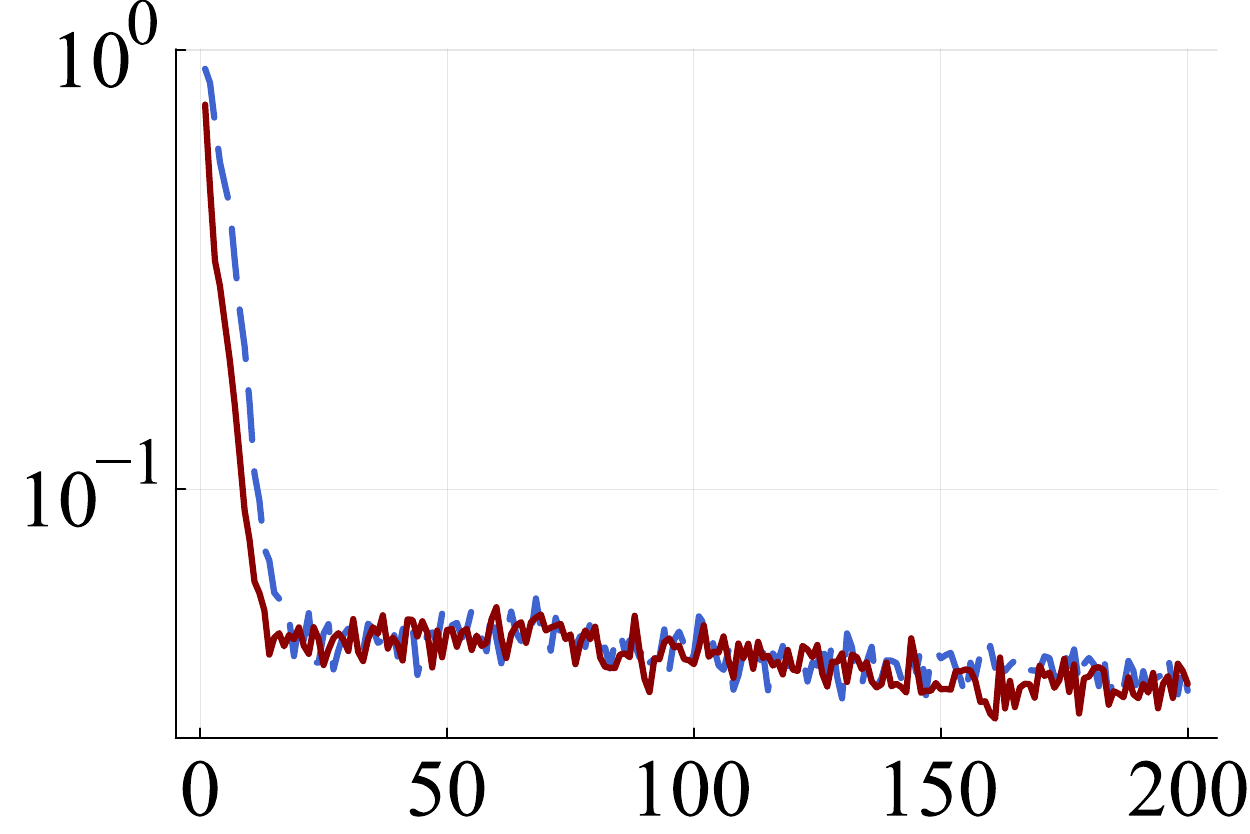}
        \caption{Pareto$(1,1)$}
    \end{subfigure}
    \hfill
    \begin{subfigure}[b]{0.32\linewidth} 
        \centering
        \includegraphics[width=\linewidth]{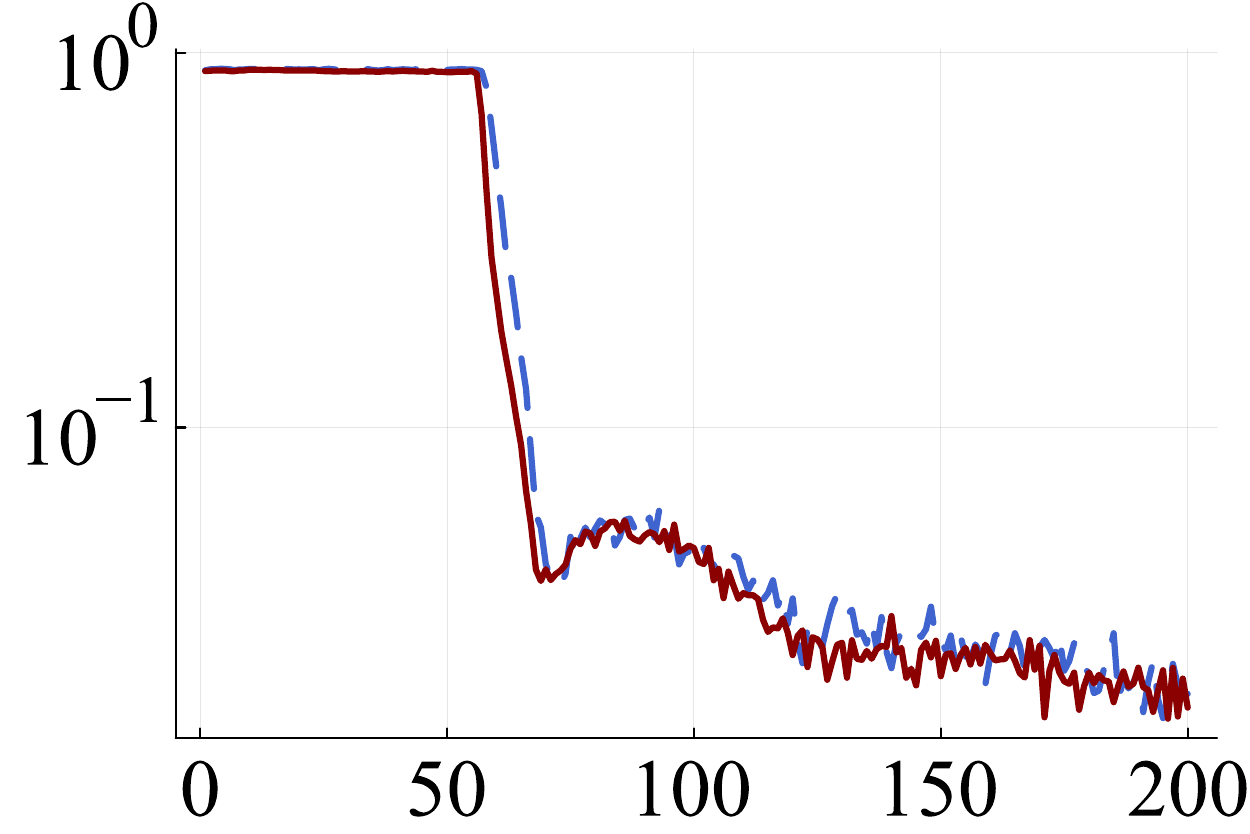}
        \caption{$\text{Model}_3$}
    \end{subfigure}
    
    \caption{\small Comparison of the surrogate loss and the theoretical one during training for different distributions. 
    Only the first 200 epochs are shown, and the scale in the vertical axis is logarithmic.
    }
    \label{Comparison error Surrogate Loss vs Theoretical figures}
\end{figure}

In conclusion, the surrogate loss performs well for different target distributions, as shown in Table \ref{Learning1D table} and Figure \ref{Comparison error Surrogate Loss vs Theoretical figures}. The small \(L_1\) and \(L_{\infty}\) norms with minimal percentage errors indicate its close approximation of the theoretical loss. Even for complex distributions like the mixture models, it maintains low error rates, demonstrating robustness and reliability.

\subsection{Efficiency gains from progressive $K$ training vs. fixed $K$} \label{subsection: Efficiency Gains from Progressive $K$ Training vs. Fixed $K$}

In Table \ref{Results K progressive training}, we compare the outcomes of progressively increasing \( K \) during training against a fixed \( K \) value. After running 100 trials and averaging the results, we found that progressively increasing \( K \) reduces training time by up to 50\% without sacrificing accuracy. Additionally, the KSD values are comparable or slightly better, indicating that this strategy offers a more efficient balance between time savings and model performance.

\begin{table}[h]
  \centering
  \begingroup
    \sisetup{
      separate-uncertainty   = true,
      uncertainty-separator  = \pm,
      table-number-alignment = center,
      group-digits           = false
    }
    \small
    \setlength{\tabcolsep}{4pt}

    \begin{adjustbox}{width=0.8\textwidth,center}
      \rowcolors{2}{gray!15}{white}
      \begin{tabular}{
        l
        S[table-format=1.5(4)]  
        S[table-format=4.0]     
        S[table-format=1.5(4)]  
        S[table-format=5.0]     
        S[table-format=2.2]     
      }
        \rowcolor{gray!30}\toprule
        & \multicolumn{2}{c}{\textbf{Progressive $K$}}
        & \multicolumn{2}{c}{\textbf{$K=10$}}
        & {\textbf{Time Imp.\,\%}} \\
        \cmidrule(lr){2-3}\cmidrule(lr){4-5}\cmidrule(lr){6-6}
        & {KSD} & {Time}
        & {KSD} & {Time}
        & {Time}\\
        \midrule
        $\mathcal{N}(4,2)$ 
          & 0.00831(268) & 4149
          & 0.00894(276) &  9016
          & 53.98 \\
        $\mathcal{U}(-2,2)$ 
          & 0.01431(158) & 8721
          & 0.01373(152) & 11745
          & 25.74 \\
        Cauchy(1,2)      
          & 0.01084(146) & 3700
          & 0.01190(161) &  8002
          & 53.76 \\
        Pareto(1,1)      
          & 0.08344(334) & 4978
          & 0.24238(17257) & 10800
          & 53.90 \\
        Model$_1$        
          & 0.01001(134) & 8000
          & 0.01175(336) & 11044
          & 27.56 \\
        Model$_2$        
          & 0.01067(270) & 11042
          & 0.00921(187) &  8039
          & 27.19 \\
        Model$_3$        
          & 0.18525(406) & 8700
          & 0.18053(4512) & 11207
          & 22.37 \\
        \bottomrule
      \end{tabular}
    \end{adjustbox}
  \endgroup
  \caption{\small Comparison of progressive‐$K$ training versus fixed $K=10$. Noise is
    $\sim\mathcal{N}(0,1)$, $K_{\max}=10$ (progressive only), epochs = 200, and $N=1000$.
    Times are in seconds; “KSD” entries are mean\,$\pm$\,std.}
  \label{Results K progressive training}
\end{table}

\subsection{Comparing ISL and Pareto-ISL for approximating heavy-tailed multi-dimensional distributions} \label{subsection: Multi-dimensional distributions heavy-tailed}

We define multidimensional distributions with heavy-tailed characteristics and train a generator to approximate them. Specifically, we introduce a joint distribution on \([X_0, X_1]\) with the following component definitions

\begin{equation*}
X_0 = A + B, \qquad X_1 = \operatorname{sign}(A - B) \, |A - B|^{1/2},
\end{equation*}
where \( A \) and \( B \) are independent Cauchy r.vs. with location $0.5$ and scale $1.0$. Note that \( X_0 \) and \( X_1 \) have different tail indices (1 and 1/2, respectively) and are not independent.

We trained a Pareto-ISL model on this distribution using a NN with an input layer of 2 features, followed by three hidden layers of 256 neurons each with ReLU activations. The input was a mixture of GPD with tail indices of 1 and 0.5. As shown in Figure \ref{fig:Approximanting_Multi-dimensional_distributions}, the marginals and joint distributions closely match the target. We also compared these results to those obtained using ISL with multivariate normal input noise, characterized by a zero mean and identity covariance matrix.

\begin{figure}[h] \label{appendix:figure:isl-pareto vs isl multidim}
    \centering
    \begin{subfigure}[tb]{0.32\linewidth}
        \centering
        \includegraphics[width=\linewidth]{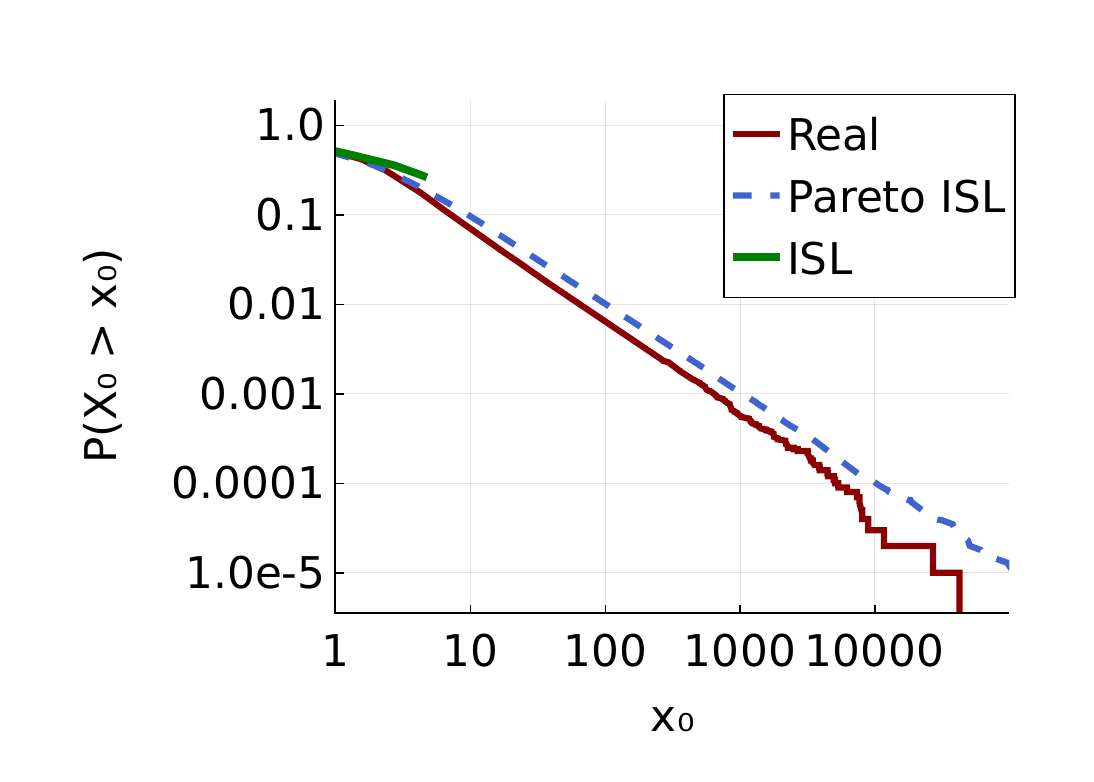}
        \subcaption{\small Marginal Distribution along \( X_0 \)}
        \label{fig:pareto_marginal_x0}
    \end{subfigure}
    \hfill
    \begin{subfigure}[tb]{0.32\linewidth}
        \centering
        \includegraphics[width=\linewidth]{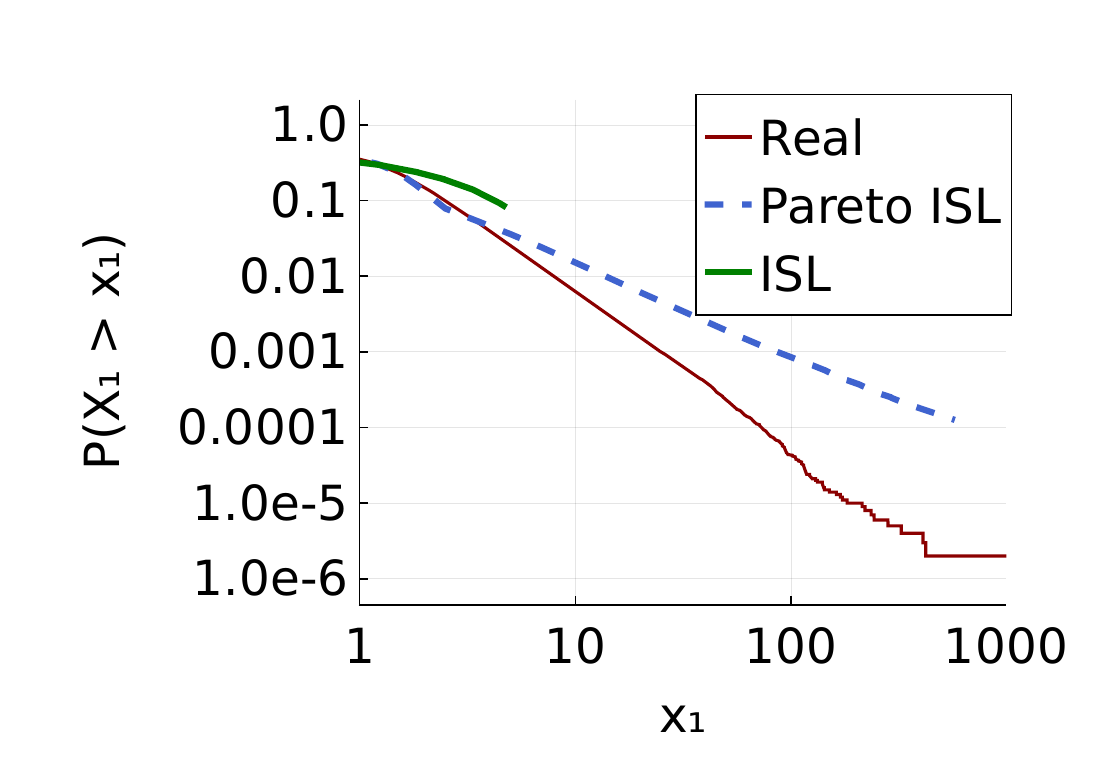}
        \subcaption{\small Marginal Distribution along \( X_1 \)}
        \label{fig:pareto_marginal_x1}
    \end{subfigure}
    \hfill
    \begin{subfigure}[tb]{0.32\linewidth}
        \centering
        \includegraphics[width=\linewidth]{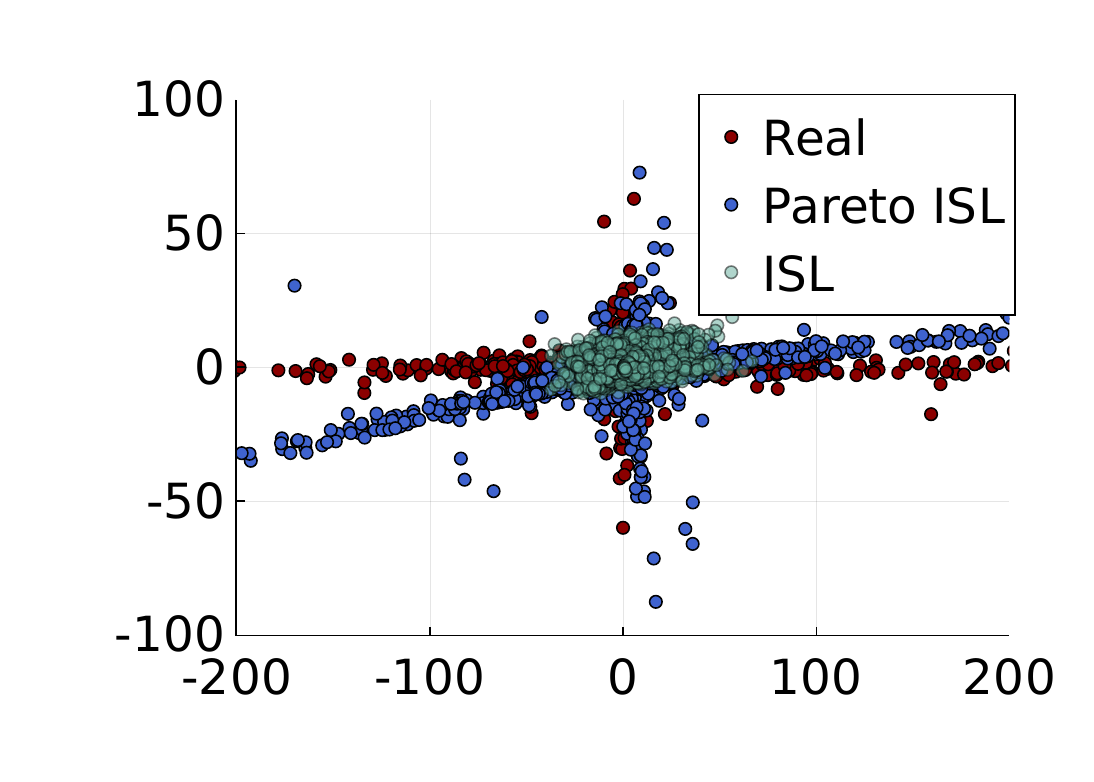}
        \subcaption{\small Joint Distribution Scatter Plot: Pareto-ISL vs. Gaussian Noise ISL}
        \label{fig:gaussian_noise_x1}
    \end{subfigure}
    \caption{\small Plots (a) and (b) show the marginal distributions along the \(X_0\) and \(X_1\) axes for the Pareto-ISL model and ISL with Gaussian noise. Plot (c) presents a scatter plot of 10000 samples from both models, illustrating the joint distribution between the two dimensions.}
    \label{fig:Approximanting_Multi-dimensional_distributions}
\end{figure}

\section{Experimental Setup}\label{Experimental Setup}
All experiments were performed on a MacBook Pro running macOS 13.2.1, equipped with an Apple M1 Pro CPU and 16 GB of RAM. When GPU acceleration was required, we used a single NVIDIA TITAN Xp with 12 GB of VRAM. Detailed hyperparameter settings for each experiment are provided in the corresponding sections.

\bibliography{sample}

\begin{thebibliography}{57}
\providecommand{\natexlab}[1]{#1}
\providecommand{\url}[1]{\texttt{#1}}
\expandafter\ifx\csname urlstyle\endcsname\relax
  \providecommand{\doi}[1]{doi: #1}\else
  \providecommand{\doi}{doi: \begingroup \urlstyle{rm}\Url}\fi

\bibitem[Arjovsky et~al.(2017)Arjovsky, Chintala, and Bottou]{arjovsky2017wasserstein}
Martin Arjovsky, Soumith Chintala, and L{\'e}on Bottou.
\newblock Wasserstein generative adversarial networks.
\newblock In \emph{International conference on machine learning}, pages 214--223. PMLR, 2017.

\bibitem[Arora et~al.(2016)Arora, Basu, Mianjy, and Mukherjee]{arora2016understanding}
Raman Arora, Amitabh Basu, Poorya Mianjy, and Anirbit Mukherjee.
\newblock Understanding deep neural networks with rectified linear units.
\newblock \emph{arXiv preprint arXiv:1611.01491}, 2016.

\bibitem[Arora et~al.(2018)Arora, Ge, Liang, Ma, and Zhang]{arora2018principled}
Sanjeev Arora, Rong Ge, Yingyu Liang, Tengyu Ma, and Yi~Zhang.
\newblock Towards principled methods for training generative adversarial networks.
\newblock In \emph{International Conference on Learning Representations (ICLR)}, 2018.

\bibitem[Balkema and De~Haan(1974)]{balkema1974residual}
August~A Balkema and Laurens De~Haan.
\newblock Residual life time at great age.
\newblock \emph{The Annals of probability}, 2\penalty0 (5):\penalty0 792--804, 1974.

\bibitem[Bellemare et~al.(2017)Bellemare, Danihelka, Dabney, Mohamed, Lakshminarayanan, Hoyer, and Munos]{bellemare2017cramer}
Marc~G Bellemare, Ivo Danihelka, Will Dabney, Shakir Mohamed, Balaji Lakshminarayanan, Stephan Hoyer, and R{\'e}mi Munos.
\newblock The cramer distance as a solution to biased wasserstein gradients.
\newblock \emph{arXiv preprint arXiv:1705.10743}, 2017.

\bibitem[Bhatia et~al.(2021)Bhatia, Jain, and Hooi]{bhatia2021exgan}
Siddharth Bhatia, Arjit Jain, and Bryan Hooi.
\newblock Exgan: Adversarial generation of extreme samples.
\newblock In \emph{Proceedings of the AAAI Conference on Artificial Intelligence}, volume~35, pages 6750--6758, 2021.

\bibitem[Bonneel et~al.(2015)Bonneel, Rabin, Peyr{\'e}, and Pfister]{bonneel2015sliced}
Nicolas Bonneel, Julien Rabin, Gabriel Peyr{\'e}, and Hanspeter Pfister.
\newblock Sliced and radon wasserstein barycenters of measures.
\newblock \emph{Journal of Mathematical Imaging and Vision}, 51:\penalty0 22--45, 2015.

\bibitem[Br{\'e}zis(2011)]{brezis2011functional}
Haim Br{\'e}zis.
\newblock \emph{Functional analysis, Sobolev spaces and partial differential equations}, volume~2.
\newblock Springer, 2011.

\bibitem[Brock et~al.(2018)Brock, Donahue, and Simonyan]{brock2018large}
Andrew Brock, Jeff Donahue, and Karen Simonyan.
\newblock Large scale gan training for high fidelity natural image synthesis.
\newblock \emph{arXiv preprint arXiv:1809.11096}, 2018.

\bibitem[Chen et~al.(2016)Chen, Duan, Houthooft, Schulman, Sutskever, and Abbeel]{chen2016infogan}
Xi~Chen, Yan Duan, Rein Houthooft, John Schulman, Ilya Sutskever, and Pieter Abbeel.
\newblock Infogan: Interpretable representation learning by information maximizing generative adversarial nets.
\newblock \emph{Advances in neural information processing systems}, 29, 2016.

\bibitem[Chen et~al.(2023)Chen, Katsoulakis, Rey-Bellet, and Zhu]{chen2023sample}
Ziyu Chen, Markos Katsoulakis, Luc Rey-Bellet, and Wei Zhu.
\newblock Sample complexity of probability divergences under group symmetry.
\newblock In \emph{International Conference on Machine Learning}, pages 4713--4734. PMLR, 2023.

\bibitem[Choi and Han(2022)]{choi2022mcl}
Jinyoung Choi and Bohyung Han.
\newblock Mcl-gan: Generative adversarial networks with multiple specialized discriminators.
\newblock \emph{Advances in Neural Information Processing Systems}, 35:\penalty0 29597--29609, 2022.

\bibitem[Coles et~al.(2001)Coles, Bawa, Trenner, and Dorazio]{coles2001introduction}
Stuart Coles, Joanna Bawa, Lesley Trenner, and Pat Dorazio.
\newblock \emph{An introduction to statistical modeling of extreme values}, volume 208.
\newblock Springer, 2001.

\bibitem[de~Frutos et~al.(2024)de~Frutos, Olmos, V{\'a}zquez, and M{\'\i}guez]{de2024training}
Jos{\'e}~Manuel de~Frutos, Pablo Olmos, Manuel~A. V{\'a}zquez, and Joaqu{\'\i}n M{\'\i}guez.
\newblock Training implicit generative models via an invariant statistical loss.
\newblock In \emph{International Conference on Artificial Intelligence and Statistics}, pages 2026--2034. PMLR, 2024.

\bibitem[Dinh et~al.(2016)Dinh, Sohl-Dickstein, and Bengio]{dinh2016density}
Laurent Dinh, Jascha Sohl-Dickstein, and Samy Bengio.
\newblock Density estimation using real nvp.
\newblock \emph{arXiv preprint arXiv:1605.08803}, 2016.

\bibitem[Djuric and M{\'\i}guez(2010)]{djuric2010assessment}
Petar~M Djuric and Joaqu{\'\i}n M{\'\i}guez.
\newblock Assessment of nonlinear dynamic models by kolmogorov--smirnov statistics.
\newblock \emph{IEEE transactions on signal processing}, 58\penalty0 (10):\penalty0 5069--5079, 2010.

\bibitem[Durugkar et~al.(2016)Durugkar, Gemp, and Mahadevan]{durugkar2016generative}
Ishan Durugkar, Ian Gemp, and Sridhar Mahadevan.
\newblock Generative multi-adversarial networks.
\newblock \emph{arXiv preprint arXiv:1611.01673}, 2016.

\bibitem[Elvira et~al.(2021)Elvira, Miguez, and Djuri{\'c}]{elvira2021performance}
V{\'\i}ctor Elvira, Joaqu{\'\i}n Miguez, and Petar~M Djuri{\'c}.
\newblock On the performance of particle filters with adaptive number of particles.
\newblock \emph{Statistics and Computing}, 31:\penalty0 1--18, 2021.

\bibitem[Evans(2018)]{evans2018measure}
LawrenceCraig Evans.
\newblock \emph{Measure theory and fine properties of functions}.
\newblock Routledge, 2018.

\bibitem[Gong et~al.(2024)Gong, Xie, Xie, and Ma]{gong2024testing}
Yanxiang Gong, Zhiwei Xie, Mei Xie, and Xin Ma.
\newblock Testing generated distributions in gans to penalize mode collapse.
\newblock In \emph{International Conference on Artificial Intelligence and Statistics}, pages 442--450. PMLR, 2024.

\bibitem[Goodfellow et~al.(2014)Goodfellow, Pouget-Abadie, Mirza, Xu, Warde-Farley, Ozair, Courville, and Bengio]{goodfellow2014generative}
Ian Goodfellow, Jean Pouget-Abadie, Mehdi Mirza, Bing Xu, David Warde-Farley, Sherjil Ozair, Aaron Courville, and Yoshua Bengio.
\newblock Generative adversarial nets.
\newblock \emph{Advances in neural information processing systems}, 27, 2014.

\bibitem[Gorban and Tyukin(2018)]{gorban2018blessing}
Alexander~N Gorban and Ivan~Yu Tyukin.
\newblock Blessing of dimensionality: mathematical foundations of the statistical physics of data.
\newblock \emph{Philosophical Transactions of the Royal Society A: Mathematical, Physical and Engineering Sciences}, 376\penalty0 (2118):\penalty0 20170237, 2018.

\bibitem[Gulrajani et~al.(2017)Gulrajani, Ahmed, Arjovsky, Dumoulin, and Courville]{gulrajani2017improved}
Ishaan Gulrajani, Faruk Ahmed, Martin Arjovsky, Vincent Dumoulin, and Aaron~C Courville.
\newblock Improved training of wasserstein gans.
\newblock \emph{Advances in neural information processing systems}, 30, 2017.

\bibitem[Ho et~al.(2020)Ho, Jain, and Abbeel]{ho2020denoising}
Jonathan Ho, Ajay Jain, and Pieter Abbeel.
\newblock Denoising diffusion probabilistic models.
\newblock \emph{Advances in neural information processing systems}, 33:\penalty0 6840--6851, 2020.

\bibitem[Huster et~al.(2021)Huster, Cohen, Lin, Chan, Kamhoua, Leslie, Chiang, and Sekar]{huster2021pareto}
Todd Huster, Jeremy Cohen, Zinan Lin, Kevin Chan, Charles Kamhoua, Nandi~O Leslie, Cho-Yu~Jason Chiang, and Vyas Sekar.
\newblock Pareto gan: Extending the representational power of gans to heavy-tailed distributions.
\newblock In \emph{International Conference on Machine Learning}, pages 4523--4532. PMLR, 2021.

\bibitem[Karras et~al.(2019)Karras, Laine, and Aila]{karras2019style}
Tero Karras, Samuli Laine, and Timo Aila.
\newblock A style-based generator architecture for generative adversarial networks.
\newblock In \emph{Proceedings of the IEEE/CVF conference on computer vision and pattern recognition}, pages 4401--4410, 2019.

\bibitem[Karras et~al.(2020)Karras, Laine, Aittala, Hellsten, Lehtinen, and Aila]{karras2020analyzing}
Tero Karras, Samuli Laine, Miika Aittala, Janne Hellsten, Jaakko Lehtinen, and Timo Aila.
\newblock Analyzing and improving the image quality of stylegan.
\newblock In \emph{Proceedings of the IEEE/CVF conference on computer vision and pattern recognition}, pages 8110--8119, 2020.

\bibitem[Karras et~al.(2021)Karras, Aittala, Laine, H{\"a}rk{\"o}nen, Hellsten, Lehtinen, and Aila]{karras2021alias}
Tero Karras, Miika Aittala, Samuli Laine, Erik H{\"a}rk{\"o}nen, Janne Hellsten, Jaakko Lehtinen, and Timo Aila.
\newblock Alias-free generative adversarial networks.
\newblock \emph{Advances in neural information processing systems}, 34:\penalty0 852--863, 2021.

\bibitem[Kingma et~al.(2013)Kingma, Welling, et~al.]{kingma2013auto}
Diederik~P Kingma, Max Welling, et~al.
\newblock Auto-encoding variational bayes, 2013.

\bibitem[Kolouri et~al.(2019)Kolouri, Nadjahi, Simsekli, Badeau, and Rohde]{kolouri2019generalized}
Soheil Kolouri, Kimia Nadjahi, Umut Simsekli, Roland Badeau, and Gustavo Rohde.
\newblock Generalized sliced wasserstein distances.
\newblock \emph{Advances in neural information processing systems}, 32, 2019.

\bibitem[Li et~al.(2017)Li, Chang, Cheng, Yang, and P{\'o}czos]{li2017mmd}
Chun-Liang Li, Wei-Cheng Chang, Yu~Cheng, Yiming Yang, and Barnab{\'a}s P{\'o}czos.
\newblock {MMD GAN}: Towards deeper understanding of moment matching network.
\newblock \emph{Advances in neural information processing systems}, 30, 2017.

\bibitem[Lu et~al.(2023)Lu, Lu, Jiang, Szabados, Sun, and Yu]{lu2023cm}
Haoye Lu, Yiwei Lu, Dihong Jiang, Spencer~Ryan Szabados, Sun Sun, and Yaoliang Yu.
\newblock Cm-gan: Stabilizing gan training with consistency models.
\newblock In \emph{ICML 2023 Workshop on Structured Probabilistic Inference $\{$$\backslash$\&$\}$ Generative Modeling}, 2023.

\bibitem[Luo and Yang(2024)]{luo2024dyngan}
Yixin Luo and Zhouwang Yang.
\newblock Dyngan: Solving mode collapse in {GAN}s with dynamic clustering.
\newblock \emph{IEEE Transactions on Pattern Analysis and Machine Intelligence}, 2024.

\bibitem[Mao et~al.(2017)Mao, Li, Xie, Lau, Wang, and Paul~Smolley]{mao2017least}
Xudong Mao, Qing Li, Haoran Xie, Raymond~YK Lau, Zhen Wang, and Stephen Paul~Smolley.
\newblock Least squares generative adversarial networks.
\newblock In \emph{Proceedings of the IEEE international conference on computer vision}, pages 2794--2802, 2017.

\bibitem[Markovich(2016)]{markovich2016light}
L~Markovich.
\newblock Light-and heavy-tailed density estimation by gamma-weibull kernel.
\newblock In \emph{Conference of the International Society for Non-Parametric Statistics}, pages 145--158. Springer, 2016.

\bibitem[Mescheder et~al.(2018)Mescheder, Geiger, and Nowozin]{mescheder2018training}
Lars Mescheder, Andreas Geiger, and Sebastian Nowozin.
\newblock Which training methods for gans do actually converge?
\newblock In \emph{International conference on machine learning}, pages 3481--3490. PMLR, 2018.

\bibitem[Metz et~al.(2016)Metz, Poole, Pfau, and Sohl-Dickstein]{metz2016unrolled}
Luke Metz, Ben Poole, David Pfau, and Jascha Sohl-Dickstein.
\newblock Unrolled generative adversarial networks.
\newblock \emph{arXiv preprint arXiv:1611.02163}, 2016.

\bibitem[Miyato et~al.(2018)Miyato, Kataoka, Koyama, and Yoshida]{miyato2018spectral}
Takeru Miyato, Toshiki Kataoka, Masanori Koyama, and Yuichi Yoshida.
\newblock Spectral normalization for generative adversarial networks.
\newblock \emph{arXiv preprint arXiv:1802.05957}, 2018.

\bibitem[Mohamed and Lakshminarayanan(2016)]{mohamed2016learning}
Shakir Mohamed and Balaji Lakshminarayanan.
\newblock Learning in implicit generative models.
\newblock \emph{arXiv preprint arXiv:1610.03483}, 2016.

\bibitem[Nadjahi et~al.(2021)Nadjahi, Durmus, Jacob, Badeau, and Simsekli]{nadjahi2021fast}
Kimia Nadjahi, Alain Durmus, Pierre~E Jacob, Roland Badeau, and Umut Simsekli.
\newblock Fast approximation of the sliced-wasserstein distance using concentration of random projections.
\newblock \emph{Advances in Neural Information Processing Systems}, 34:\penalty0 12411--12424, 2021.

\bibitem[Nie et~al.(2022)Nie, Zhou, Li, Wang, Lin, and Tong]{nie2022logtrans}
Xingqing Nie, Xiaogen Zhou, Zhiqiang Li, Luoyan Wang, Xingtao Lin, and Tong Tong.
\newblock Logtrans: Providing efficient local-global fusion with transformer and cnn parallel network for biomedical image segmentation.
\newblock In \emph{2022 IEEE 24th Int Conf on High Performance Computing \& Communications; 8th Int Conf on Data Science \& Systems; 20th Int Conf on Smart City; 8th Int Conf on Dependability in Sensor, Cloud \& Big Data Systems \& Application (HPCC/DSS/SmartCity/DependSys)}, pages 769--776. IEEE, 2022.

\bibitem[Papamakarios et~al.(2021)Papamakarios, Nalisnick, Rezende, Mohamed, and Lakshminarayanan]{papamakarios2021normalizing}
George Papamakarios, Eric Nalisnick, Danilo~Jimenez Rezende, Shakir Mohamed, and Balaji Lakshminarayanan.
\newblock Normalizing flows for probabilistic modeling and inference.
\newblock \emph{Journal of Machine Learning Research}, 22\penalty0 (57):\penalty0 1--64, 2021.

\bibitem[Radford et~al.(2015)Radford, Metz, and Chintala]{radford2015unsupervised}
Alec Radford, Luke Metz, and Soumith Chintala.
\newblock Unsupervised representation learning with deep convolutional generative adversarial networks.
\newblock \emph{arXiv preprint arXiv:1511.06434}, 2015.

\bibitem[Resnick and St{\u{a}}ric{\u{a}}(1997)]{resnick1997smoothing}
Sidney Resnick and C{\u{a}}t{\u{a}}lin St{\u{a}}ric{\u{a}}.
\newblock Smoothing the hill estimator.
\newblock \emph{Advances in Applied Probability}, 29\penalty0 (1):\penalty0 271--293, 1997.

\bibitem[Rosenblatt(1952)]{rosenblatt1952remarks}
Murray Rosenblatt.
\newblock Remarks on a multivariate transformation.
\newblock \emph{The annals of mathematical statistics}, 23\penalty0 (3):\penalty0 470--472, 1952.

\bibitem[Saatci and Wilson(2017)]{saatci2017bayesian}
Yunus Saatci and Andrew~G Wilson.
\newblock Bayesian gan.
\newblock \emph{Advances in neural information processing systems}, 30, 2017.

\bibitem[Sajjadi et~al.(2018)Sajjadi, Bachem, Lucic, Bousquet, and Gelly]{sajjadi2018assessing}
Mehdi~SM Sajjadi, Olivier Bachem, Mario Lucic, Olivier Bousquet, and Sylvain Gelly.
\newblock Assessing generative models via precision and recall.
\newblock \emph{Advances in neural information processing systems}, 31, 2018.

\bibitem[Seo et~al.(2024)Seo, Hwang, Lee, and Seok]{seo2024stabilized}
Jangwon Seo, Hyo-Seok Hwang, Minhyeok Lee, and Junhee Seok.
\newblock Stabilized gan models training with kernel-histogram transformation and probability mass function distance.
\newblock \emph{Applied Soft Computing}, 164:\penalty0 112003, 2024.

\bibitem[Shorten and Khoshgoftaar(2019)]{shorten2019survey}
Connor Shorten and Taghi~M Khoshgoftaar.
\newblock A survey on image data augmentation for deep learning.
\newblock \emph{Journal of big data}, 6\penalty0 (1):\penalty0 1--48, 2019.

\bibitem[Srivastava et~al.(2017)Srivastava, Valkov, Russell, Gutmann, and Sutton]{srivastava2017veegan}
Akash Srivastava, Lazar Valkov, Chris Russell, Michael~U Gutmann, and Charles Sutton.
\newblock Veegan: Reducing mode collapse in {GAN}s using implicit variational learning.
\newblock \emph{Advances in neural information processing systems}, 30, 2017.

\bibitem[Stimper et~al.(2023)Stimper, Liu, Campbell, Berenz, Ryll, Schölkopf, and Hernández-Lobato]{Stimper2023}
Vincent Stimper, David Liu, Andrew Campbell, Vincent Berenz, Lukas Ryll, Bernhard Schölkopf, and José~Miguel Hernández-Lobato.
\newblock normflows: A pytorch package for normalizing flows.
\newblock \emph{Journal of Open Source Software}, 8\penalty0 (86):\penalty0 5361, 2023.
\newblock \doi{10.21105/joss.05361}.
\newblock URL \url{https://doi.org/10.21105/joss.05361}.

\bibitem[Sugiyama et~al.(2013)Sugiyama, Liu, du~Plessis, Yamanaka, Yamada, Suzuki, and Kanamori]{sugiyama2013direct}
Masashi Sugiyama, Song Liu, Marthinus~Christoffel du~Plessis, Masao Yamanaka, Makoto Yamada, Taiji Suzuki, and Takafumi Kanamori.
\newblock Direct divergence approximation between probability distributions and its applications in machine learning.
\newblock \emph{Journal of Computing Science and Engineering}, 7\penalty0 (2):\penalty0 99--111, 2013.

\bibitem[Wu et~al.(2021)Wu, Xu, Wang, and Long]{wu2021autoformer}
Haixu Wu, Jiehui Xu, Jianmin Wang, and Mingsheng Long.
\newblock Autoformer: Decomposition transformers with auto-correlation for long-term series forecasting.
\newblock \emph{Advances in Neural Information Processing Systems}, 34:\penalty0 22419--22430, 2021.

\bibitem[Xiao et~al.(2018)Xiao, Zhong, and Zheng]{xiao2018bourgan}
Chang Xiao, Peilin Zhong, and Changxi Zheng.
\newblock Bourgan: Generative networks with metric embeddings.
\newblock \emph{Advances in neural information processing systems}, 31, 2018.

\bibitem[Zeng et~al.(2023)Zeng, Chen, Zhang, and Xu]{zeng2023transformers}
Ailing Zeng, Muxi Chen, Lei Zhang, and Qiang Xu.
\newblock Are transformers effective for time series forecasting?
\newblock In \emph{Proceedings of the AAAI conference on artificial intelligence}, volume~37, pages 11121--11128, 2023.

\bibitem[Zhong et~al.(2019)Zhong, Mo, Xiao, Chen, and Zheng]{zhong2019rethinking}
Peilin Zhong, Yuchen Mo, Chang Xiao, Pengyu Chen, and Changxi Zheng.
\newblock Rethinking generative mode coverage: A pointwise guaranteed approach.
\newblock \emph{Advances in Neural Information Processing Systems}, 32, 2019.

\bibitem[Zhou et~al.(2021)Zhou, Zhang, Peng, Zhang, Li, Xiong, and Zhang]{zhou2021informer}
Haoyi Zhou, Shanghang Zhang, Jieqi Peng, Shuai Zhang, Jianxin Li, Hui Xiong, and Wancai Zhang.
\newblock Informer: Beyond efficient transformer for long sequence time-series forecasting.
\newblock In \emph{Proceedings of the AAAI conference on artificial intelligence}, volume~35, pages 11106--11115, 2021.

\end{thebibliography}

\end{document}